\documentclass[11 pt]{article}
\usepackage[margin=1in]{geometry}                % See geometry.pdf to learn the layout options.
\usepackage{graphicx}
\usepackage{amssymb}
\usepackage{epstopdf}
\usepackage{subfigure}
\usepackage{algorithm}
\usepackage{algorithmic}
\usepackage{amsmath}
\usepackage{enumerate}
\usepackage{color}
\usepackage{textcomp}
\usepackage{amsthm}
\usepackage{relsize}
\usepackage{cite}
\usepackage{cases}

\newtheorem{theorem}{Theorem}
\newtheorem{lemma}{Lemma}
\newtheorem{proposition}{Proposition}
\newtheorem{corollary}{Corollary}

\newcommand{\M}{\ensuremath{\mathcal{M}}}
\newcommand{\D}{\ensuremath{\mathcal{D}}}
\newcommand{\Rsq}{\ensuremath{\mathbb{R}^2}}
\newcommand{\intRsq}{\ensuremath{ \int_{\mathbb{R}^2} }}
\newcommand{\half}{\ensuremath{ \frac{1}{2} }}

\newcommand{\targetp}{\ensuremath{ q}}
\newcommand{\Mlin}{\ensuremath{\mathcal{S}}}
\newcommand{\Gij}{\ensuremath{\mathcal{G}_{ij}}}
\newcommand{\invGij}{\ensuremath{\mathcal{G}^{ij}}}
\newcommand{\Gik}{\ensuremath{\mathcal{G}_{ik}}}
\newcommand{\invGik}{\ensuremath{\mathcal{G}^{ik}}}
\newcommand{\tr}{\ensuremath{ \mathrm{tr}}}
\newcommand{\diag}{\ensuremath{ \mathrm{diag}}}
\newcommand{\lambdamax}{\ensuremath{ \eta_{\mathrm{max}}}}
\newcommand{\lambdamin}{\ensuremath{ \eta_{\mathrm{min}}}}
\newcommand{\hess}{\ensuremath{ h}}
\newcommand{\Hess}{\ensuremath{ H}}
\newcommand{\normhess}{\ensuremath{ N_\hess}}
\newcommand{\normgrad}{\ensuremath{ N_\nabla}}
\newcommand{\op}{\ensuremath{ \diamond}}
\newcommand{\filtsc}{\ensuremath{ \Upsilon}}
\newcommand{\geoconone}{\ensuremath{ C_2}}
\newcommand{\geocontwo}{\ensuremath{ C_1}}

\newcommand{\sepmarg}{\ensuremath{ \epsilon}}
\newcommand{\maxdist}{\ensuremath{ \mathcal{V}}}
\newcommand{\tpar}{\ensuremath{ \lambda}}
\newcommand{\tpari}{\ensuremath{ \lambda^i}}
\newcommand{\tparj}{\ensuremath{ \lambda^j}}
\newcommand{\tpardom}{\ensuremath{ \Lambda}}
\newcommand{\ptran}{\ensuremath{p_{ \lambda}}}
\newcommand{\coordch}{\ensuremath{a}}
\newcommand{\ptranlong}{\ensuremath{A_{ \lambda}(p)}}
\newcommand{\deriptran}{\ensuremath{\partial_{i} \, \ptran}}
\newcommand{\derjptran}{\ensuremath{\partial_{j} \, \ptran}}
\newcommand{\derijptran}{\ensuremath{\partial_{ij} \, \ptran}}
\newcommand{\tparopt}{\ensuremath{ \lambda_{o}}}
\newcommand{\ptranopt}{\ensuremath{p_{ \tparopt}}}
\newcommand{\deriptranopt}{\ensuremath{\partial_{i} \, \ptranopt}}
\newcommand{\derjptranopt}{\ensuremath{\partial_{j} \, \ptranopt}}
\newcommand{\lintranopt}{\ensuremath{ l_{\tparopt}}}

\newcommand{\quadtranoptzetak}{\ensuremath{ \kappa_{\tparopt, \zeta_k}}}
\newcommand{\quadtranoptzetaj}{\ensuremath{ \kappa_{\tparopt, \zeta_j}}}
\newcommand{\noiseopt}{\ensuremath{n}}
\newcommand{\noiselev}{\ensuremath{\nu}}
\newcommand{\noiseeff}{\ensuremath{\nu_{e}}}
\newcommand{\tparref}{\ensuremath{\tpar_r}}
\newcommand{\tparrefi}{\ensuremath{ \tparref^i}}

\newcommand{\ptranref}{\ensuremath{p_{\tparref}}}
\newcommand{\deriptranref}{\ensuremath{\partial_{i} \, \ptranref}}
\newcommand{\derjptranref}{\ensuremath{\partial_{j} \, \ptranref}}
\newcommand{\derkptranref}{\ensuremath{\partial_{k} \, \ptranref}}

\newcommand{\tparest}{\ensuremath{\tpar_e}}
\newcommand{\deltaderi}{\ensuremath{\Delta_i}}
\newcommand{\deltaderj}{\ensuremath{\Delta_j}}
\newcommand{\atomj}{\ensuremath{\phi_{\gamma_j}(X) }}
\newcommand{\atomk}{\ensuremath{\phi_{\gamma_k}(X) }}
\newcommand{\atomjsq}{\ensuremath{\phi^2_{\gamma_j}(X) }}

\newcommand{\atomktrans}{\ensuremath{\phi_{\gamma_k \circ \tpar}(X)}}
\newcommand{\coef}{\ensuremath{ c } }
\newcommand{\coefj}{\ensuremath{ c_j } }
\newcommand{\coefk}{\ensuremath{ c_k } }
\newcommand{\sigmaxj}{\ensuremath{ \sigma_{x,j}  }}
\newcommand{\sigmayj}{\ensuremath{ \sigma_{y,j}  }}
\newcommand{\sigmaxk}{\ensuremath{ \sigma_{x,k}  }}
\newcommand{\sigmayk}{\ensuremath{ \sigma_{y,k}  }}
\newcommand{\sumid}{\ensuremath{\sum_{i=1}^d}}
\newcommand{\sumjd}{\ensuremath{\sum_{j=1}^d}}

\newcommand{\sumjinf}{\ensuremath{\sum_{j=1}^{\infty}}}
\newcommand{\sumkinf}{\ensuremath{\sum_{k=1}^{\infty}}}

\newcommand{\MsecderUB}{\ensuremath{ \mathcal{K} }}
\newcommand{\MderUB}{\ensuremath{ \mathcal{T} }}
\newcommand{\alerrbnd}{\ensuremath{ E }}
\newcommand{\disterr}{\ensuremath{ E }}
\newcommand{\dercoordUB}{\ensuremath{ M}}
\newcommand{\jacobUB}{\ensuremath{ C }}
\newcommand{\Ljk}{\ensuremath{ L_{jk} }}
\newcommand{\LjkUB}{\ensuremath{ \overline{L}_{jk} }}
\newcommand{\LjUB}{\ensuremath{  \overline{L}_{j} }}
\newcommand{\LkUB}{\ensuremath{  \overline{L}_{k} }}
\newcommand{\Mjk}{\ensuremath{ M_{jk} }}
\newcommand{\MjkUB}{\ensuremath{ \overline{M}_{jk} }}
\newcommand{\MjUB}{\ensuremath{ \overline{M}_{j} }}
\newcommand{\MkUB}{\ensuremath{ \overline{M}_{k} }}
\newcommand{\Njk}{\ensuremath{ N_{jk} }}
\newcommand{\NjkUB}{\ensuremath{ \overline{N}_{jk} }}
\newcommand{\Nkj}{\ensuremath{ N_{kj} }}
\newcommand{\Pjk}{\ensuremath{ P_{jk} }}
\newcommand{\PjkUB}{\ensuremath{ \overline{P}_{jk} }}
\newcommand{\derxp}{\ensuremath{\partial_{x} \, p}}
\newcommand{\deryp}{\ensuremath{\partial_{y} \, p}}
\newcommand{\derxxp}{\ensuremath{\partial_{xx} \, p}}
\newcommand{\derxyp}{\ensuremath{\partial_{xy} \, p}}
\newcommand{\deryyp}{\ensuremath{\partial_{yy} \, p}}
\newcommand{\derixp}{\ensuremath{\partial_{i} \, x'}}
\newcommand{\derjxp}{\ensuremath{\partial_{j} \, x'}}
\newcommand{\derijxp}{\ensuremath{\partial_{ij} \, x'}}
\newcommand{\deriyp}{\ensuremath{\partial_{i} \, y'}}
\newcommand{\derjyp}{\ensuremath{\partial_{j} \, y'}}
\newcommand{\derijyp}{\ensuremath{\partial_{ij} \, y'}}
\newcommand{\deriXp}{\ensuremath{\partial_{i} \, X'}}
\newcommand{\hatsigma}{\ensuremath{ \hat{\sigma} }}
\newcommand{\hatsigmaxj}{\ensuremath{ \hat{\sigma}_{x,j}  }}
\newcommand{\hatsigmayj}{\ensuremath{ \hat{\sigma}_{y,j}  }}
\newcommand{\hatsigmaxk}{\ensuremath{ \hat{\sigma}_{x,k}  }}
\newcommand{\hatsigmayk}{\ensuremath{ \hat{\sigma}_{y,k}  }}
\newcommand{\hatSigma}{\ensuremath{ \hat{\Sigma} }}
\newcommand{\hatvartheta}{\ensuremath{ \hat{\vartheta} }}
\newcommand{\hatvarsigma}{\ensuremath{ \hat{\varsigma} }}
\newcommand{\hatgamma}{\ensuremath{ \hat{\gamma} }}
\newcommand{\hatp}{\ensuremath{ \hat{p} }}
\newcommand{\hattargetp}{\ensuremath{ \hat{\targetp} }}
\newcommand{\hatptran}{\ensuremath{\hatp_{ \lambda}}}
\newcommand{\hatptranref}{\ensuremath{\hatp_{\tparref}}}
\newcommand{\filtptranopt}{\ensuremath{ \widehat{ p_{\tparopt} }  }}
\newcommand{\filtptran}{\ensuremath{ \widehat{ p_{\tpar} }  }}
\newcommand{\hattpar}{\ensuremath{ \hat{ \lambda}}}
\newcommand{\hattparopt}{\ensuremath{ \hattpar_{o}}}
\newcommand{\hattparest}{\ensuremath{\hattpar_e}}
\newcommand{\hatptranopt}{\ensuremath{\hatp_{ \hattparopt}}}
\newcommand{\hatptranoptinit}{\ensuremath{\hatp_{ \tparopt}}}
\newcommand{\hatptraninit}{\ensuremath{\hatp_{ \tpar}}}
\newcommand{\hatderiptran}{\ensuremath{\partial_{i} \, \hatptran}}

\newcommand{\hatderijptran}{\ensuremath{\partial_{ij} \, \hatptran}}
\newcommand{\hatderiptranref}{\ensuremath{\partial_{i} \, \hatptranref}}
\newcommand{\hatderjptranref}{\ensuremath{\partial_{j} \, \hatptranref}}

\newcommand{\hatGij}{\ensuremath{\hat{\mathcal{G}}_{ij}}}
\newcommand{\hatinvGij}{\ensuremath{\hat{\mathcal{G}}^{ij}}}
\newcommand{\filtnoiseopt}{\ensuremath{\hat{\noiseopt}}}
\newcommand{\hatnoiseopt}{\ensuremath{\tilde{\noiseopt}}}
\newcommand{\hatMsecderUB}{\ensuremath{ \hat{\MsecderUB } }}
\newcommand{\hatalerrbnd}{\ensuremath{ \hat{\alerrbnd} }}
\newcommand{\hatLjk}{\ensuremath{ \hat{L}_{jk} }}
\newcommand{\hatLjkUB}{\ensuremath{ \overline{\hat{L}}_{jk} }}
\newcommand{\hatMjk}{\ensuremath{ \hat{M}_{jk} }}
\newcommand{\hatMjkUB}{\ensuremath{ \overline{\hat{M}}_{jk} }}
\newcommand{\hatMjUB}{\ensuremath{ \overline{\hat{M}}_{j} }}
\newcommand{\hatMkUB}{\ensuremath{ \overline{\hat{M}}_{k} }}
\newcommand{\hatNjk}{\ensuremath{ \hat{N}_{jk} }}
\newcommand{\hatNjkUB}{\ensuremath{ \overline{\hat{N}}_{jk} }}
\newcommand{\hatNkj}{\ensuremath{ \hat{N}_{kj} }}
\newcommand{\hatPjk}{\ensuremath{ \hat{P}_{jk} }}
\newcommand{\hatPjkUB}{\ensuremath{ \overline{\hat{P}}_{jk} }}
\newcommand{\hatmaxdist}{\ensuremath{ \hat{\mathcal{V}}}}
\newcommand{\hatderxp}{\ensuremath{\partial_{x} \, \hatp}}
\newcommand{\hatderyp}{\ensuremath{\partial_{y} \, \hatp}}
\newcommand{\hatderxxp}{\ensuremath{\partial_{xx} \, \hatp}}
\newcommand{\hatderxyp}{\ensuremath{\partial_{xy} \, \hatp}}
\newcommand{\hatderyyp}{\ensuremath{\partial_{yy} \, \hatp}}
\newcommand{\hatcoefj}{\ensuremath{ \hat{c}_j } }
\newcommand{\hatcoefk}{\ensuremath{ \hat{c}_k } }

\title{A Study of Image Analysis with Tangent Distance} 

\author{Elif Vural and Pascal Frossard \thanks{E. Vural is with Centre de Recherche INRIA Rennes - Bretagne Atlantique, Rennes, France ({elif.vural@inria.fr}).
\newline \indent \indent P. Frossard is with Ecole Polytechnique F\'{e}d\'{e}rale de Lausanne (EPFL), Signal Processing Laboratory - LTS4, Lausanne, Switzerland 
({pascal.frossard@epfl.ch}). 
\newline \indent \indent Most part of the work was performed while the first author was at EPFL.}}

\begin{document}
\maketitle
%\section{}
%\subsection{}

\abstract{The computation of the geometric transformation between a reference and a target image, known as registration or alignment, corresponds to the projection of the target image onto the transformation manifold of the reference image (the set of images generated by its geometric transformations). It, however, often takes a nontrivial form such that the exact computation of projections on the manifold is difficult. The tangent distance method is an effective algorithm to solve this problem by exploiting a linear approximation of the manifold. As theoretical studies about the tangent distance algorithm have been largely overlooked, we present in this work a detailed performance analysis of this useful algorithm, which can eventually help its implementation. We consider a popular image registration setting using a multiscale pyramid of lowpass filtered versions of the (possibly noisy) reference and target images, which is particularly useful for recovering large transformations. We first show that the alignment error has a nonmonotonic variation with the filter size, due to the opposing effects of filtering on both manifold nonlinearity and image noise. We then study the convergence of the multiscale tangent distance method to the optimal solution. We finally examine the performance of the tangent distance method in image classification applications. Our theoretical findings are confirmed by experiments on image transformation models involving translations, rotations and scalings. Our study is the first detailed study of the tangent distance algorithm that leads to a better understanding of its efficacy and to the proper selection of its design parameters.}

\textit{Keywords.} Image registration, tangent distance,  image analysis,  hierarchical registration methods, performance analysis.

\section{Introduction}
\label{ch:tan_dist:sec:intro}

The estimation of the geometric transformation that gives the best match between  a target image and a reference image is known as image registration or image alignment. This operation is commonly used in many problems in image processing or computer vision, such as image analysis, biomedical imaging, video coding and stereo vision. The set of images generated by the geometric transformations of a reference pattern is called a transformation manifold. In several image registration problems, it is possible to represent the geometric transformation between the reference and target images by a few parameters, e.g., translation, rotation, and affine transformation parameters. In this case, the image registration problem can be geometrically regarded as the projection of the target image onto the transformation manifold of the reference image. The transformation parameters that best align the image pair are then given by the transformation parameters of the manifold point that has the smallest distance to the target image. By extension, in image analysis problems where different classes are represented by different transformation manifolds, classification can be achieved by measuring the distance of the query image to the transformation manifold of each class.

Even if the image registration problem is generally not easy to solve exactly due to the nontrivial form of the transformation manifold, its geometric interpretation allows for efficient alignment solutions. A well-known alignment method consists of constructing a first-order approximation of the transformation manifold of the reference image by computing the tangent space of the manifold at a reference point, assuming that the manifold is smooth and that this tangent space can be computed. The transformation parameters are then estimated by calculating the orthogonal projection of the target image onto the tangent space of the manifold. This method is known as the tangent distance method. The tangent distance method has been proposed by Simard et al.~and its efficiency has been demonstrated in numerous settings, like handwritten digit recognition applications \cite{Simard92}, \cite{Simard00} for example. Since then, many variations on the tangent distance method have been presented. The work in \cite{ZissCVPR2003}, for example, introduces the joint manifold distance for transformation-invariance in clustering, which is a similarity measure that is based on the prior distributions of the images and the distance between the linear approximations of their manifolds. The recent work \cite{Fabrizio12} utilizes the tangent distance for motion compensation in video compression. In fact, some early examples of image alignment using manifold linearizations are found in the motion estimation literature, which are called gradient-based optical flow computation methods \cite{Tziritas94}, \cite{Barron94}. Gradient-based methods exploit a linear approximation of the image intensity function in the estimation of the displacement between two image blocks. Applying a first-order approximation of the intensity function of the reference image block and then computing the displacement in a least-squares manner is actually equivalent to projecting the target image block onto the linear approximation of the manifold formed by the translations of the reference image block. 

In image alignment with the tangent distance method, the point around which the reference manifold is linearized is required to be sufficiently close to the exact projection of the target image onto the manifold, which corresponds to the optimal transformation parameters. In that case, the linear approximation of the manifold is valid and the optimal transformation parameters can be estimated accurately. When the distance between the reference and optimal transformation parameters is large, an efficient way to get around this limitation is to apply the tangent distance method in a hierarchical manner \cite{Simard00}, \cite{VasconcelosL05}. In hierarchical alignment, a pyramid of low-pass filtered and downsampled versions of the reference and target images is built, and the alignment is achieved in a coarse-to-fine manner, which is illustrated in Figure \ref{fig:illus_td_hier}.  The transformation parameters are first roughly estimated using the smoothest  images in the pyramid, and then progressively improved by passing to the finer scales. The low-pass filtering applied to generate the coarse-scale images helps to reduce the nonlinearity of the manifold, which renders the linear approximation more accurate and permits the recovery of relatively large transformations. Once the transformation parameters are estimated roughly from coarse scale images, the adjustment in the transformation parameters to be computed in fine scales is relatively small and the linear approximation of the manifold is therefore accurate. The study presented in \cite{VasconcelosL05} applies the multiresolution tangent distance method in image registration and image classification problems and experimentally shows that the similarity measure obtained with the multiresolution tangent distance outperforms those obtained with the Euclidean distance and the single-scale tangent distance. The hierarchical estimation of transformation parameters using manifold linearizations is also very common in motion estimation \cite{Tziritas94}, \cite{Barron94}, and stereo vision \cite{Kanade81}. The multiscale and iterative smoothing approach is in fact used in a wide range of image registration algorithms and transformation models, including nonrigid deformations studied commonly in medical imaging \cite{Thirion98}, \cite{Wrangsjš05}, \cite{Vercauteren09}. While the efficiency of the hierarchical alignment strategy has been observed in many applications, a true characterization of the performance of this family of algorithms for general geometric transformation models is still missing in the literature. The objective of this work is to fill this gap.

%\footnotesize
\begin{figure}[t]
 \centering
  \includegraphics[width=16cm]{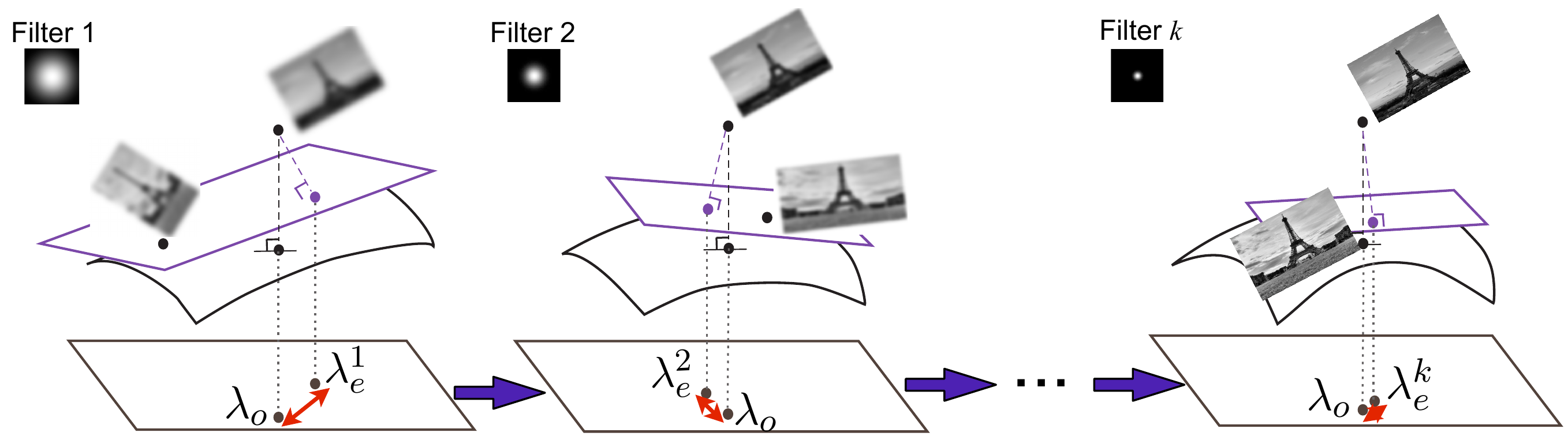}
  \caption{Image alignment with the coarse-to-fine tangent distance method. The target image is a noisy and transformed version of the reference image. The optimal transformation parameters $\tparopt$ that best align the images are estimated in a coarse-to-fine manner with a pyramid of low-pass filtered versions of the images. The estimate $\tparest^k$ of each stage is obtained by linearizing the transformation manifold of the reference image around the reference point given by the estimate  $\tparest^{k-1}$ of the previous stage. The sizes of the low-pass filters are decreased throughout the alignment algorithm as the estimates $\tparest^1, \tparest^2, \dots, \tparest^k$ are refined progressively. (Photos in illustration borrowed from \cite{CopyFreeParis}.)}
  \label{fig:illus_td_hier}
\end{figure}
%\normalsize

We present a theoretical analysis of the properties of the tangent distance method in image alignment and image classification applications. We consider a setting where the reference image is noiseless and the target image is a noisy and transformed version of the reference image. The study of the filtering in the hierarchical alignment method on the accuracy of the solution is especially important, so that the size of the low-pass filter can be properly selected at each stage of the multi-resolution representation. Therefore, an essential step in our study of the tangent distance method is the characterization of the alignment error as a function of the filter size. The second important parameter in our study is the influence of the additive noise that affects images, on the performance of the registration algorithm. Our paper provides a complete analysis of the hierarchical tangent distance algorithm as a function of the manifold properties, the smoothing filter size and the image noise level, and observes the impact of these parameters in both image registration and image classification problems. 

We first analyze the tangent distance method in the original image space (without filtering the images) and derive an upper bound for the alignment error, which is defined as the parameter-domain distance between the optimal transformation parameters that align the image pair perfectly, and their estimate computed with the tangent distance method. The upper bound for the alignment error is obtained in terms of the noise level of the target image, the parameter-domain distance between the reference manifold point (around which the manifold is linearized) and the actual projection onto the manifold, and some geometric parameters of the transformation manifold such as the curvature and the metric tensor. In particular, \textbf{the alignment error bound linearly increases with the manifold curvature and the noise level, and monotonically increases with the parameter-domain distance between the reference and the optimal transformation parameters.}

Next, we study the tangent distance method in a hierarchical registration setting. We first consider that both the reference and the target images are smoothed with a low-pass filter before alignment and examine the variation of the alignment error with the filter size. We show that \textbf{the alignment error decreases with the filter size $\rho$ for small filter kernels at a rate of $O(1+(1+\rho^2)^{-1/2})$}. This is due to the fact that filtering  smoothes the manifold and decreases its nonlinearity, which improves the accuracy of the linear approximation of the manifold. However, as one keeps increasing the filter size, the decrease in the alignment error due to the improvement of the manifold nonlinearity converges, and the error starts to increase with filtering at an approximate rate of $O(\rho)$ for relatively large values of the filter size. The increase in the error stems from the adverse effect of filtering, which amplifies the alignment error caused by image noise. Therefore, we show that, in a noisy setting where the target image is not exactly on the transformation manifold of the reference image, there is an optimal size for the filter kernel where the alignment error takes its minimum value. A related study focusing on the multiscale representations of image manifolds is \cite{Wakin05}, where it is shown that the transformation manifolds of images containing sharp edges are nowhere differentiable. This observation provides an interpretation of why the multiscale application of the Newton algorithm is useful for the registration of non-differentiable images.

We then build on our analysis of the alignment error and study the convergence of the hierarchical tangent distance method. We show that the tangent distance is guaranteed to converge to the optimal solution provided that (i) the product of the noise level and the manifold curvature is below a threshold that depends on the manifold dimension, and (ii) the amount of transformation between the reference and the target images is sufficiently small. Furthermore, we  determine the optimal value of the filter size that minimizes the alignment error in each iteration of the hierarchical alignment algorithm. Our analysis shows that, \textbf{the optimal update of the filter size $\rho$ between adjacent iterations $k-1$ and $k$ is approximately given by $\rho_k = \sqrt{\alpha} \, \rho_{k-1}$}, where the geometric decay factor $\alpha<1$ increases linearly with the noise level, the manifold curvature and the initialization error of the hierarchical alignment algorithm (i.e., the amount of transformation at the beginning of the algorithm). This result theoretically justifies the common strategy of reducing the filter size at a geometric rate, which is used very often in coarse-to-fine image registration. Meanwhile, although it is very common to update the filter size as $\rho_k = 1/2 \, \rho_{k-1}$ with a constant decay factor of $1/2$ in practice \cite{VasconcelosL05}, \cite{Burt83}, our result rather suggests that the noise level, the expected amount of transformation, and the frequency characteristics of the images to be aligned must be taken into account in determining the best filter size updates.

Finally, we study the accuracy of image classification based on the manifold distance estimates obtained with the tangent distance method.  In an image classification application where a query image is classified with respect to its distance to the transformation manifold of each class, the accuracy of classification largely depends on the accuracy of the estimation of the projection of the query image onto the manifolds. Therefore, one expects the classification performance to vary similarly to the alignment performance. We consider a setting where the query image and the reference images representing different classes are smoothed with low-pass filters. Then, we approximate the projection of the query image onto the transformation manifolds of the reference images with the tangent distance method. We determine the relation between the accuracy of classification and the size of the low-pass filter used for smoothing the images. Our result shows that, assuming bounded and non-intersecting distributions of the images around the transformation manifolds of their classes, the variation of the misclassification probability with the filter size is similar to that of the alignment error. Therefore, the filter size that minimizes the alignment error also minimizes the misclassification probability. 

Our theoretical results about the alignment and classification performance of the tangent distance method are confirmed by experiments conducted on transformation manifolds generated with rotations, translations and scale changes, both with synthetic smooth images and natural images. Our study provides insights into the principles behind the efficacy of the hierarchical alignment strategy in image registration and motion estimation, which are helpful for optimizing the performance of numerous image analysis algorithms that rely on first-order approximations of transformation manifolds.

Finally, we mention some previous works focusing on parametric manifolds to address common image processing problems. In \cite{Donoho05}, the geometric structure of manifolds generated by varying a few parameters that control the appearance of an object in an image (image appearance manifolds - IAMs) is examined and several examples of IAMs that are isometric to the Euclidean space are provided. In \cite{Peyre09}, various parametrizable patch manifolds such as cartoon images and oscillating textures are studied and their application is demonstrated in the regularization of inverse problems in image processing. The analysis in \cite{Jacques08} focuses on parametrizable dictionary manifolds generated by the geometric transformations of a prototype function and studies the performance of matching pursuit approximations of signals using a discretization of the dictionary manifold.

The rest of the text is organized as follows. In Section \ref{ch:tan_dist:sec:framework}, we introduce the notation, give an overview of the tangent distance algorithm, and formulate the problem. In Section \ref{ch:tan_dist:sec:analysis}, we present a theoretical analysis of image registration with the tangent distance method. We first  state an upper bound for the alignment error and then examine its variation with the noise level and filtering. In Section \ref{ch:tan_dist:ssec:conv_anly}, we study the convergence of the coarse-to-fine tangent distance method. In Section \ref{ssec:class_td}, we extend our results to analyze the performance of image classification with the tangent distance algorithm. In Section \ref{sec:exp_img_align}, we evaluate our theoretical findings with some experiments. In Section \ref{ch:tan_dist:sec:discussion}, we give a discussion of our results in comparison with previous works. Finally, we conclude in Section \ref{ch:tan_dist:sec:conclusion}.

%%%%%%%%%%%%%%%%%%%%%%
% PROBLEM FORMULATION SECTION
%%%%%%%%%%%%%%%%%%%%%%%
\section{Image Registration with Tangent Distance}
\label{ch:tan_dist:sec:framework}

The computation of the exact projection of a target image onto a reference transformation manifold is a complicated optimization problem, especially when the manifold is high-dimensional and generated by complex geometric transformations. The tangent distance method proposes to solve this problem by using a first-order approximation of the transformation manifold, which is illustrated in Figure \ref{fig:illus_td}. In the figure, $\M(p)$ is the transformation manifold of the reference pattern $p$ defined over the parameter domain $\tpardom$, and $\targetp$ is the target image to be aligned with $p$. The exact projection of $\targetp$ on $\M(p)$ is the point $\ptranopt$, so that $\tparopt$ is the optimal transformation parameter vector that best aligns $p$ with $\targetp$. In order to estimate $\tparopt$ with the tangent distance method, a first order approximation $\Mlin_{\tparref}(p)$ of the manifold $\M(p)$ is computed  at a reference point $\ptranref$, which is preferably not too distant from $\ptranopt$.  The distance of $\targetp$ to $\Mlin_{\tparref}(p)$ can be easily computed with a least squares solution and the point of projection on $\Mlin_{\tparref}(p)$ gives the transformation parameter vector $\tparest$, which is the estimate of $\tparopt$.

\begin{figure}[t]
 \centering
  \includegraphics[width=8cm, trim=0cm 0cm 0cm 0cm, clip=true]{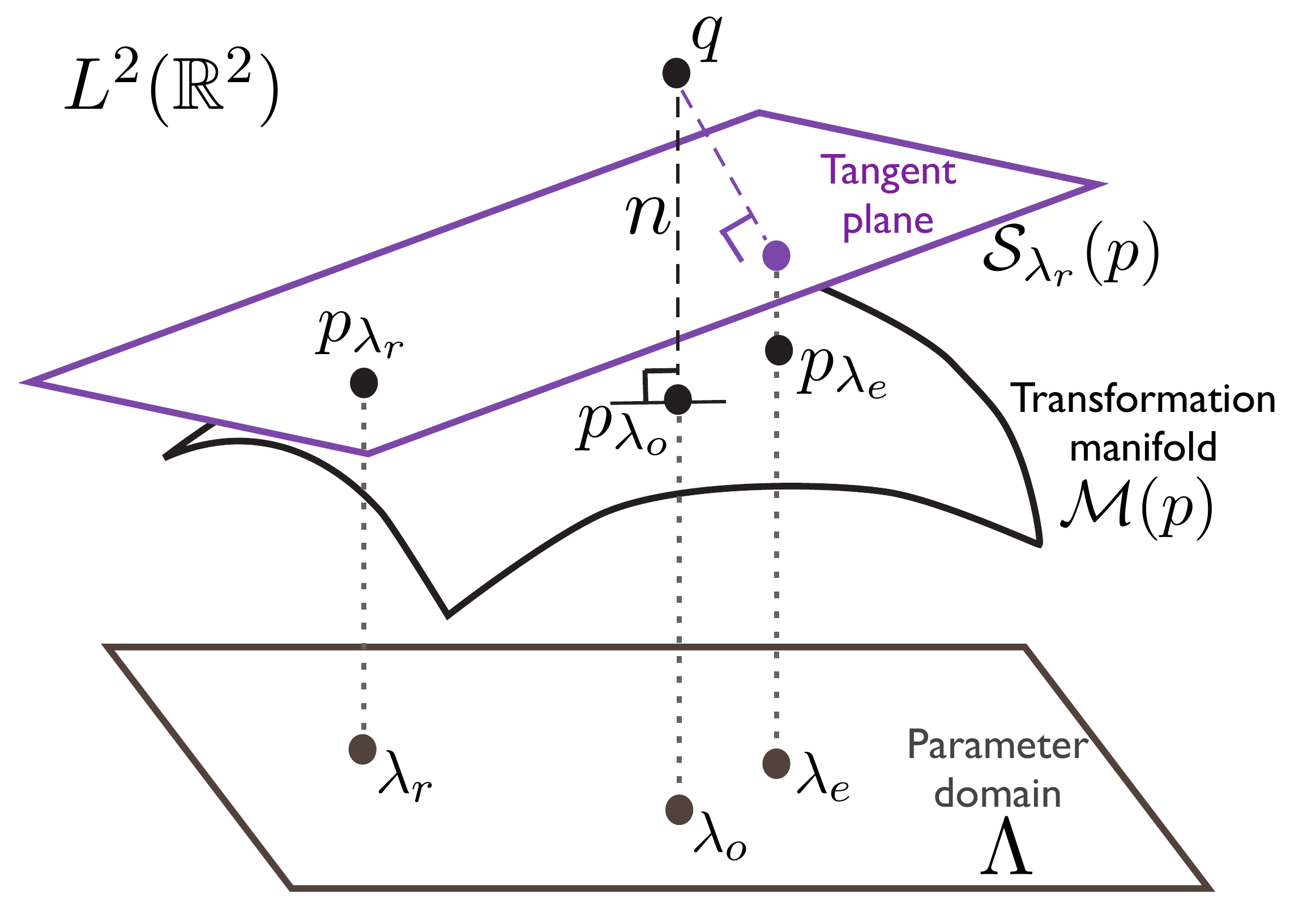}
  \caption{Illustration of image alignment with the tangent distance method. $\Mlin_{\tparref}(p)$ is the first-order approximation of the transformation manifold $\M(p)$ around the reference point $\ptranref$. The estimate $\tparest$ of the optimal transformation parameters $\tparopt$ is obtained by computing the orthogonal projection of the target image $\targetp$ onto $\Mlin_{\tparref}(p)$.}
  \label{fig:illus_td}
\end{figure}

Previous works such as \cite{Simard00} and \cite{ZissCVPR2003} using the tangent distance in image classification and clustering compute the distance in a symmetric fashion; i.e., they linearize the transformation manifolds of both the reference and the target images and compute the subspace-to-subspace distance. In our analysis of the tangent distance method, we consider the point-to-subspace distance obtained by linearizing the transformation manifold of only the reference image \cite{VasconcelosL05}, \cite{Kanade81}. The point-to-subspace distance is more suitable than the subspace-to-subspace distance in image registration applications since its computation does not only yield a similarity measure, but also aligns  the target image with respect to the manifold of the reference image. The point-to-subspace distance can also be used effectively in image analysis \cite{VasconcelosL05}.

In the following, we first settle the notation and describe the tangent distance method formally. We then formulate the registration analysis problem studied in this work.

%\subsection{Notation and problem formulation}
%\label{ch:tan_dist:sec:notation}

Let $p \in L^2(\mathbb{R}^2)$ be a reference pattern that is $C^2$-smooth with square-integrable derivatives and $\targetp \in L^2(\mathbb{R}^2)$ be a target pattern. Let $\tpardom \subset \mathbb{R}^d$ denote a compact, $d$-dimensional transformation parameter domain and 
$
\tpar = [\tpar^1 \ \tpar^2 \ \cdots \ \tpar^d] \in \tpardom
$
be a transformation parameter vector. We denote the pattern obtained by applying to $p$ the geometric transformation specified by $\tpar$ as $\ptranlong \in L^2(\Rsq)$. Defining the spatial coordinate variable $X=[x \ y]^T$ in $\Rsq$, we can express the relation between $\ptranlong$ and $p$ as
\begin{equation}
\label{eq:rel_geom_trans}
\ptranlong(X) = p(\coordch(\tpar,X))
\end{equation}
where $\coordch: \tpardom \times \Rsq \rightarrow \Rsq$ is a $C^2$-smooth function representing the change of coordinates defined by the geometric transformation $\tpar$. We also assume that the coordinate change function $\coordch_{\tpar}: \Rsq \rightarrow \Rsq$ such that $\coordch_{\tpar}(X) := \coordch(\tpar, X)$, is a bijection for a fixed $\tpar$.

Let us write $\ptran=\ptranlong$ for convenience. Then, the transformation manifold $\M(p)$ of the pattern $p$ is given by
\begin{equation*}
\M(p)=\{ \ptran: \tpar \in \tpardom  \} \subset L^2(\Rsq)
\end{equation*}
which consists of transformed versions of $p$ over the parameter domain $\tpardom$. Since $\coordch$ and $p$ are $C^2$-smooth, the local embedding of $\M(p)$ in $L^2(\Rsq)$ is $C^2$-smooth. Therefore, the first and second-order derivatives of manifold points with respect to the transformation parameters exist. We denote the derivative of the manifold point $\ptran$ with respect to the $i$-th transformation parameter $\tpari$ as $\deriptran$, where
$
\deriptran (X) = \partial \, \ptran (X) / \partial \tpari.
$
The derivatives $\deriptran$ correspond to the tangent vectors of $\M(p)$ on $\ptran$. Similarly, we denote the second-order derivatives by
$
\derijptran (X) = \partial^2 \ptran (X) / \partial \tpari \partial \tparj.
$
Then, the tangent space $T_{\tpar} \M(p)$ of the manifold at a point $\ptran$ is the subspace generated by the tangent vectors at $\ptran$
\begin{equation}
\label{eq:tanspace_defn}
T_{\tpar} \M(p) = \left\{  \deriptran \ \zeta^i: \zeta \in \mathbb{R}^d \right\} \subset L^2(\Rsq)
\end{equation}
where $\{ \deriptran \}_{i=1}^d$ are the basis vectors of $T_{\tpar} \M(p)$, and $\{ \zeta^i \}_{i=1}^d$ are the coefficients in the representation of a vector in $T_{\tpar} \M(p)$ in terms of the basis vectors. Throughout the paper, we use the Einstein notation\footnote{In Einstein summation convention, an index variable appearing twice in a term (once in a superscript and once in a subscript) indicates a summation; i.e., $\sum_{i=1}^d v_i w^i$ is simply written as $v_i w^i$.} for representing the summations over the parameter space whenever it simplifies the writing. The term $\deriptran \ \zeta^i$ in (\ref{eq:tanspace_defn}) thus corresponds to the linear combination of the tangent vectors given by the coefficients $\zeta$. 

Now, given the reference pattern $p$ and a target pattern $\targetp$, the image registration problem consists of the computation of an optimal transformation parameter vector $\tparopt$ that gives the best approximation of $\targetp$ with the points $\ptran$ on $\M(p)$, 
\begin{equation}
\label{eq:lambdaopt_defn}
\textbf{Registration problem: } \qquad  \qquad  
\tparopt = \arg \min_{\tpar \in \tpardom} \| \targetp - \ptran  \|^2
\qquad \qquad \qquad \qquad \qquad
\end{equation}
where $\| \cdot \|$ denotes the $L^2$-norm for vectors in the continuous space $L^2(\Rsq)$ and the $\ell^2$-norm for vectors in the discrete space $\mathbb{R}^n$. Then, the transformed pattern $\ptranopt$ is called a projection of $\targetp$ on $\M(p)$. However, the exact calculation of $\tparopt$ is difficult in general, since the nonlinear and highly intricate geometric structure of pattern transformation manifolds renders the distance minimization problem quite complicated.

%\subsection{Tangent distance algorithm}
%\label{ch:tan_dist:sec:td_algo}

The tangent distance method simplifies this problem  to a least squares problem, where the transformation parameters are estimated by using a linear approximation of the manifold $\M(p)$ and then computing $\tparopt$ by minimizing the distance of $\targetp$ to the linear approximation of $\M(p)$ \cite{VasconcelosL05}. The first-order approximation of $\M(p)$ around a reference manifold point $\ptranref$ is given by
\begin{equation}
\Mlin_{\tparref}(p) = \{ \ptranref +  \deriptranref (\tpari - \tparrefi): \tpar \in \mathbb{R}^d \} \subset L^2(\Rsq).
\label{eq:defn_mlin_appx}
\end{equation}
Then, the estimate $\tparest$ of $\tparopt$ with the tangent distance method is given by the solution of the following least squares problem, which seeks the closest point in $\Mlin_{\tparref}(p)$ to $\targetp$.
\begin{equation}
\label{eq:tdreg_problem}
\textbf{Tangent distance:} \qquad
\tparest = \arg \min_{\tpar  \in \mathbb{R}^d} \|  q - \ptranref -  \deriptranref (\tpari - \tparrefi) \|^2
\qquad
\end{equation}
The solution of the above problem can be obtained as
\begin{equation}
\label{eq:tdreg_soln}
\tparest^i = \tparref^i +  \invGij(\tparref)  \langle \targetp - \ptranref, \derjptranref  \rangle
\end{equation}
where
$\Gij(\tpar)=\langle \deriptran, \derjptran \rangle$ is the metric tensor induced from the standard inner product on $L^2(\Rsq)$, $[ \Gij(\tpar) ] \in \mathbb{R}^{d \times d}$ is the matrix representation of the metric tensor, and $\invGij$ represents the entries of the inverse  $[ \Gij(\tpar) ]^{-1} $ of the metric. The estimate $\tparest$ of the transformation parameters obtained by solving (\ref{eq:tdreg_problem}) is expected to be closer to the optimal solution $\tparopt$ than the reference parameters $\tparref$; therefore, $\tparest$  can be regarded as a refinement of $\tparref$ if the reference parameters $\tparref$ are considered as an initial guess for the optimal ones $\tparopt$ (see Figure \ref{fig:illus_td}).

The alignment error of the tangent distance method is thus given by the deviation $\| \tparest - \tparopt  \|$ between the estimated and the optimal parameters. In Section \ref{ch:tan_dist:sec:deriv_ub}, this error is bounded in terms of the known geometric parameters of the manifold $\M(p)$ that can be computed from $p$ (such as its curvature and metric tensor), the distance $\| \tparopt - \tparref \|$ between the optimal and the reference transformation parameters, and the noise level $\noiselev$ of the target image. We define the noise level $\noiselev$ as the distance of the target pattern to the transformation manifold of the reference pattern. Decomposing the target image as
\[
\targetp = \ptranopt + \noiseopt
\]
in terms of its projection $\ptranopt$ onto the manifold $\M(p)$ and its deviation $\noiseopt \in L^2(\Rsq)$ from $\M(p)$, the noise level parameter is given by $ \noiselev =\| \noiseopt \| $.

Note that it is also possible to formulate the alignment error as the manifold distance estimation error measured in the ambient space $L^2(\Rsq)$. Even if both errors are expected to have similar behaviors, in this study, we characterize the error in the parameter space $\Lambda$ instead of the ambient space $L^2(\Rsq)$ because of the following reason. Since we examine the problem in a multiscale setting, it is easier to characterize the error in the parameter domain as the distances in the ambient space are not invariant to smoothing.

%%%%%

In the multiscale tangent distance algorithm, the transformation parameters are estimated by using a pyramid of low-pass filtered versions of the reference and target images. We consider a Gaussian kernel for the low-pass filter, since it is a popular smoothing kernel whose distinctive properties have been well-studied in scale-space theory \cite{Lindeberg94}. Let 
$ \phi(X) = e^{-X^T X}=e^{-(x^2+y^2)} $
denote a Gaussian mother function. Then, the family of functions
\begin{equation}
\frac{1}{\pi \rho^2} \phi_{\rho}(X)
\label{eq:Gausskerdefn}
\end{equation}
define variable-sized, unit $L^1$-norm Gaussian low-pass filters, where $\phi_{\rho}(X)=\phi(\filtsc^{-1} (X))$ is a scaled version of the mother function $\phi(X)$ with
\begin{equation}
\label{eq:defnfiltsc}
\filtsc= \left[
\begin{array}{c c}
 \rho & 0  \\
 0 & \rho
\end{array} \right].
\end{equation}
%.
Here, the scale parameter $\rho$ corresponds to the radius of the filter kernel, which controls the filter size. The transformation parameters are estimated using the filtered versions of the reference and target patterns
\begin{equation*}
\hatp(X) = \frac{1}{\pi \rho^2} \,  (\phi_{\rho} * p)(X) 
\qquad \qquad
\hattargetp (X)  = \frac{1}{\pi \rho^2} \,  (\phi_{\rho} * \targetp)(X) 
\end{equation*}
where $*$ denotes a convolution. Throughout the hierarchical alignment algorithm, the size $\rho$ of the low-pass filter is reduced gradually and the estimate of each stage is used as the initial guess of the next stage. 

%\subsection{Analysis methodology}
%\label{ssec:prob_form}

Let $\hattparopt$ denote the transformation parameter vector corresponding to the projection of the smoothed target pattern $\hattargetp$ onto the transformation manifold $\M(\hatp)$ of the smoothed reference pattern $\hatp$
\begin{equation}
\label{eq:hatlambdaopt_defn}
\hattparopt = \arg \min_{\tpar \in \tpardom} \| \hatptran - \hattargetp \|^2.
\end{equation}
Hence, $\hattparopt$ is the optimal transformation parameter vector that aligns $\hatp$ with $\hattargetp$. Throughout the paper, we write the parameters that are associated with the filtered versions of the reference and target patterns with the notation $\hat{(\cdot)}$. Hence, $\hatderiptran$ and  $\hatGij$ denote respectively the first derivatives and the metric tensor of the manifold $\M(\hatp)$. From (\ref{eq:tdreg_soln}), the transformation estimate $\hattparest$ obtained with the filtered versions of the reference and target patterns by linearizing the manifold $\M(\hatp)$ is given by 
\begin{equation*}
\hattparest^i = \tparref^i + \hatinvGij(\tparref) \langle \hattargetp - \hatptranref, \hatderjptranref  \rangle
\end{equation*}
where $\tparref$ is the reference parameter vector. The alignment error obtained with the smoothed patterns is given as $\|  \hattparest - \hattparopt  \|$, whose variation with the filter size $\rho$ and the noise level $\noiselev$ is studied in Section \ref{ch:tan_dist:ssec:dep_align_bnd}. 

% In Section \ref{ch:tan_dist:ssec:conv_anly}, we use these results to study the \textbf{convergence of the multiscale alignment with tangent distance}. We determine some conditions on the noise level $\noiselev$, the manifold curvature, the accuracy of the initial solution $\| \tparopt - \tparref \|$, and the selection of the filter sizes throughout the iterations of hierarchical alignment  that guarantee the convergence of the estimates to the optimal solution $\tparopt$.

%Finally, we would like to study the accuracy of the tangent distance method in transformation-invariant classification applications. In Section \ref{ssec:class_td}, we consider a setting with several transformation manifolds representing different classes and assume that the distribution of the images from each class is concentrated around the transformation manifold of that class. We then study the probability of misclassifying an image with the tangent distance method based on the estimates $\tparest$ of the transformation parameters $\tparopt$.

%%%%%%%%%%%%%%%%%%%%%%
% ANALYSIS SECTION
%%%%%%%%%%%%%%%%%%%%%%%

\section{Analysis of Alignment Error with Tangent Distance}
\label{ch:tan_dist:sec:analysis}

\subsection{Upper bound for the alignment error}
\label{ch:tan_dist:sec:deriv_ub}

We now present an upper bound for the error of the alignment computed with the tangent distance method. We can assume that the parameter domain $\tpardom$ is selected sufficiently large, so that $\ptranopt$ is not on the boundary of $\M(p)$. Then, the noise pattern $\noiseopt$ is orthogonal to the tangent space of  $\M(p)$ at $\ptranopt$. In other words, we have
\begin{equation}
\label{eq:noiseopt_orth}
\langle \noiseopt , \deriptranopt \rangle = 0, \ \ \ \forall i=1, \cdots , d.
\end{equation}
The deviation of the target image from the transformation manifold model impairs the estimation of transformation parameters. In our analysis of the alignment error, this deviation is characterized by the distance $\noiselev$ between $\targetp$ and $\M(p)$. Then, there is another source of error that causes the deviation of the estimated parameters $\tparest$ from the optimal ones $\tparopt$. It is related to  the nonzero curvature of the manifold, as a result of which $\M(p)$ diverges from its linear approximation $\Mlin_{\tparref}(p)$. The nonlinearity of the manifold can be characterized with an upper bound $\MsecderUB$ on the norm of the second derivatives of the manifold 
\begin{equation}
\MsecderUB := \max_{i,j = 1, \cdots, d}  \  \sup_{\tpar \in \tpardom} \| \derijptran  \| .
\label{defn:sup_curvature}
\end{equation}
Since $\MsecderUB$ is an upper bound for the norms of the derivatives of tangent vectors, it can be regarded as a uniform curvature bound parameter for $\M(p)$. 

We can now state our result that defines an upper bound on the alignment error.

%%%% THEOREM: ALIGNMENT ERROR BOUND
\begin{theorem}
\label{thm:bnd_alignerrTD}
The parameter-domain distance between the optimal transformation $\tparopt$ and its estimate $\tparest$ given by the tangent distance method can be upper bounded as
\begin{equation}
\label{eq:alerrbnd}
\| \tparest - \tparopt   \| \leq \ \alerrbnd := \MsecderUB \ \lambdamin^{-1} \ \big( [ \Gij (\tparref) ] \big) 
\left( \half \,  \sqrt{\tr( [ \Gij (\tparref) ] )} \ \| \tparopt - \tparref  \|_{1}^2
+  \sqrt{d} \ \noiselev  \   \|  \tparopt - \tparref   \|_1
\right)
\end{equation}
where $\lambdamin( \cdot )$ and $\tr(.)$ denote respectively the smallest eigenvalue and the trace of a matrix, and the notation $\| \cdot \|_{1}$ stands for the $\ell^1$-norm in $\mathbb{R}^n$.
\end{theorem}
%%%%%%%

Theorem \ref{thm:bnd_alignerrTD} is proved in Appendix \ref{app:pf_thm_bnd_alignerrTD}. The result is obtained by examining the effects of both the nonlinearity of the manifold and the image noise on the alignment error. The theorem shows that the alignment error augments with the increase in the manifold curvature parameter $\MsecderUB$ and the noise level $\noiselev$, as expected. Moreover, another important factor affecting the alignment error is the distance $\|  \tparopt - \tparref  \|$ between the reference and the optimal transformation parameters. If the reference manifold point $\ptranref$ around which the manifold is linearized is sufficiently close to the true projection of the target image onto the manifold, the tangent distance method is more likely to give a good estimate of the registration parameters. In particular, bounding the $\ell^1$-norms in terms of $\ell^2$-norms in the theorem, we obtain
\begin{equation*}
\| \tparest - \tparopt   \| \leq  d \, \MsecderUB \ \lambdamin^{-1} \ \big( [ \Gij (\tparref) ] \big) 
\left( \half \,  \sqrt{\tr( [ \Gij (\tparref) ] )} \ \| \tparopt - \tparref  \|^2
+   \ \noiselev  \   \|  \tparopt - \tparref   \|
\right).
\end{equation*}
Therefore, the accuracy of the initial solution $\| \tparopt - \tparref    \| $  must be of $O(d^{-1/2} \, \MsecderUB^{-1/2} )$ in order to establish a practically useful guarantee on the alignment performance. Nevertheless, the dimension $d$ of the parameter space is usually small. Hence, the requirement on the accuracy of the initial solution $\| \tparopt - \tparref    \| $ is set by the curvature of the transformation manifold, which depends on the type of the geometric transformation and the smoothness of the image intensity function $p$.

\subsection{Alignment error with low-pass filtering}
\label{ch:tan_dist:ssec:dep_align_bnd}

We now analyze the influence of the low-pass filtering of the reference and target patterns on the accuracy of alignment with the tangent distance method as it is the case in multiscale registration algorithms. We consider a setting where the reference pattern $p$ and the target pattern $\targetp$ are low-pass filtered and the transformation parameters are estimated with the smoothed versions of $p$ and $\targetp$. The purpose of this section is then to analyze the variation of the alignment error bound given in Theorem \ref{thm:bnd_alignerrTD} with respect to the kernel size of the low-pass filter used in smoothing.

We first remark the following. The optimal transformation parameter vector $\hattparopt$ corresponding to the smoothed patterns is in general different from the optimal transformation parameter vector $\tparopt$ corresponding to the unfiltered patterns $p$ and $\targetp$. This is due to the fact that both the image noise and the filtering cause a perturbation in the global minimum of the function $f(\tpar)=\| \targetp - \ptran \|^2$, which represents the distance between the target pattern $q$ and the transformed versions of the reference pattern $p$. Note that the overall error in the transformation parameter estimation is $\| \hattparest - \tparopt \|$ and it can be upper bounded as
\begin{equation*}
\| \hattparest - \tparopt \| \leq \| \hattparest -  \hattparopt  \| + \|  \hattparopt - \tparopt \|.
\end{equation*}
Here, the first error term  $\|  \hattparest - \hattparopt  \|$ results from the linearization of the manifold, whereas the second error term $ \|  \hattparopt - \tparopt \|$ is due to the shift in the global minimum of the distance function $f(\tpar)$. The second error term $\| \hattparopt - \tparopt \|$ depends on the geometric transformation model. In our recent work \cite{Vural12}, this error is examined for the transformation model of 2-D translations and its dependence on the noise level and low-pass filtering is studied. In this study, we analyze how  the linearization of the manifold affects the estimation of the transformation parameters for generic transformation models. Therefore, we focus on the first error term $\|  \hattparest - \hattparopt  \|$ associated particularly with the registration of the images using the tangent distance, and examine its variation with the noise level and the filtering process. The error term $\|  \hattparest - \hattparopt  \|$ caused by the manifold linearization is in general expected to be dominant over the error term $\| \hattparopt - \tparopt \|$ unless the reference parameters $\tparref$ are really close to the optimal parameters $\tparopt$.

The filtered target pattern can be decomposed as
\begin{equation*}
\hattargetp = \hatptranopt + \hatnoiseopt
\end{equation*}
where the noise pattern $\hatnoiseopt$ is orthogonal to the tangent space $T_{\hattparopt} \M(\hatp)$ at $\hatptranopt$. Let $\hatMsecderUB$ denote the curvature bound parameter of the manifold $\M(\hatp)$. Then, from Theorem \ref{thm:bnd_alignerrTD}, the alignment error obtained with the smoothed patterns can be upper bounded as
\begin{equation}
\label{eq:hatalerrbnd}
\| \hattparest - \hattparopt  \| \leq \ \hatalerrbnd =  \hatalerrbnd_1 +  \hatalerrbnd_2
\end{equation}
where
\begin{equation*}
\begin{split}
 \hatalerrbnd_1  &= \half \,  \,  \hatMsecderUB \ \lambdamin^{-1} \ \big( [ \hatGij (\tparref) ] \big)  \sqrt{\tr( [ \hatGij (\tparref) ] )} \ \| \hattparopt - \tparref  \|_{1}^2 \\
 \hatalerrbnd_2 &= \sqrt{d} \ \hatMsecderUB \ \lambdamin^{-1} \ \big( [ \hatGij (\tparref) ] \big)  \ \| \hatnoiseopt \|  \   \|  \hattparopt - \tparref   \|_1.
\end{split}
\end{equation*}
Here the components $ \hatalerrbnd_1$ and $ \hatalerrbnd_2$ of the overall error correspond respectively to the first and second additive terms in (\ref{eq:alerrbnd}). The error term $ \hatalerrbnd_1$ results from the manifold nonlinearity, while the error term $ \hatalerrbnd_2$ is due to noise.

%\hatalerrbnd = \hatMsecderUB \ \lambdamin^{-1} \ \big( [ \hatGij (\tparref) ] \big) 
%\left( \half \,  \sqrt{\tr( [ \hatGij (\tparref) ] )} \ \| \hattparopt - \tparref  \|_{1}^2
%+  \sqrt{d} \ \| \hatnoiseopt \|  \   \|  \hattparopt - \tparref   \|_1
%\right).
%

%%%% CUT HERE START

%%%%% CUT HERE END %%%%%%

We present below our second result, which states the dependence of the alignment error $\hatalerrbnd$ on the initial noise level of the target pattern and the filter size.

\begin{theorem}
\label{thm:dep_alerrbnd}
The alignment error obtained when the smoothed image pair is aligned with the tangent distance method is upper bounded as
\begin{equation*}
\| \hattparest - \hattparopt  \| \leq \ \hatalerrbnd =  \hatalerrbnd_1 +  \hatalerrbnd_2
\end{equation*}
where the error component $ \hatalerrbnd_1$ resulting from manifold nonlinearity decreases at rate
\begin{equation*}
\hatalerrbnd_1 =  O\left(1 + (1+\rho^2)^{-1/2} \right)
\end{equation*}
with the size $\rho$ of the low-pass filter kernel used for smoothing the reference and target images. The second component $\hatalerrbnd_2$ of the alignment error associated with  image noise has the variation 
\begin{equation*}
\hatalerrbnd_2 =  O\left( (\noiselev+1) \, (1+\rho^2)^{1/2}  \right) 
\end{equation*}
with the filter size $\rho$ and the noise level $\noiselev$  if the geometric transformation model includes a scale change. The variation of $\hatalerrbnd_2$ with $\rho$ and $\noiselev$ is
\begin{equation*}
\hatalerrbnd_2 = O\left( \noiselev \, (1+\rho^2)^{1/2} \right)
\end{equation*}
if the geometric transformation model does not change the scale of the pattern.
\end{theorem}

The proof of Theorem \ref{thm:dep_alerrbnd} is given in Appendix \ref{app:pf_thm_dep_alerrbnd}. In short, this result is obtained by studying the variation of each one of the terms in the alignment error bound with the filter size $\rho$. These are then finally put together to  determine the behavior of the overall error.

Theorem \ref{thm:dep_alerrbnd} can be interpreted as follows. The first error component $\hatalerrbnd_1$ related to manifold nonlinearity is of $O\left(1+ \,(1+\rho^2)^{-1/2}\right)$. Since filtering the patterns makes the manifold smoother and decreases the manifold curvature, it improves the accuracy of the first-order approximation of the manifold used in tangent distance. Therefore, the first component of the alignment error decreases with the filter size $\rho$. Then, we observe that the second error component $\hatalerrbnd_2 = O\left( (\noiselev+1) \, (1+\rho^2)^{1/2} \right) $ resulting from image noise is proportional to the noise level, as expected, but it also increases with the filter size $\rho$. The increase of the error with smoothing is due to the fact that filtering has the undesired effect of amplifying the alignment error caused by the noise. This result is in line with the findings of our previous study \cite{Vural12}, and previous works such as \cite{Robinson04}, \cite{Yetik06} that examine the Cr\'amer-Rao lower bound in image registration. This is discussed in more detail in Section \ref{ch:tan_dist:sec:discussion}.

The dependence of the overall alignment error on the filter size can be interpreted as follows. For reasonably small values of the image noise level, the overall error $\hatalerrbnd$ first decreases with the filter size $\rho$ at small filter sizes due to the decrease in the first term $\hatalerrbnd_1$, since filtering improves the manifold linearity. As one keeps increasing the filter size, the first error term $\hatalerrbnd_1 =O\left(1+ \,(1+\rho^2)^{-1/2}\right)$ gradually decreases and finally converges to a constant value. After that, the second error term $\hatalerrbnd_2$ takes over and the overall alignment error $\hatalerrbnd$ starts to increase with the filter size. The amplification of the registration error resulting from the image noise then becomes the prominent factor that determines the overall dependence of the error on the filter size. As the alignment error first decreases and then increases with filtering, there exists an optimal value of the filter size $\rho$ for a given noise level $\noiselev$. In the noiseless case where $\noiselev=0$, our result shows that applying a big filter is favorable as it flattens the manifold, provided that the transformation model does not involve a scale change. Meanwhile, for geometric transformations involving a scale change, there exists a nontrivial optimal filter size even in the noiseless case $\noiselev=0$. This is due to the non-commutativity of the operations of filtering a pattern and applying it a geometric transformation, i.e., if the geometric transformation involves a scale change, the transformation manifold of the filtered version of a pattern is not the same as the filtered version of the transformation manifold of that pattern. This introduces a further error in the alignment, in addition to the errors due to the curvature and the image noise. This is discussed in more detail in Lemma \ref{lem:dep_noiselev} in the proof of Theorem \ref{thm:dep_alerrbnd} in Appendix \ref{app:pf_thm_dep_alerrbnd}.\\

The results obtained in this section provide a characterization of the alignment error of the tangent distance method in multiscale image registration. The understanding of the behavior of the error in case of low-pass filtering provides a means for optimizing the performance of the tangent distance algorithm by adapting the filter size to the characteristics of the image data. In Section \ref{ch:tan_dist:ssec:conv_anly}, we examine the implications of our findings in the convergence of the hierarchical image registration algorithm.

\section{Convergence analysis of tangent distance}
\label{ch:tan_dist:ssec:conv_anly}

We now use the results obtained in Sections \ref{ch:tan_dist:sec:deriv_ub} and \ref{ch:tan_dist:ssec:dep_align_bnd} to analyze the convergence behavior of the tangent distance method in a general setting where the target image is a noisy transformed version of the reference image. We first examine the conditions under which the tangent distance converges to the correct solution at a single scale without filtering. We then generalize this to the convergence of the coarse-to-fine tangent distance method and propose some practical guidelines for optimal filter selection in each scale of the hierarchical alignment process.

\subsection{Convergence of the single-scale registration algorithm}

Consider that the tangent distance method is applied in an iterative manner, starting with the reference parameter vector $\tparref$ and then refining it gradually by taking the estimate from the previous iteration as the reference transformation parameter vector in each iteration. In this way, we obtain a sequence of estimates $\tparest^0, \tparest^1, \dots,  \tparest^k$ where the initial estimate is $\tparest^0 = \tparref$ and each subsequent estimate $\tparest^k$ is computed by linearizing the manifold around the point given by the previous parameter estimate $\tparest^{k-1}$.

First, based on the alignment error bound (\ref{eq:alerrbnd}) in Theorem \ref{thm:bnd_alignerrTD}, we define the following geometric constants on $\M(p)$:
\begin{equation}
\geocontwo:= \sup_{\tpar \in \tpardom}  \sqrt{\tr( [ \Gij (\tpar) ] )}
\ , 
\qquad \qquad
\geoconone:= \MsecderUB \ \sup_{\tpar \in \tpardom} \lambdamin^{-1} \ \big( [ \Gij (\tpar) ] \big) .
\label{eq:defn_geo_constants}
\end{equation}
The parameter $\geocontwo$ is a constant bounding the magnitude of the tangent vectors since it scales with the supremum of the tangent norms.  Similarly, the parameter $\geoconone$ is a normalized curvature constant, as the inverse of the metric tensor $ [ \Gij (\tpar) ] $ normalizes the inner products with tangent vectors in the least-squares estimation of transformation parameters in (\ref{eq:tdreg_soln}). The geometric constants  $\geocontwo$ and $\geoconone$ thus bound the magnitudes of the first-order and second-order variations of the manifold. 

In the next theorem, we focus on a single-scale setting where no filtering is done throughout the iterations. We state conditions guaranteeing that the estimates $\tparest^0, \tparest^1, \dots,  \tparest^k$ converge to the optimal transformation parameters $\tparopt$.

\begin{theorem}
Let the product of the noise level $\noiselev$ and the curvature constant $\geoconone$ be upper bounded as 
%$\geoconthree$ be a constant such that
%%
%\begin{equation}
%\geoconthree < \frac{1}{ d \, \geoconone}.
%\label{eq:cond_geoconthree}
%\end{equation}
%
%If the noise level $\noiselev$ is bounded as
%
\begin{equation}
\noiselev \, \geoconone < \frac{ 1} {d}.
\label{eq:cond_noise_conv}
\end{equation}
Furthermore, let us assume that the initialization of the tangent distance algorithm is such that 
\begin{equation}
\| \tparopt - \tparref \| < \frac{2}{ \geocontwo} \left( \frac{1}{d \geoconone}-   \noiselev  \right).
\label{eq:cond_inisol_conv}
\end{equation}
Then, the successive estimates given by the iterative application of the tangent distance method at a single scale converge to the optimal solution $\tparopt$, i.e.,
\begin{equation*}
\lim_{k\rightarrow \infty} \tparest^k = \tparopt.
\end{equation*}
\label{thm:conv_sin_scale}
\end{theorem}

Theorem \ref{thm:conv_sin_scale} is proved in Appendix \ref{app:tan_dist:sec:pf_conv_sin_scale} by using the error bound in Theorem \ref{thm:bnd_alignerrTD}. Theorem \ref{thm:conv_sin_scale} can be interpreted as follows. First, we observe from the condition in (\ref{eq:cond_noise_conv}) that  the noise level - curvature product must be below a certain level to recover the correct solution.\footnote{Note that, if the transformation manifold is defined such that transformed patterns are normalized, i.e., $\| \ptran \|=1$, then the metric tensor can be shown to be given by $\Gij(\tpar)=-\langle \ptran, \derijptran \rangle$. In this case, the constants $\geocontwo$ and $\geoconone$ defined as in (\ref{eq:defn_geo_constants}) satisfy $\geocontwo \leq \sqrt{d \MsecderUB}$ and $\geoconone \geq 1/d$. The bound on the noise level - curvature product in (\ref{eq:cond_noise_conv}) then requires that $\noiselev < 1$, i.e., the noise level should be smaller than the norm of the transformed patterns, or, the ``radius'' of the manifold.} It has been seen in Theorem \ref{thm:bnd_alignerrTD} that the alignment error is affected by both the manifold nonlinearity and the noise level. The condition  (\ref{eq:cond_noise_conv}) thus excludes the case where both the curvature and the noise level take large values, in order to ensure that the tangent distance method yields an accurate estimation.

%\[
%\noiselev \leq O\left( \frac{1}{\MsecderUB} \right),
%\]
%

Next, the inequality (\ref{eq:cond_inisol_conv}) implies that the accuracy of the initial solution must satisfy
$
\| \tparopt - \tparref \| \leq O\left( \MsecderUB^{-1} - \noiselev\right).
$
%
%The dependence of  $\noiselev$ and $\| \tparopt - \tparref \|$ on the manifold curvature demonstrates the effect of the nonlinearity of the manifold on the convergence behavior of the tangent distance method.
This condition requires the initial alignment error to be inversely proportional to the manifold curvature in a noiseless setting. Meanwhile, in a noisy setting, the increase in the noise level also brings a restriction on the accuracy of the initial solution $\tparref$ in order to preserve the convergence guarantee. In particular, the initialization error $\| \tparopt - \tparref \|$ must decrease linearly with the increase in the noise level $\noiselev$. The overall dependence of the initialization error $\| \tparopt - \tparref \|$ on $\MsecderUB$ and $\noiselev$ is intuitive in the sense that, as the curvature of the manifold approaches $0$, the accuracy of the linear approximation of the manifold increases, and the tangent distance method can recover the correct solution for arbitrarily large values of the initialization error even in the presence of noise. 

\subsection{Convergence of the coarse-to-fine registration algorithm}

We now study the convergence of the tangent distance method when it is implemented in a hierarchical, coarse-to-fine manner, with image filtering at each successive level. Let the estimation $\tparest^k$ be obtained by linearizing the manifold around the point corresponding to the parameter $\tparest^{k-1}$ as above. Consider, however, that in iterations $1, 2, \dots, k$, the reference and the target images are filtered with low-pass Gaussian filters of size $\rho_1, \rho_2, \dots, \rho_k$.

Before stating our result on the convergence of the hierarchical registration algorithm, we first define an effective noise level parameter. Recall from Section \ref{ch:tan_dist:ssec:dep_align_bnd} that geometric transformations that involve a scale change do not commute with low-pass filtering. This introduces an additional increase in the alignment error and can be modeled as a secondary source of noise in the alignment.\footnote{In particular, we show in Lemma \ref{lem:dep_noiselev} in Appendix \ref{app:pf_thm_dep_alerrbnd} that the distance between the filtered target pattern and the transformation manifold of the filtered reference pattern is of $O\left( (\noiselev+1) (1+\rho^2)^{-1/2} \right)$ for transformations with scale changes, and $O\left( \noiselev (1+\rho^2)^{-1/2} \right)$ for transformations without scale changes.} In order to model this phenomenon, we define an effective noise level parameter $\noiseeff$ such that 
\begin{equation*}
\noiseeff=\bigg\{ 
\begin{array} {l}
\noiselev + \noiselev_s \, \, \text{    if the transformation model includes a scale change}\\
\noiselev \, \, \text{   otherwise }  
\end{array}
\end{equation*}
where $\noiselev_s$ is a constant that represents the secondary noise term due to the non-commutativity of filtering and scaling.

We are now ready to present some conditions that guarantee that the hierarchical alignment process with the tangent distance method converges to the correct solution. 
%In order to analyze these, we first build on Lemma \ref{lem:dep_noiselev} and the proof of Theorem \ref{thm:dep_alerrbnd}. 

\begin{corollary}
\label{thm:conv_td_hier}
Let the product of the effective noise level $\noiseeff$ and the curvature constant $\geoconone$ be upper bounded as follows
\begin{equation}
\noiseeff \, \geoconone < \frac{1}{d}.
\label{eq:cond_noise_convfilt}
\end{equation}
Furthermore, let the initialization error of the hierarchical tangent distance algorithm be bounded as 
\begin{equation}
\| \tparopt - \tparref \| < \frac{2}{ \geocontwo} \left( \frac{1}{d \geoconone}-   \noiseeff  \right).
\label{eq:cond_init_err_filt}
\end{equation}
Then, if the filter size $\rho_k$ in each iteration $k$ is chosen as $\rho_k \in [0, \rho_k^{\max}]$, where
\begin{numcases}{\rho_k^{\max} = }
 \sqrt{ \frac{ \geocontwo \| \tparopt - \tparest^{k-1} \| }{ 2 \, \noiseeff  } - 1} 
& if $\| \tparopt - \tparest^{k-1} \| \geq \frac{2 \, \noiseeff }{ \geocontwo}$
\label{eq:opt_choice_rhok}
\\
0
& if $ \| \tparopt - \tparest^{k-1} \| < \frac{2 \, \noiseeff  }{ \geocontwo}$
\label{eq:opt_choice_rhok0}
\end{numcases}
the successive estimates of the hierarchical tangent distance method converge to the optimal solution $\tparopt$, i.e.,
\begin{equation*}
\lim_{k \rightarrow \infty} \tparest^k = \tparopt.
\end{equation*}
\end{corollary}

The proof of Corollary \ref{thm:conv_td_hier} is given in Appendix \ref{sec:pf_thm_conv_td_hier}. In the proof, we first derive the ``optimal'' filter size selection strategies given in (\ref{eq:opt_choice_rhok})-(\ref{eq:opt_choice_rhok0}), which are computed by minimizing an approximate expression for the alignment error represented as a function of the filter size $\rho$ following Theorem \ref{thm:dep_alerrbnd}. The convergence guarantee then follows from the observation that the above selection of the filter size yields an error that is not larger than the error obtained by applying no filtering. Finally, the same steps as in the proof of Theorem \ref{thm:conv_sin_scale} are applied to obtain the stated result.

The suggestion for the filter size selection in (\ref{eq:opt_choice_rhok})-(\ref{eq:opt_choice_rhok0}) shows that $\rho$ must be chosen large if the current estimation error $\| \tparopt - \tparest^{k-1} \|$ at the beginning of iteration $k$ is large. The noise level of the target image also influences the optimal filter size. It must be chosen inversely proportional to the square root of the noise level, because of the increase of the alignment error with filtering. These provide a  justification of the strategy of reducing the filter size gradually in coarse-to-fine alignment, since the successive estimates $\{ \tparest^k \}$  approach the optimal solution progressively and the estimation error $\|  \tparopt - \tparest^{k} \|$ decreases throughout the iterations of the hierarchical alignment algorithm. In particular, the filter size selection strategies in (\ref{eq:opt_choice_rhok})-(\ref{eq:opt_choice_rhok0}) suggest that, when the estimation error decreases below a threshold that depends on the noise level, it is better to stop filtering the images and to use their original versions in the alignment process.

In a practical implementation of the tangent distance method, it is not easy to exactly compute the optimal value of the filter size in (\ref{eq:opt_choice_rhok})-(\ref{eq:opt_choice_rhok0}) since the alignment error $  \| \tparopt - \tparest^{k} \| $ in an arbitrary iteration is not exactly known. However, using our results, we can deduce a suitable rule for updating the filter sizes $\rho_k$ in practice. First observe that, from Theorem \ref{thm:bnd_alignerrTD}, the alignment error at iteration $k$ is bounded as $\| \tparest^k - \tparopt   \| \leq \alerrbnd_k $, where 
\begin{equation*}
\alerrbnd_k =  \MsecderUB \ \lambdamin^{-1} \ \big( [ \Gij (\tparest^{k-1}) ] \big) 
\left( \half \,  \sqrt{\tr( [ \Gij (\tparest^{k-1}) ] )} \ \| \tparopt - \tparest^{k-1}  \|_{1}^2
+  \sqrt{d} \ \noiselev  \   \|  \tparopt - \tparest^{k-1}   \|_1
\right).
\end{equation*}
If the noise level and the distance between the reference and optimal transformation parameters are sufficiently small to satisfy (\ref{eq:cond_noise_convfilt}) and (\ref{eq:cond_init_err_filt}), the alignment error upper bounds $\{ \alerrbnd_k \}$ in the iterative registration process decay at a geometric rate such that
\begin{equation}
\alerrbnd_k \leq \alpha \, \alerrbnd_{k-1},
\label{eq:alerr_dec_rate}
\end{equation}
where 
\begin{equation}
\alpha = \half d \geocontwo \geoconone \alerrbnd_0 + d \ \noiseeff \geoconone <1
\label{eq:defn_decay_fact}
\end{equation}
and $\alerrbnd_0 = \| \tparopt - \tparref \| $ denotes the initialization error (see the proof of Corollary \ref{thm:conv_td_hier} in Appendix \ref{sec:pf_thm_conv_td_hier}). Now, from (\ref{eq:alerr_dec_rate}), the alignment error bound $\alerrbnd_k$ in iteration $k$ is bounded as $\alerrbnd_k \leq \alpha^k \alerrbnd_0$, which gives
\begin{equation*}
\| \tparopt - \tparest^{k}  \| \leq \alerrbnd_k \leq \alpha^k \alerrbnd_0 = \alpha^k \| \tparopt - \tparest^{0}  \|.
\end{equation*}
Due to the relation $\| \tparopt - \tparest^{k}  \| \leq  \alpha^k \| \tparopt - \tparest^{0}  \|$ for all $k$, one may expect the actual alignment errors $\| \tparopt - \tparest^{k}  \| $ to decay at the same rate $\alpha$ as well. Thus, a reasonable approximation for the relation between the alignment errors in adjacent iterations is given by
\begin{equation*}
 \| \tparopt - \tparest^{k} \| \approx \alpha \,  \| \tparopt - \tparest^{k-1} \| .
\end{equation*}
Applying this approximation in the expressions of the optimal filter sizes in (\ref{eq:opt_choice_rhok})-(\ref{eq:opt_choice_rhok0}), we then get the following update for the filter size 
\begin{equation}
\rho_k \approx \sqrt{\alpha} \, \rho_{k-1}.
\label{eq:rhok_update_prac}
\end{equation}
Notice that, at the early stages of the alignment, the alignment error is large. Then, ignoring the subtractive constant in (\ref{eq:opt_choice_rhok}) yields the above approximation. Meanwhile, in the late stages of the iterative alignment, the error is small; the geometric decay of the filter sizes in the update rule (\ref{eq:rhok_update_prac}) makes $\rho_k$ approach $0$, which approximates well the selection $\rho_k=0$ in (\ref{eq:opt_choice_rhok0}).

The filter size update rule in (\ref{eq:rhok_update_prac}) is in agreement with the common practice of reducing the filter size with a geometric decay. While it is typical to reduce the filter size by a factor of $\sqrt{\alpha} = 1/2$ in the implementation of hierarchical image registration algorithms \cite{VasconcelosL05}, \cite{Burt83}, we can now reinterpret the selection of the factor $\alpha$ in the light of our results. First, an immediate consequence of the linear proportion between the decay factor $\alpha$ in  (\ref{eq:defn_decay_fact}) and the curvature parameter $\geoconone$ is that $\alpha$ should increase with manifold nonlinearity. This is in agreement with the expectation that applying large filters throughout the iterations improves the accuracy of the linear approximation of the manifold. Similarly, the decay factor $\alpha$ is seen to increase linearly with the initialization error $\alerrbnd_0$. This shows that adapting $\alpha$ to the accuracy of the initial solution helps to mitigate the influence of the initialization error, which propagates and affects the estimates of the algorithm throughout the iterations. Finally, regarding the dependence of the filter update strategy on the noise level, we observe the following. From (\ref{eq:opt_choice_rhok}), we observe that the initial filter size $\rho_1$ in iteration $1$ must be chosen as
\begin{equation*}
\rho_1 \approx \sqrt{  \frac{ \geocontwo  \alerrbnd_0}{2 \noiseeff} }.
\end{equation*}
Therefore, at small values of the noise level $\noiseeff$, one can begin with a relatively large filter size $\rho_1$ in the first iteration. The decay factor $\alpha$ takes a small value in this case, which is useful for speeding up the convergence of the algorithm. On the other hand, at high noise levels, the above expression for $\rho_1$ suggests that the initial filter size should be chosen small in order to control the influence of noise on the alignment accuracy. The factor $\alpha$ becomes larger in this case; therefore, the decay in the filter size between adjacent iterations needs to be slower. \\

We have studied in this section the convergence of the multiscale tangent distance method and shown that the convergence of the algorithm is guaranteed if the noise level, the curvature and the initialization error are sufficiently small. Moreover, we have shown that, in the coarse-to-fine tangent distance method, the optimal choice of the filter size depends on the data and transformation model characteristics. Providing an insight into the performance of multiscale image registration, our results can be used in devising effective tools for image registration and analysis.

\section{Analysis of the error in classification problems}
\label{ssec:class_td}

We have so far studied the registration performance of the tangent distance method. Meanwhile, the tangent distance method is also used commonly in image analysis problems for the transformation-invariant estimation of the similarity between a query image and a set of image manifold models representing different classes. A typical similarity measure is the distance between the query image and the class-representative transformation manifolds.  Since the distances to the manifolds are computed by estimating the projection of the query image onto the manifolds, the accuracy of the distance estimation is highly influenced by the accuracy of the estimation of the transformation parameters. The classification performance is thus quite related to the registration performance.

In this section, we study the link between the image classification and registration problems and extend our results on the registration analysis to study the performance of the tangent distance method in image classification. Consider a setting with $M$ class-representative patterns $\{p^m \}_{m=1}^M$ whose transformation manifolds 
\begin{equation*}
\M(p^m)=\{ \ptran^m: \tpar \in \tpardom  \} \subset L^2(\Rsq)
\end{equation*}
are used for the classification of query patterns $\targetp \in L^2(\Rsq) $ in the image space. We assume that the correct class label $l(\targetp)$ of a query pattern $\targetp$ is given by the class label of the manifold $\M(p^m)$ with smallest distance to it, i.e.,
\begin{equation}
l(\targetp)= \arg \min_{m \in \{ 1, \dots, M \} } \| \targetp - p_{\tparopt^m}^m \|
\label{eq:defn_true_classlab}
\end{equation}
where
\begin{equation*}
\tparopt^m = \arg  \min_{\lambda \in \Lambda} \| \targetp - \ptran^m  \|
\end{equation*}
is the optimal transformation parameter vector corresponding to the projection of $\targetp$ on $\M(p^m)$. 

Our purpose is then to study in this context the performance penalty when the class label of a query pattern is estimated by employing first-order approximations of the manifolds. Obviously, if the transformation parameters are estimated with an iterative application of the tangent distance method (at a single scale or in a coarse-to-fine manner), the convergence guarantees   to the optimal solution established in Theorem \ref{thm:conv_sin_scale} and Corollary \ref{thm:conv_td_hier} ensure that the target pattern be correctly classified. Hence, in this section, we focus on the accuracy of classifying a query image with a one-step application of the tangent distance method, i.e., by estimating the transformation parameters $\{ \tparopt^m \}$ with a single linearization of each manifold, possibly by filtering the target and reference images. We study the performance of classification in this setting and its dependence on the choice of the filter size. 

Let $\tparest^m$ denote the estimate of $\tparopt^m$ computed with the tangent distance method as in (\ref{eq:tdreg_soln}) by linearizing the manifold $\M(p^m)$ around a reference point with parameter vector $\tparref^m$. The class label of $\targetp$ is then estimated with the tangent distance method as follows\footnote{Note that the class label of a query image can also be estimated by comparing its distance to the first-order approximation $\Mlin_{\tparref^m}(p^m)$ of each manifold defined in (\ref{eq:defn_mlin_appx}). While Simard et al.~use this subspace distance for classification \cite{Simard00}, the estimate in (\ref{eq:est_classlab_td}) is also commonly used in image analysis problems (e.g., as in \cite{VasconcelosL05}). We base our analysis on the definition in (\ref{eq:est_classlab_td}) since it is likely to give more accurate estimates, especially when it is generalized to a multiscale setting as in (\ref{eq:est_classlab_td_multis}).} 
\begin{equation}
\tilde l(\targetp) = \arg \min_{m \in \{ 1, \dots, M \} } \| \targetp -  p_{\tparest^m}^m \|
\label{eq:est_classlab_td}.
\end{equation}

Comparing the estimated class label in (\ref{eq:est_classlab_td}) and the true class label in (\ref{eq:defn_true_classlab}), it can be observed that the performance of classification depends on the accuracy of the estimation of the transformation parameters. In particular, if the estimate $\| \targetp -  p_{\tparest^m}^m \|$ of the distance between the query pattern and the manifold is sufficiently close to the true manifold distance $ \| \targetp - p_{\tparopt^m}^m \|$ for each one of the manifolds, the estimated class label $\tilde l(\targetp)$ in (\ref{eq:est_classlab_td}) is the same as the true class label $l(\targetp)$. Based on this observation, we study the classification performance of the tangent distance method as follows. First, given a reference pattern $p$ and a target pattern $\targetp$, we derive a relation between the distance estimation error
\begin{equation*}
\big| \|  \targetp - p_{\tparopt} \| -    \| \targetp -  p_{\tparest} \|  \big|
\end{equation*}
and the alignment error $\| \tparopt - \tparest  \| $ in the parameter domain in the following lemma.
\begin{lemma}
\label{lem:dist_est_error}
The distance estimation error of the tangent distance method can be upper bounded in terms of its alignment error as 
\begin{equation}
\big| \|  \targetp - p_{\tparopt} \| -    \| \targetp -  p_{\tparest} \|  \big| \leq  \MderUB \,  \|  \tparopt - \tparest  \|_1,
\label{eq:dist_err_bnd}
\end{equation}
where $\MderUB$ denotes the supremum of the tangent norms on $\M(p)$ 
\begin{equation}
\MderUB := \max_{i= 1, \dots, d}  \  \sup_{\tpar \in \tpardom} \| \deriptran  \|.
\label{eq:defn_suptan}
\end{equation}

\end{lemma}
The proof of Lemma \ref{lem:dist_est_error} is given in Appendix \ref{sec:pf_lem_dist_est_error}. Lemma  \ref{lem:dist_est_error} provides a link between the accuracy of the alignment measured in the parameter domain $\tpardom$, and in the ambient space $L^2(\Rsq)$, respectively. It shows that the distance estimation error can be upper bounded with a linear function of the alignment error.

The relation in (\ref{eq:dist_err_bnd}) suggests that one may expect the classification performance of the tangent distance method to vary linearly with the accuracy of alignment in the parameter domain. In order to construct a more precise relation, we now consider a setting where the query images of class $m$ have a distribution that is concentrated around the manifold $\M(p^m)$. We then examine the probability of correctly classifying $\targetp$ based on the distance estimates given by the tangent distance method.

Using the notation of Section \ref{ssec:prob_form}, let
\begin{equation*}
\noiselev_j = \| \targetp - p^j_{\tparopt^j}  \|
\end{equation*}
denote the deviation of a query image $\targetp$ from the manifold $\M(p^j)$ of class $j$. Furthermore, let $\targetp$ belong to class $m$. The distance of $\targetp$ to $\M(p^m)$ is the smallest among the distances of $\targetp$ to all manifolds; therefore,
$
\noiselev_m < \noiselev_j
$
for all $j \neq m$. Let us assume that the distributions of the images belonging to different classes have bounded and non-intersecting supports around the manifolds, so that the classification rule in (\ref{eq:defn_true_classlab}) always gives the true class label. We can then define the following parameters. Let
\begin{equation*}
\maxdist_m := \sup_{q: \ l(\targetp) = m } \left\{  \| \targetp - p^m_{\tparopt^m} \| \right\}
\end{equation*}
denote the maximal distance of query patterns of class $m$ to the manifold $\M(p^m)$ of their own class and 
\begin{equation*}
\sepmarg := \min_{m=1, \dots, M; \ j\neq m} \ \ \inf_{q: \ l(\targetp) = m} \left\{ \| \targetp - p^j_{\tparopt^j} \|  -  \| \targetp - p^m_{\tparopt^m} \|   \right\}
\end{equation*}
define a distance margin that is a measure of the minimum separation between different classes. Finally, let $\MderUB_m$ and  $\MsecderUB_m$ denote the suprema of the tangent norm and the curvature on the manifold $\M(p^m)$, as defined in (\ref{eq:defn_suptan}) and (\ref{defn:sup_curvature}) respectively. We then have the following result, which provides an upper bound for the probability of misclassifying a target image of class $m$.
\begin{theorem}
\label{thm:misclass_prob}
Let $\targetp$ be a query pattern of class $m$. Assume that the optimal transformation parameters $\tparopt^m$ aligning $\targetp$ with $p^m$ are within a $\Delta$-neighborhood of the reference transformation parameters $\tparref^m$ around which $\M(p^m)$ is linearized, such that
\begin{equation*}
\| \tparopt^m - \tparref^m \|_1 \leq \Delta.
\end{equation*}
Then, the probability of misclassifying $\targetp$ with the tangent distance method is upper bounded as
\begin{equation*}
P\left(\tilde l(\targetp) \neq l(\targetp) \right) \leq \frac{(M-1)}{\sepmarg}  \MderUB_m \, \sqrt{d} \, \MsecderUB_m \ \lambdamin^{-1} \ \big( [ \Gij^m (\tparref^m) ] \big) 
\left( \half \,  \sqrt{\tr( [ \Gij^m (\tparref^m) ] )} \ \Delta^2
+  \sqrt{d} \ \maxdist_m  \   \Delta \right)
\end{equation*}
where $d$ is the dimension of the manifolds and $[ \Gij^m (\tparref^m) ] $ denotes the metric tensor of manifold $\M(p^m)$ at the point corresponding to $\tparref^m$.
\end{theorem}
The proof of Theorem \ref{thm:misclass_prob} is given in Appendix \ref{pf:thm_misclass_prob}. The above result is obtained by upper bounding the probability of misclassification in terms of the distance estimation error. The distance estimation error is linked to the alignment error in the parameter domain using Lemma \ref{lem:dist_est_error}, which is then upper bounded using Theorem \ref{thm:bnd_alignerrTD}.  

Theorem \ref{thm:misclass_prob} shows how the probability of misclassification when the manifold distances are estimated with the tangent distance method, depends on the geometric properties of the manifolds and on the deviation $\Delta$ between the reference transformation parameters $\tparref^m$ used in the linearization of the manifold and the optimal transformation parameters $\tparopt^m$ corresponding to the projection of $\targetp$ onto the manifold. In particular, for any non-intersecting and bounded distribution of class samples, the misclassification probability increases at most linearly with the increase in the manifold curvature and the maximal distance of the images to their own representative manifold. The deviation $\Delta$ between the parameters used in the linearization and the parameters corresponding to the exact projection affects the misclassification probability due to its influence on the  alignment accuracy. We also observe that better separation of manifolds (i.e.,  increase in the distance margin $\sepmarg$) reduces the probability of misclassification, as expected.

We now discuss the classification of images with the tangent distance method in a multiscale setting and study the selection of the filter size in order to minimize the misclassification probability. Consider that the transformation parameters are estimated by filtering the query image $\hattargetp$ and the reference images $\hatp^m$. From (\ref{eq:tdreg_soln}), the following estimates $\{ \hattparest^m \}$ are obtained for the classes $m= 1, \dots, M $ by registering the query image on each class manifold with the tangent distance method 
\begin{equation*}
(\hattparest^i)^m = (\tparref^i)^m +  (\hatinvGij)^m(\tparref^m)  \langle \hattargetp - \hatptranref^m, \hatderjptranref^m  \rangle.
\end{equation*} 
Here $ \hatGij^m$ and $\hatderjptranref^m$ are respectively the metric tensor and the tangent vectors on the manifold $\M(\hatp^m)$. Once the transformation parameters  are estimated, we assume that the unfiltered versions of the reference images and the query image are used in the computation of the actual distances to the manifolds for estimating the class label of the query image. It is preferable to compare the distances in the original image space rather than the space of filtered images, as it yields more accurate estimates. The class label estimate of the query pattern is thus given by
\begin{equation}
\tilde l(\targetp) = \arg \min_{m \in \{ 1, \dots, M \} } \| \targetp -  p_{\hattparest^m}^m \|
\label{eq:est_classlab_td_multis}.
\end{equation}

Repeating the steps in the proof of Theorem \ref{thm:misclass_prob} by replacing the estimates $\{ \tparest^m \}$ with the ones $\{ \hattparest^m \}$ obtained after filtering the reference and target patterns, one can upper bound the misclassification probability as 
\begin{equation}
P\left(\tilde l(\targetp) \neq l(\targetp) \right) \leq \frac{(M-1)}{\sepmarg}  \MderUB_m \, \sqrt{d} \, \hatMsecderUB_m \ \lambdamin^{-1} \ \big( [ \hatGij^m (\tparref^m) ] \big) 
\left( \half \,  \sqrt{\tr( [ \hatGij^m (\tparref^m) ] )} \ \Delta^2
+  \sqrt{d} \ \hatmaxdist_m  \   \Delta \right)
\label{eq:misclass_prob_filt}
\end{equation}
when the filtered images are used for estimating the transformation parameters. We have neglected the perturbation $\| \tparopt - \hattparopt \|$ due to filtering in the projection of patterns onto the manifold. The above expression for the misclassification probability is in the same form as the alignment error bound in (\ref{eq:hatalerrbnd}); they only differ by a multiplicative factor (note, however, that the value of this factor depends on the geometric properties of the manifolds through the parameters $\MderUB_m$ and $\sepmarg$). Therefore, the misclassification probability bound has the same non-monotonic variation with the filter size as the alignment error. Moreover, the optimal value of the filter size that minimizes the alignment error is a minimizer of the misclassification probability upper bound as well. In an image classification application where a one-step linear approximation of the manifolds is employed, one may thus choose the optimal filter size by minimizing the alignment error. The model parameters should then be selected with respect to the expected characteristics of the data. The maximal distance $\maxdist_m$ is related to the internal variation (noise level) of the data samples within the same class and depends on how well the reference pattern $p^m$ approximates the samples of its own class, whereas the parameter $\Delta$ can be set according to the maximum amount of transformation that the data samples are likely to undergo in the application at hand.

%The results of this section show that the likeliness of misclassifying a qury pattern is proportional to the distance estimation error, while the relation in (\ref{eq:dist_err_bnd}) shows that the distance estimation error increases linearly with the alignment error. Therefore, in an image classification application, the misclassification rate is expected to increase linearly with the alignment error of the query images

%In this study, we focus on the estimate in (\ref{eq:est_classlab_td}) as it is expected to give more accurate results in a multiscale setting, provided that $\tparest^m$ are computed with the filtered versions of the images; however, .

%%%%%%%%%%%%%%%%%%%%%%%%%
%%%%%% EXPERIMENTAL RESULTS
%%%%%%%%%%%%%%%%%%%%%%%%%

\section{Experimental Results}
\label{sec:exp_img_align}

\subsection{Alignment of synthetic images}
\label{ssec:exp_align_synth}

We now present experimental results that illustrate our alignment error bounds. In all settings, we experiment on three different geometric transformation models, namely, (i) a two-dimensional translation manifold  
\begin{equation}
\M(p)=\{ \ptranlong: \tpar = (t_x, t_y) \in \tpardom \}, 
\label{eq:Mp_2d}
\end{equation}
(ii) a three-dimensional manifold given by the translations and rotations of a reference pattern
\begin{equation}
\M(p)=\{ \ptranlong: \tpar = (\overline{\theta}, t_x, t_y) \in \tpardom \},
\label{eq:Mp_3d}
\end{equation}
and a (iii) four-dimensional manifold generated by the translations, rotations and isotropic scalings of a reference pattern
\begin{equation}
\M(p)=\{ \ptranlong: \tpar = (\overline{\theta}, t_x, t_y, \overline{s}) \in \tpardom \}.
\label{eq:Mp_4d}
\end{equation}
In the above models, $t_x$ and $t_y$ represent translations in $x$ and $y$ directions, $\overline{\theta}$ denotes a rotation parameter, and $ \overline{s}$ is a  scale change parameter. The parameters $\overline{\theta}$ and $ \overline{s}$ are normalized versions of the actual rotation angle $\theta$ and scale change factor $s$, so that the magnitudes of the manifold derivatives with respect to $t_x$, $t_y$,  $\overline{\theta}$, and $ \overline{s}$ are proportional.

\begin{figure}[t]
\begin{center}
     \subfigure[]
       {\label{fig:rand_pat_realiz_orig}\includegraphics[height=3cm]{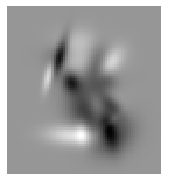}}
     \subfigure[]
       {\label{fig:rand_pat_realiz_noisy}\includegraphics[height=3cm]{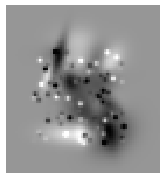}}
 \end{center}
 \caption{(a) A realization of the reference pattern randomly generated in the Gaussian dictionary. (b) Reference pattern corrupted with noise ($\noiselev=0.5$).}
 \label{fig:rand_pat_realiz}
\end{figure}

\begin{figure}[ht]
\begin{center}
     \subfigure[]
       {\label{fig:exp_2d_nu}\includegraphics[height=6cm]{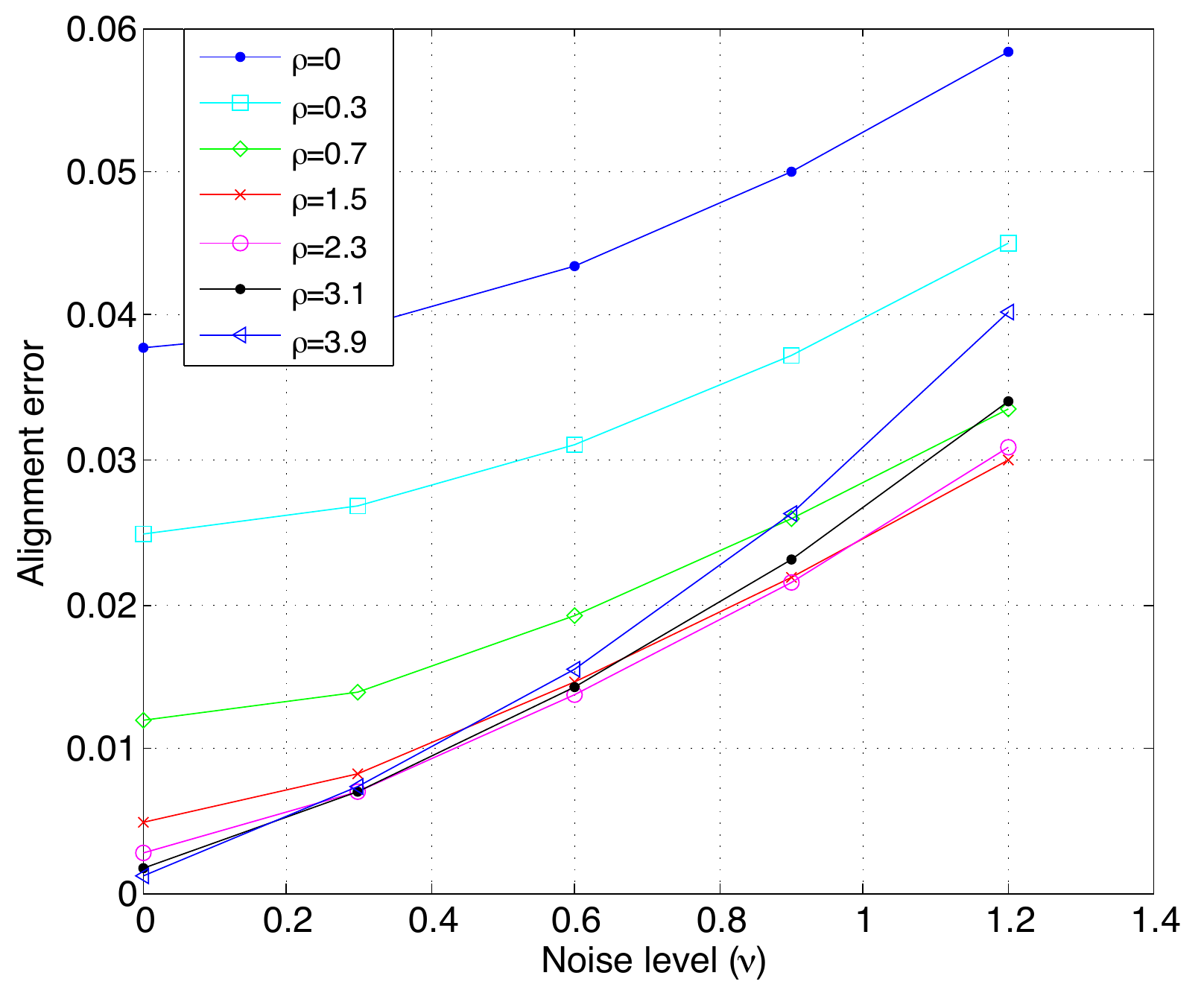}}
     \subfigure[]
       {\label{fig:theo_2d_nu}\includegraphics[height=6cm]{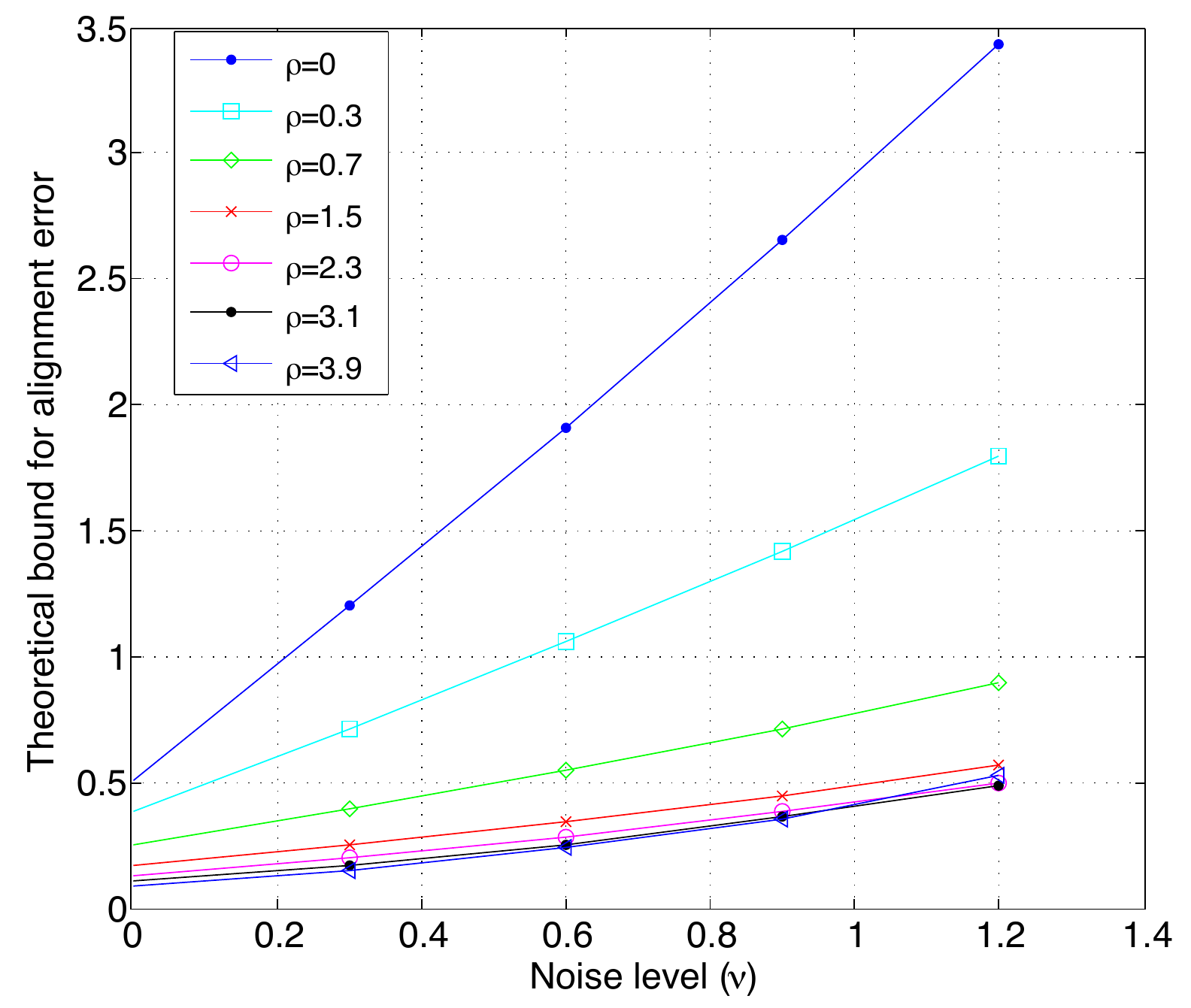}}
     \subfigure[]
       {\label{fig:exp_2d_rho}\includegraphics[height=6cm]{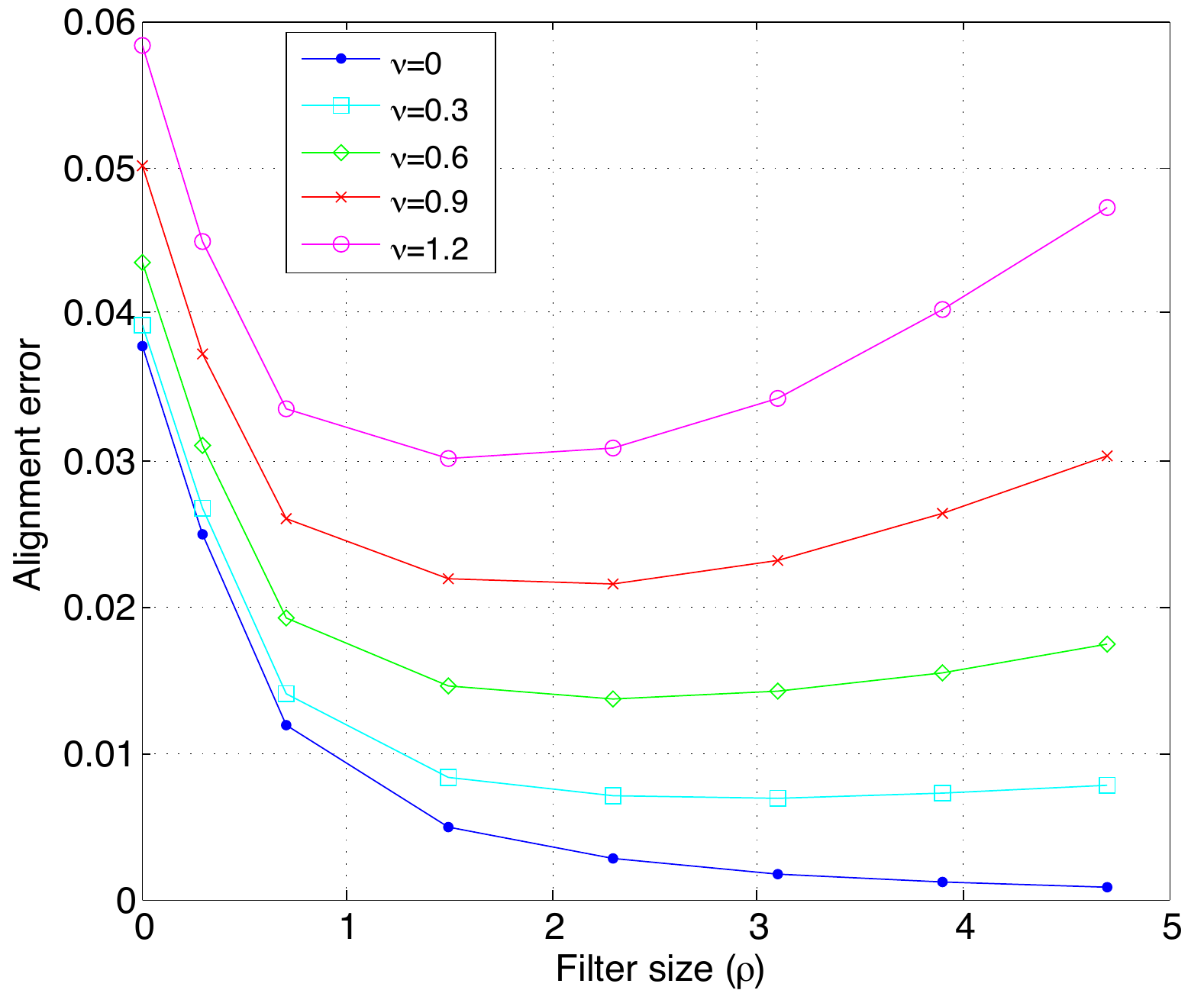}}
     \subfigure[]
       {\label{fig:theo_2d_rho}\includegraphics[height=6cm]{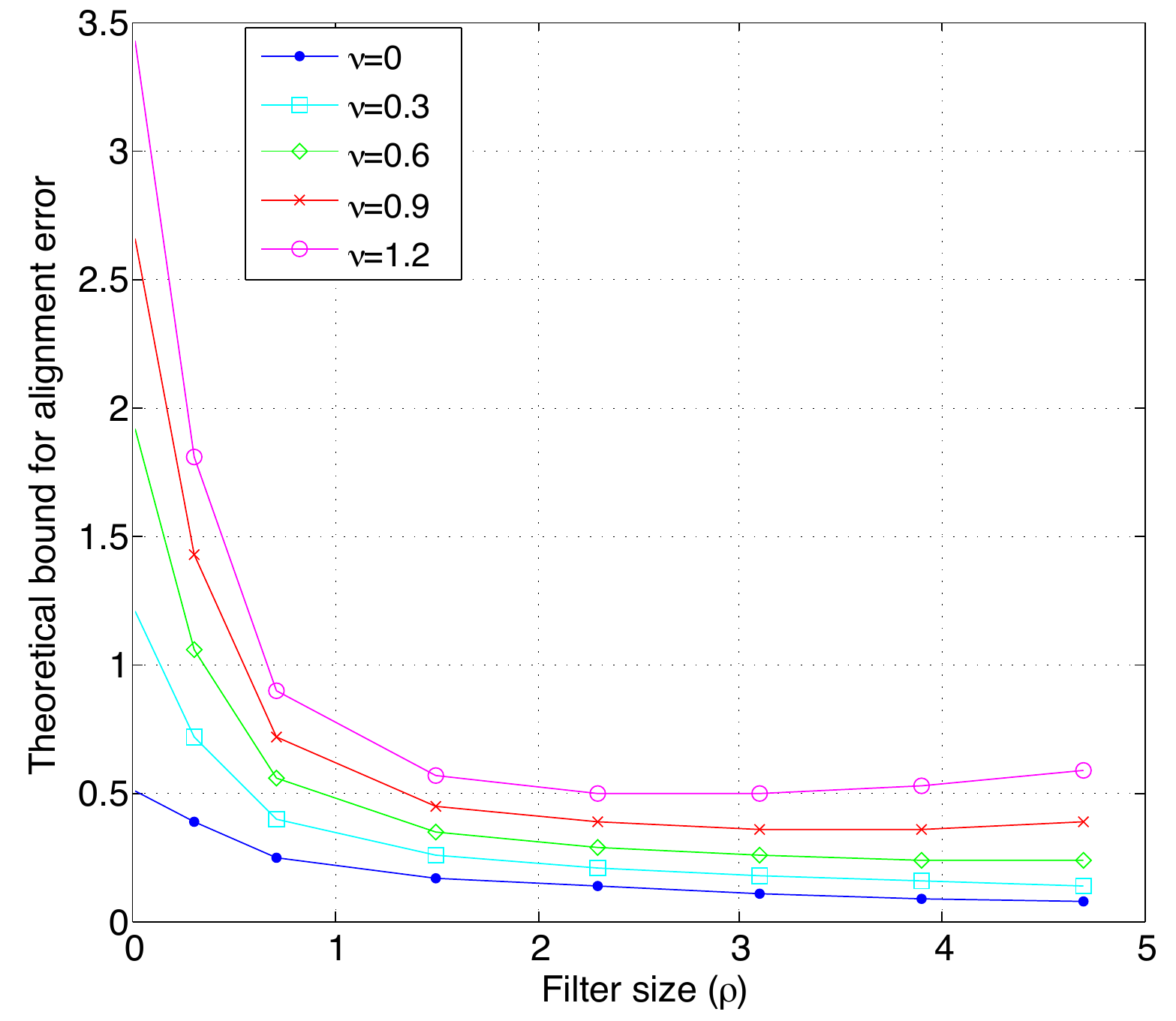}}
 \end{center}
 \caption{Alignment errors of random patterns for 2-D manifolds generated by translations.}
 \label{fig:rand_2d}
\end{figure}

\begin{figure}[ht]
\begin{center}
     \subfigure[]
       {\label{fig:exp_3d_nu}\includegraphics[height=6cm]{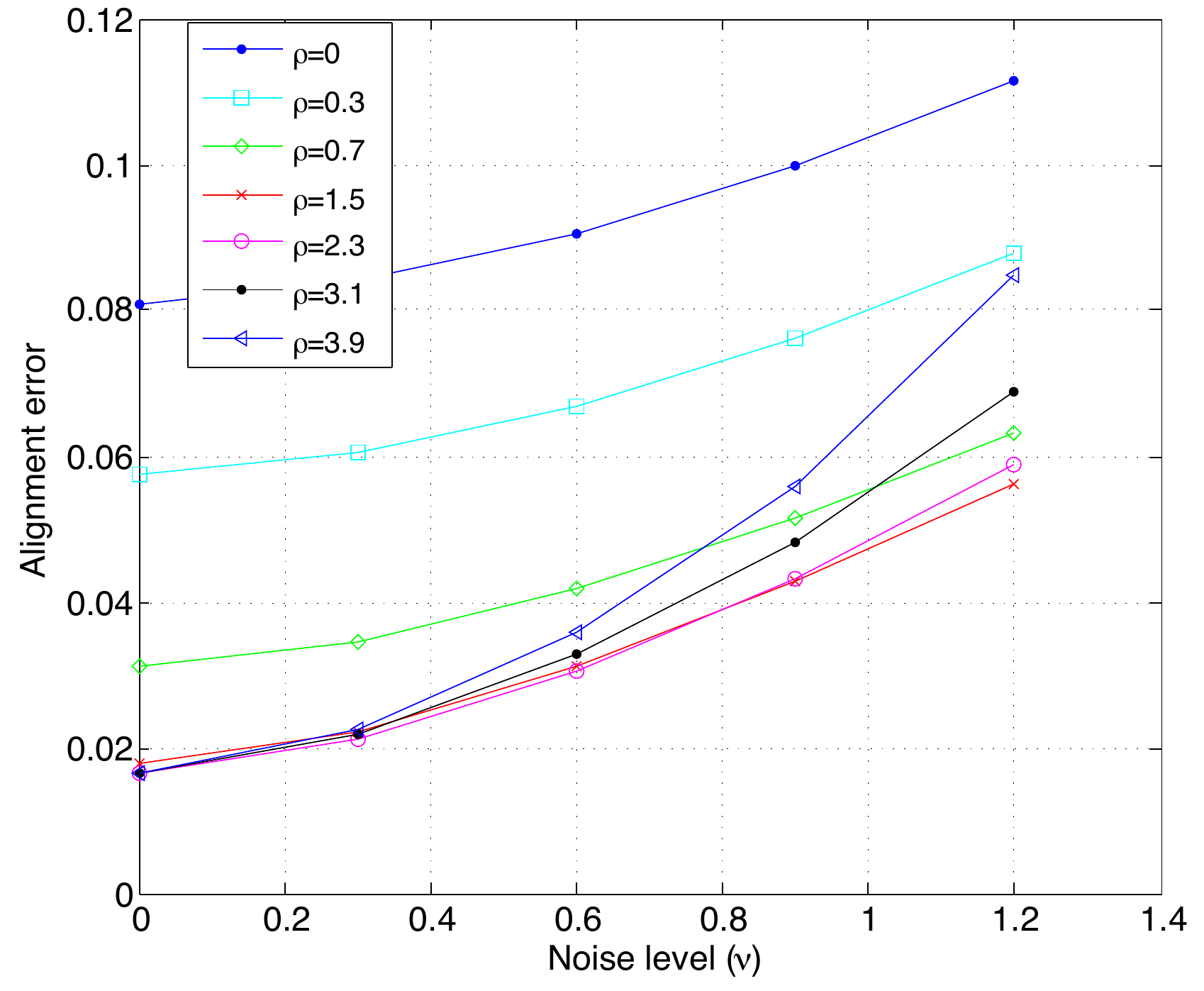}}
     \subfigure[]
       {\label{fig:theo_3d_nu}\includegraphics[height=6cm]{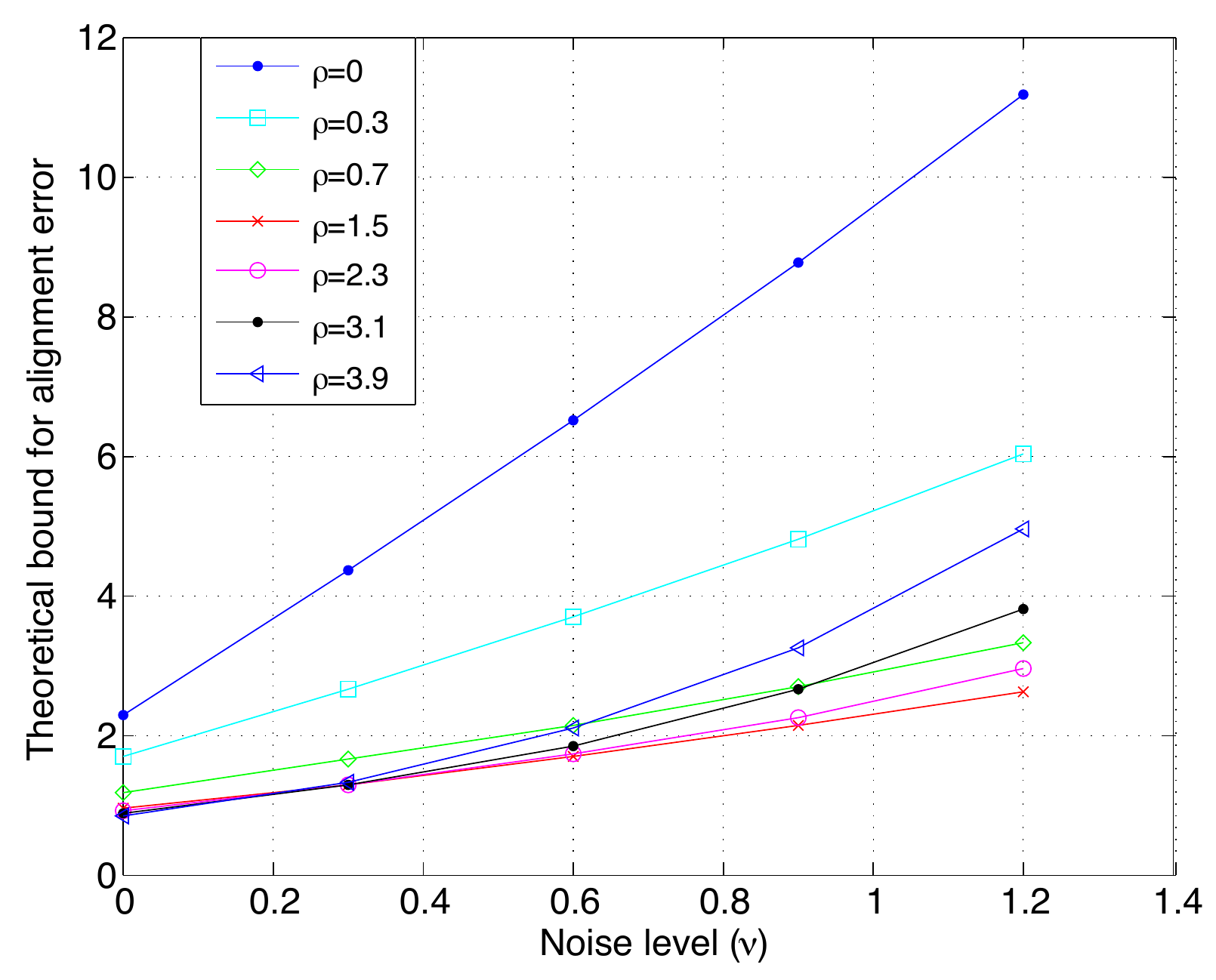}}
     \subfigure[]
       {\label{fig:exp_3d_rho}\includegraphics[height=6cm]{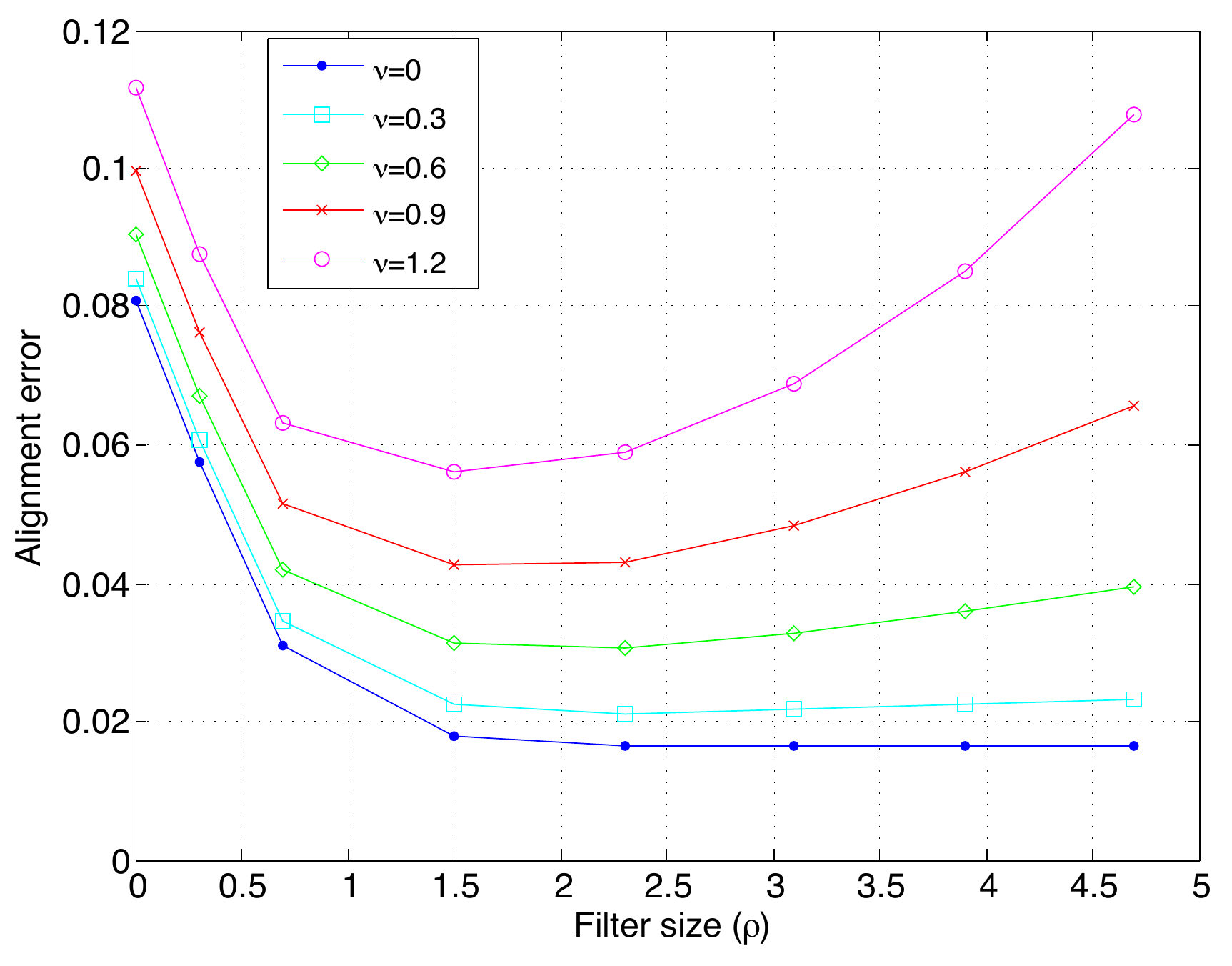}}
     \subfigure[]
       {\label{fig:theo_3d_rho}\includegraphics[height=6cm]{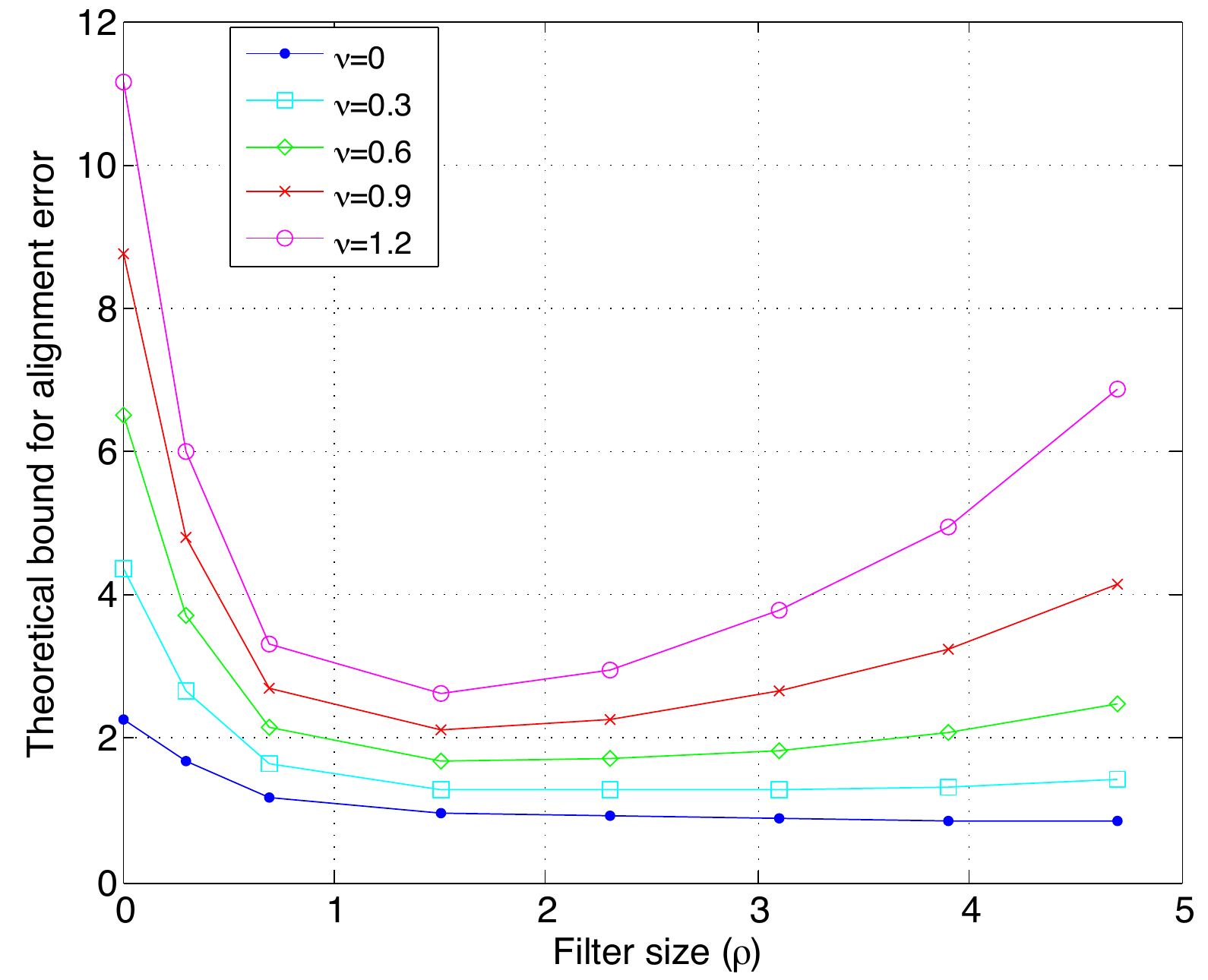}}
 \end{center}
 \caption{Alignment errors of random patterns for 3-D manifolds generated by translations and rotations.}
 \label{fig:rand_3d}
\end{figure}

In all experiments, several target patterns are generated from a reference pattern by applying a random geometric transformation according to the above models. The target patterns are then corrupted with additive noise patterns at different noise levels $\noiselev$. For each reference and target pattern pair $(p,\targetp)$, a sequence of image pairs $(\hatp, \hattargetp)$ are obtained by smoothing $p$ and $\targetp$ with low-pass filters with different kernel sizes $\rho$. Then, the target pattern $\hattargetp$ in each image pair is aligned with the reference pattern $\hatp$ using the tangent distance method, where the reference parameter vector $\tparref$ is taken as identity such that $\hatptranref=\hatp$. The experimental alignment error is measured as the parameter domain distance $\| \hattparest -  \hattparopt \|$ between the optimal transformation parameter vector $\hattparopt$ and its estimate $\hattparest$. Then, the experimental alignment error is compared to its theoretical upper bound $\hatalerrbnd$ given in Theorem \ref{thm:bnd_alignerrTD}. In the computation of the theoretical bound in Theorem \ref{thm:bnd_alignerrTD}, the curvature parameter $\MsecderUB$ is estimated numerically by computing the value of the maximal second derivative magnitude on a sufficiently dense grid on the manifold.

In the first set of experiments, we experiment on 50 different synthetically generated reference patterns. We construct the patterns with 20 atoms that are randomly selected from the Gaussian dictionary $\mathcal{D}$ given by 
\begin{equation*}
\mathcal{D}  = \{ {\phi}_{\gamma}: \gamma = (\psi, \tau_x, \tau_y, \sigma_x, \sigma_y) \in \Gamma  \} \subset L^2(\mathbb{R}^2)
\end{equation*}
Here $\phi$ is a two-dimensional Gaussian mother function and $\gamma$ is a transformation parameter vector. The transformation parameters $\psi$, $\tau_x$, $\tau_y$, $\sigma_x$, and $\sigma_y$ correspond respectively to a 2-D rotation, translations in horizontal and vertical directions, and anisotropic scale changes in horizontal and vertical directions. Each atom $\phi_\gamma$ is obtained by applying the geometric transformation specified by $\gamma$ to the Gaussian mother function $\phi$.\footnote{In the proof of Theorem \ref{thm:dep_alerrbnd} in Appendix \ref{app:pf_thm_dep_alerrbnd}, we adopt a representation of patterns in this same Gaussian dictionary in order to derive the variation of the alignment error with the filter size. Since the convolution of two Gaussian functions is also a Gaussian function, the representation of patterns in terms of Gaussian atoms facilitates the study of the variation of the manifold derivatives with the filter size $\rho$. More details on the Gaussian dictionary $\mathcal{D}$ are available in Appendix \ref{app:tan_dist_anly_manderiv}.}

In the generation of the patterns, the atom  parameters are randomly drawn from the intervals $\psi \in [-\pi, \pi)$; $\tau_x, \tau_y \in [-4, 4]$; $\sigma_x, \sigma_y \in [0.3, 2.3]$; and the atom coefficients are randomly selected within the range $[-1, 1]$. Then, for each one of the models (\ref{eq:Mp_2d})-(\ref{eq:Mp_4d}), 10 target patterns are generated for each reference pattern. The transformation parameters of target patterns are selected randomly within the ranges $\overline{\theta} \in  [-0.4, 0.4]$;  $t_x, t_y \in [-0.4, 0.4]$; and $\overline{s} \in [0.4, 1.6]$. The above ranges for the normalized rotation and scale parameters $\overline{\theta}$ and $\overline{s}$ correspond to the actual rotation angles $\theta \in [-0.04 \pi, 0.04 \pi]$ and scale change factors $s \in [0.87, 1.13]$. Each target pattern is corrupted with a different realization of a noise pattern that consists of 100 small-scale Gaussian atoms with random coefficients drawn from a normal distribution, which represents a random noise pattern in the continuous domain. The noise patterns are normalized to match a range of noise levels $\noiselev$. A realization of the random reference pattern with and without noise is shown in Figure \ref{fig:rand_pat_realiz}.

The results obtained for the transformation models (\ref{eq:Mp_2d}), (\ref{eq:Mp_3d}), and (\ref{eq:Mp_4d}) are presented respectively in Figures \ref{fig:rand_2d}, \ref{fig:rand_3d} and \ref{fig:rand_4d}, where the performance is averaged over all reference and target patterns. In all figures, the experimental alignment errors and their theoretical upper bounds are plotted with respect to the noise level $\noiselev$ in panels (a) and (b), where the noise level $\noiselev$ is normalized with the norm $\| p \|$ of the reference pattern. The same experimental errors and theoretical bounds are plotted as functions of the filter size $\rho$ in panels (c) and (d) of all figures.
 
The results of this experiment can be interpreted as follows. First, the plots in panels (a) and (b) of Figures  \ref{fig:rand_2d}-\ref{fig:rand_4d}  show that the variation of the alignment error with the noise level $\noiselev$ approaches an approximately linear rate for large values of $\noiselev$ both in the empirical and the theoretical plots. This confirms the estimations $\hatalerrbnd = O(\noiselev)$,  $\hatalerrbnd = O(\noiselev+1)$ of Theorem \ref{thm:dep_alerrbnd}. Next, the plots in (c) and (d) of the figures show that the actual alignment error and its theoretical upper bound decrease with filtering at small filter sizes $\rho$, as smoothing decreases the nonlinearity of the manifold. The error then begins to increase with the filter size $\rho$ at larger values of $\rho$ in the presence of noise. This confirms that the filter size has an optimal value when the target image is noisy, as predicted by Theorem \ref{thm:dep_alerrbnd}. The shift in the optimal value of the filter size with the increase in the noise level is observable especially in Figures  \ref{fig:rand_2d} and \ref{fig:rand_3d}, which is in agreement with the approximate relation between the optimal filter size and the noise level given in (\ref{eq:opt_choice_rhok}). Moreover, in most plots, the optimal value of the filter size that minimizes the theoretical upper bound in (d) is seen to be in the vicinity of the optimal filter size minimizing the actual alignment error in (c), which shows that the theoretical bound provides a good prediction of suitable filter sizes in alignment. The results also show that the variation of the alignment error with the filter size approximately matches the rate $\hatalerrbnd = O\left((1+\rho^2)^{1/2}\right) \approx O(\rho)$ at large filter sizes in most plots.

\begin{figure}[ht!]
\begin{center}
     \subfigure[]
       {\label{fig:exp_4d_nu}\includegraphics[height=6cm]{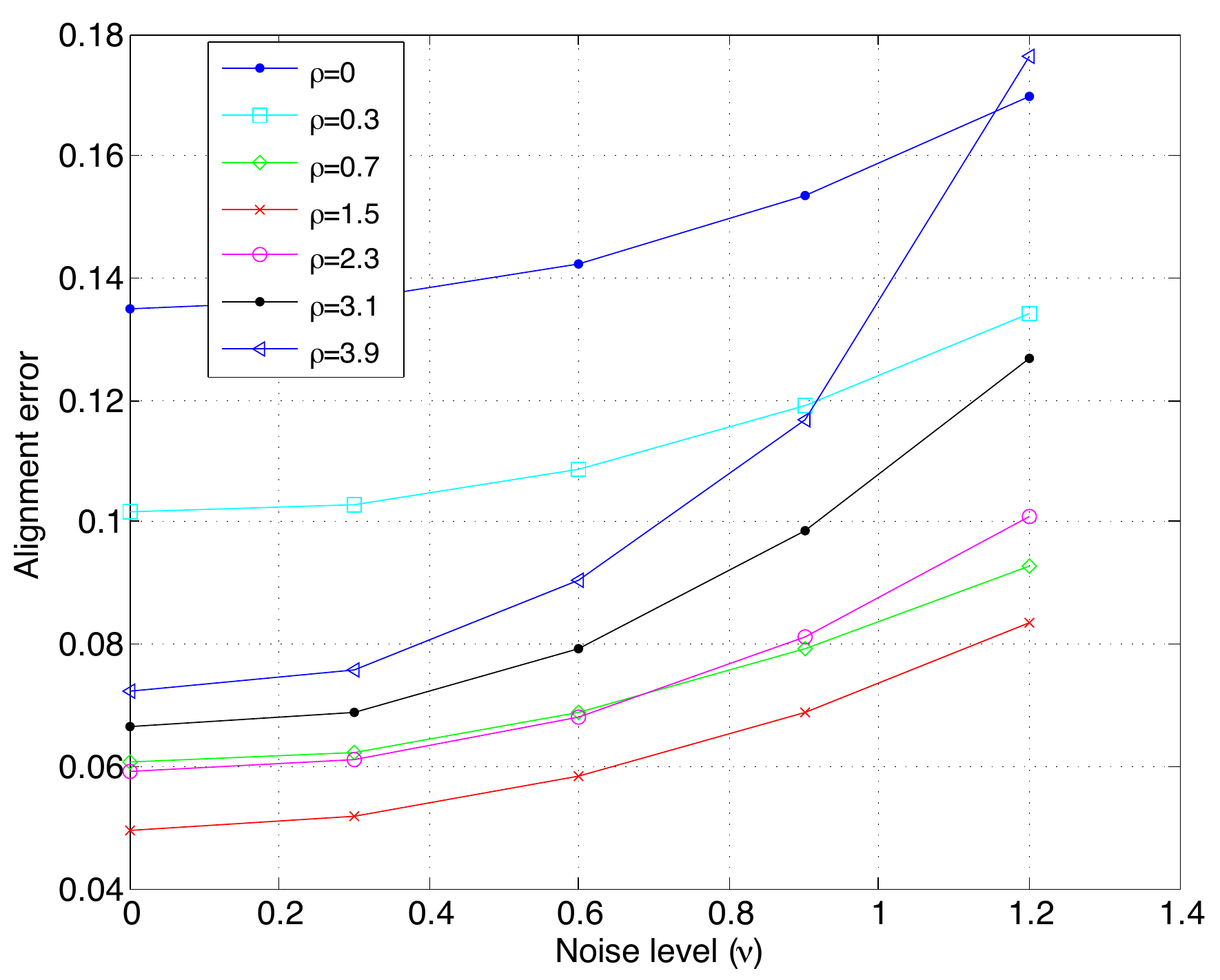}}
     \subfigure[]
       {\label{fig:theo_4d_nu}\includegraphics[height=6cm]{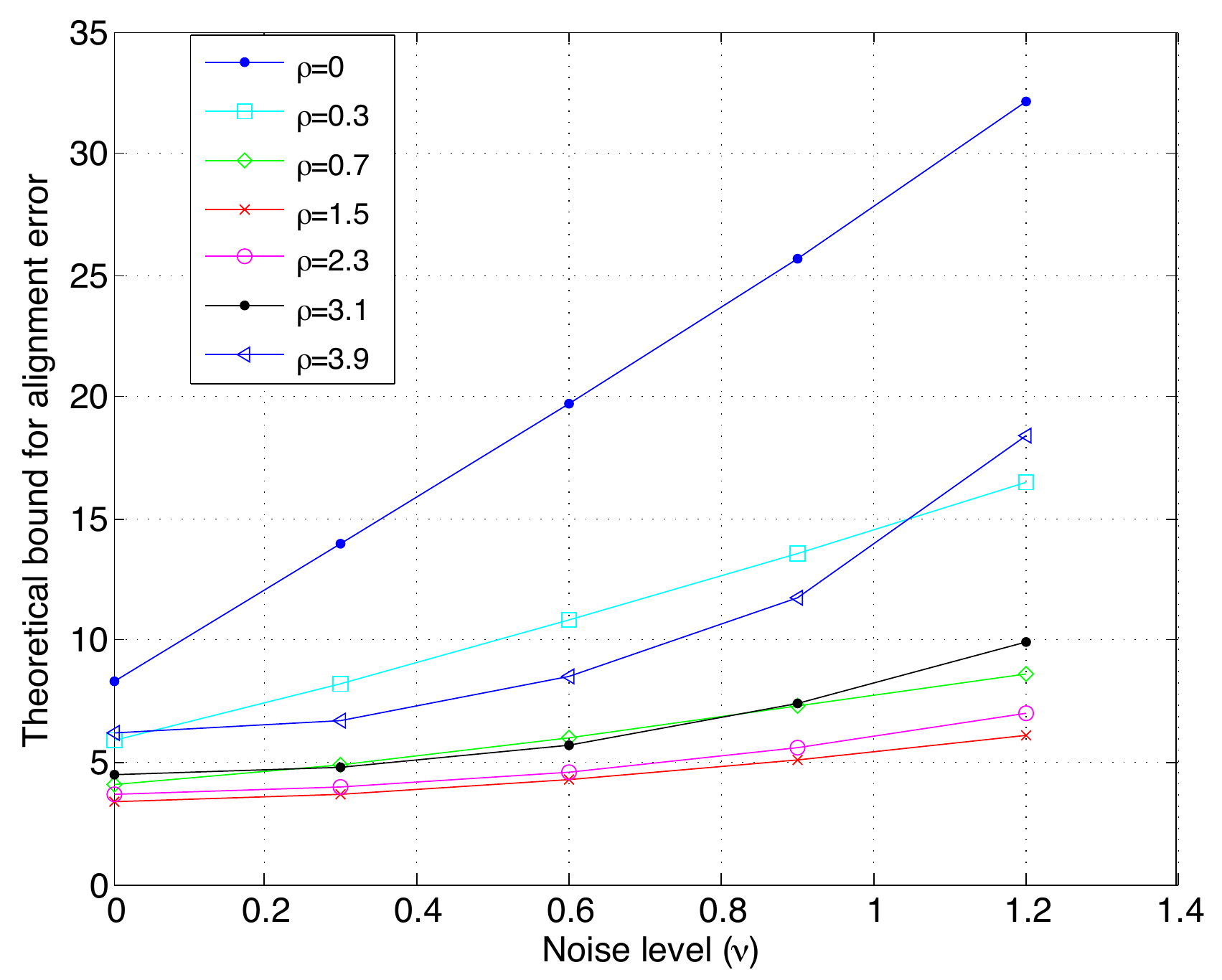}}
     \subfigure[]
       {\label{fig:exp_4d_rho}\includegraphics[height=6cm]{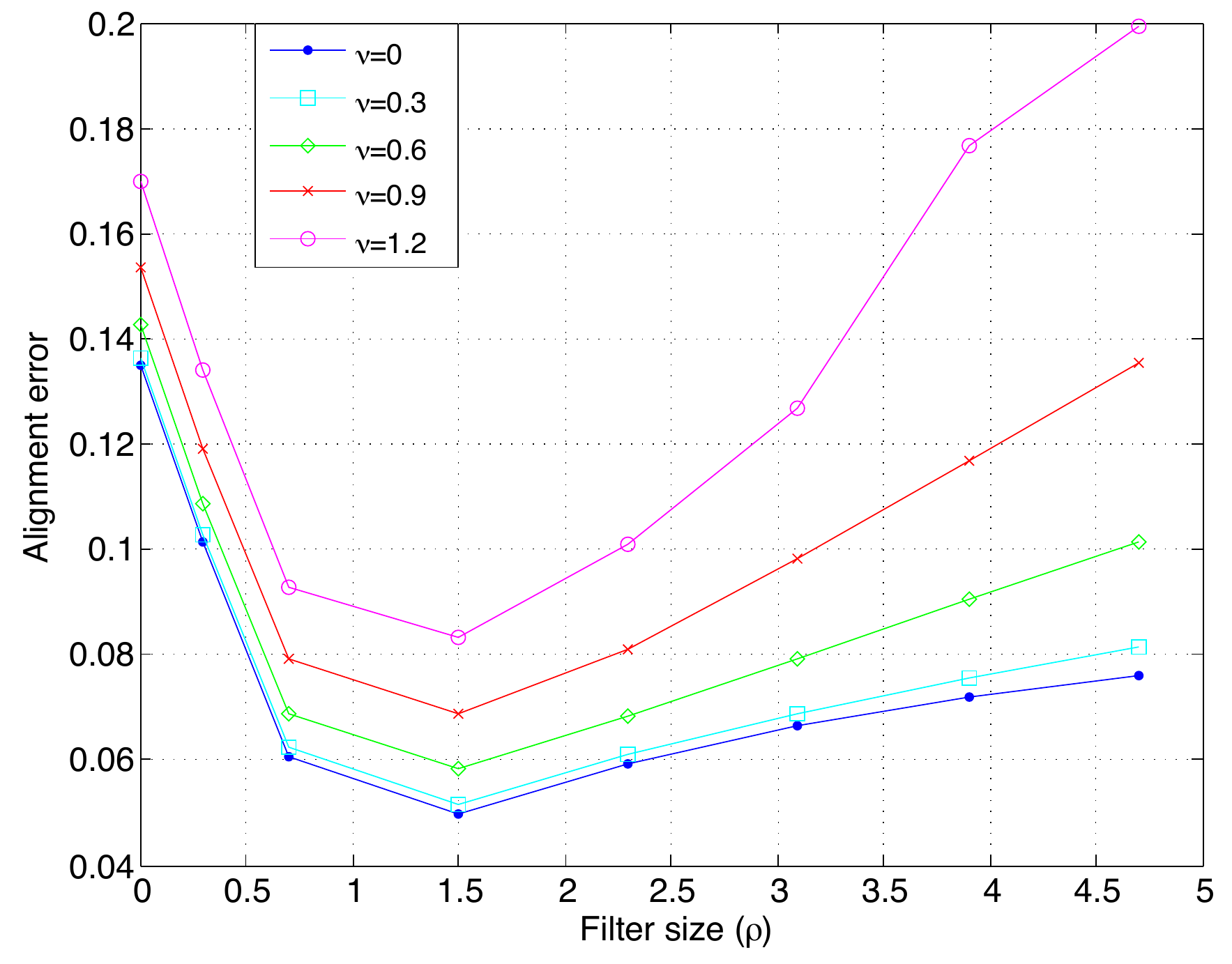}}
     \subfigure[]
       {\label{fig:theo_4d_rho}\includegraphics[height=6cm]{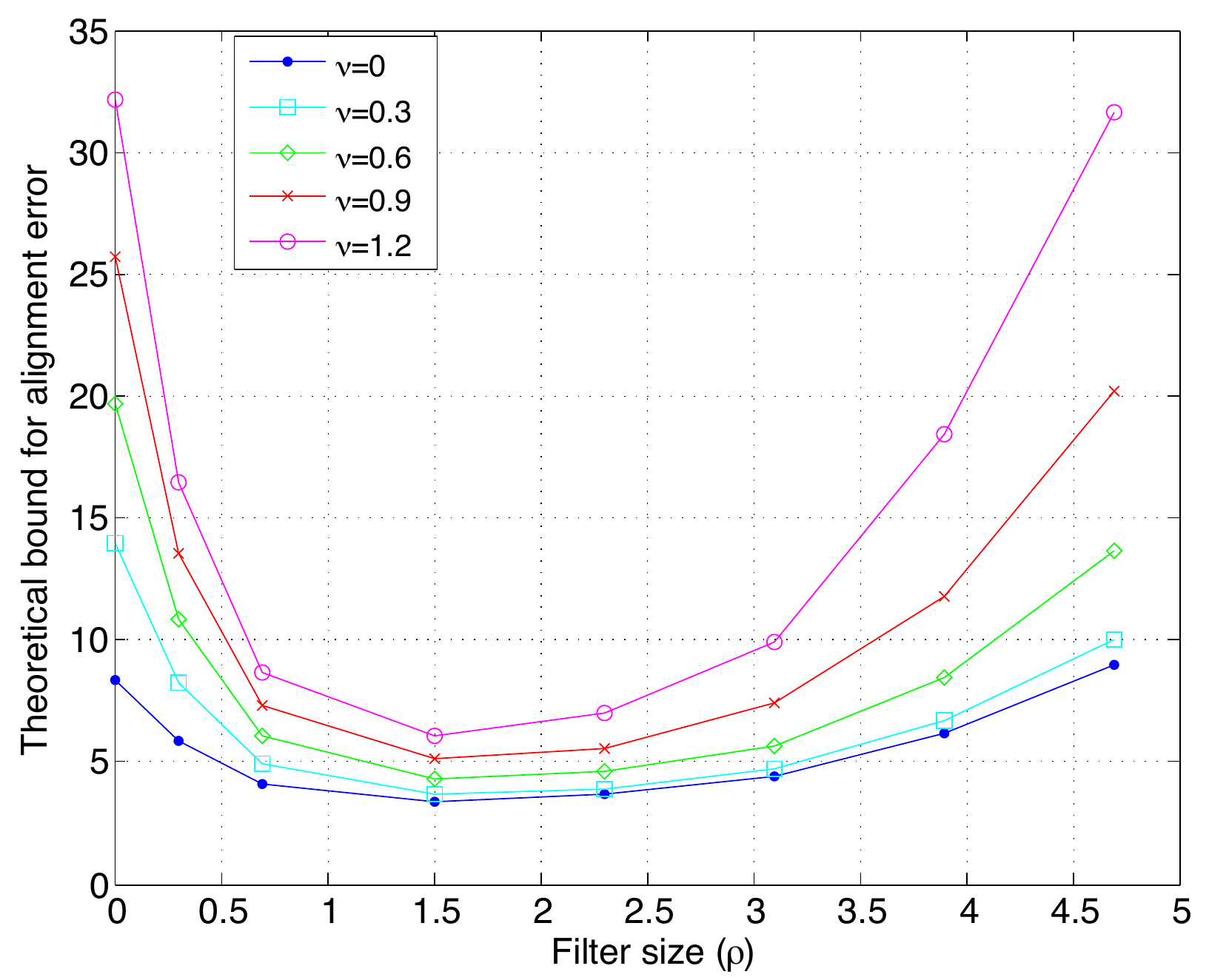}}
 \end{center}
 \caption{Alignment errors of random patterns for 4-D manifolds generated by translations, rotations, and scale changes.}
 \label{fig:rand_4d}
\end{figure}

Next, we comment on the plots in Figure \ref{fig:rand_4d} obtained for four-dimensional transformation manifolds generated by translations, rotations, and isotropic scale changes. One can observe in Figures \ref{fig:exp_4d_rho} and  \ref{fig:theo_4d_rho} that both the experimental alignment error and its theoretical upper bound increase significantly with the filter size $\rho$ in the noiseless case $\noiselev=0$ when transformations include scale changes. This is due to the secondary source of noise arising from the fact that geometric transformations with scale changes do not commute with filtering, which is discussed in Section \ref{ch:tan_dist:ssec:dep_align_bnd}. Theorem \ref{thm:dep_alerrbnd} suggests that the error increases with filtering at a rate $\hatalerrbnd = O\left((\noiselev+1)(1+\rho^2)^{1/2}\right)$ at large values of $\rho$, which corresponds to a variation  $\hatalerrbnd = O\left((1+\rho^2)^{1/2}\right)$ in the noiseless case. %Finally, the convergence of the error to increase at a linear rate with the noise level $\noiselev$ in Figures \ref{fig:exp_4d_nu} and \ref{fig:theo_4d_nu} confirm that this secondary source of noise can indeed be added to the primary noise term $\noiselev$ in the form of a constant, so that the overall error is of $O(\noiselev+1)$ at a fixed value of the filter size.

\subsection{Alignment of natural images}
\label{ssec:exp_align_real}

We perform a second set of experiments on five real images, which are shown in Figure \ref{fig:real_images_td}. The images are resized to the resolution of $60\times 60$ pixels, and for each image an analytical approximation in the Gaussian dictionary $\mathcal{D}$ is computed with 100 atoms. The dictionary is defined over the parameter domain $\psi \in [-\pi, \pi)$; $\tau_x, \tau_y \in [-6, 6]$; $\sigma_x, \sigma_y \in [0.05, 3.5]$. Two reference patterns are considered for each image; namely, the digital image itself, and its analytical approximation in $\mathcal{D}$. For each one of the transformation models  (\ref{eq:Mp_2d})-(\ref{eq:Mp_4d}), 40 test patterns are generated for each reference pattern by applying a geometric transformation and adding a digital Gaussian noise image that is i.i.d.~for each pixel. The geometric transformations are randomly selected from the transformation parameter domain $\overline{\theta} \in [-0.6, 0.6]$; $t_x, t_y \in [-0.6, 0.6]$; $\overline{s} \in [0.1, 2.1]$. The normalized rotation and scale parameters $\overline{\theta}$ and $\overline{s}$ correspond to the actual rotation angle and scale change factors $\theta \in [-0.07\pi, 0.07\pi ]$ and $s \in [0.89, 1.13]$. As the length of the interval $[-0.6, 0.6]$ of translation parameters $t_x, t_y$ is one-tenth of that of the domain $[-6, 6]$ where atom centers $\tau_x, \tau_y$ lie, the maximal amount of translation in this experiment is around one-tenth of the image size. The experimental alignment errors $\| \hattparest -  \hattparopt \|$ are computed by aligning the target patterns with the reference patterns, for both the original digital images and their approximations in the analytical dictionary $\D$. The theoretical upper bounds $\hatalerrbnd$ are computed based on the analytical representations of the reference patterns. The alignment errors are plotted in Figures \ref{fig:real_2d}-\ref{fig:real_4d}, which are averaged over all reference and target patterns. Figures \ref{fig:real_2d}, \ref{fig:real_3d}, and \ref{fig:real_4d} show the errors obtained with the 2-D, 3-D and 4-D manifold models given respectively in (\ref{eq:Mp_2d}), (\ref{eq:Mp_3d}), and (\ref{eq:Mp_4d}). In all figures, the alignment errors of the digital images, the alignment errors of the analytical approximations of images, and the theoretical upper bounds for the alignment error are plotted with respect to the noise level $\noiselev$ in panels (a)-(c), and with respect to the filter size $\rho$ in panels (d)-(f).

\begin{figure}[t]
 \centering
  \includegraphics[width=11cm]{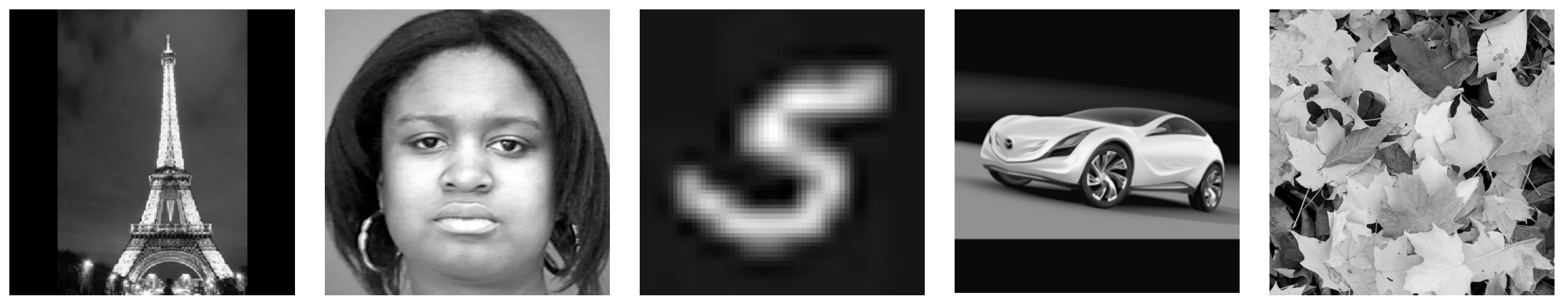}
  \caption{Images used in the second set of experiments}
  \label{fig:real_images_td}
\end{figure}

\begin{figure}[h!]
\begin{center}
     \subfigure[]
       {\label{fig:rd_expnum_2d_nu}\includegraphics[height=5cm]{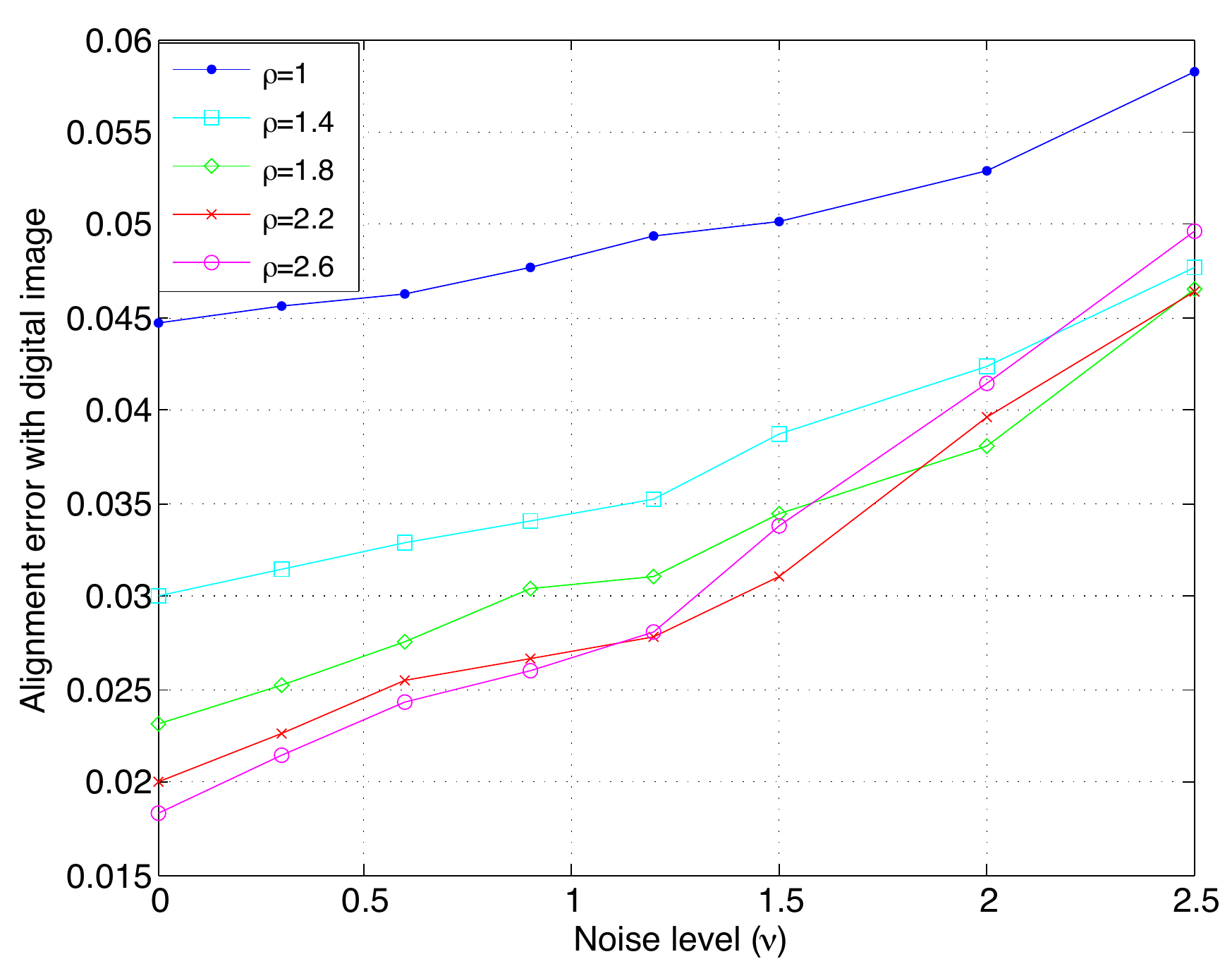}}
     \subfigure[]
       {\label{fig:rd_exp_2d_nu}\includegraphics[height=5cm]{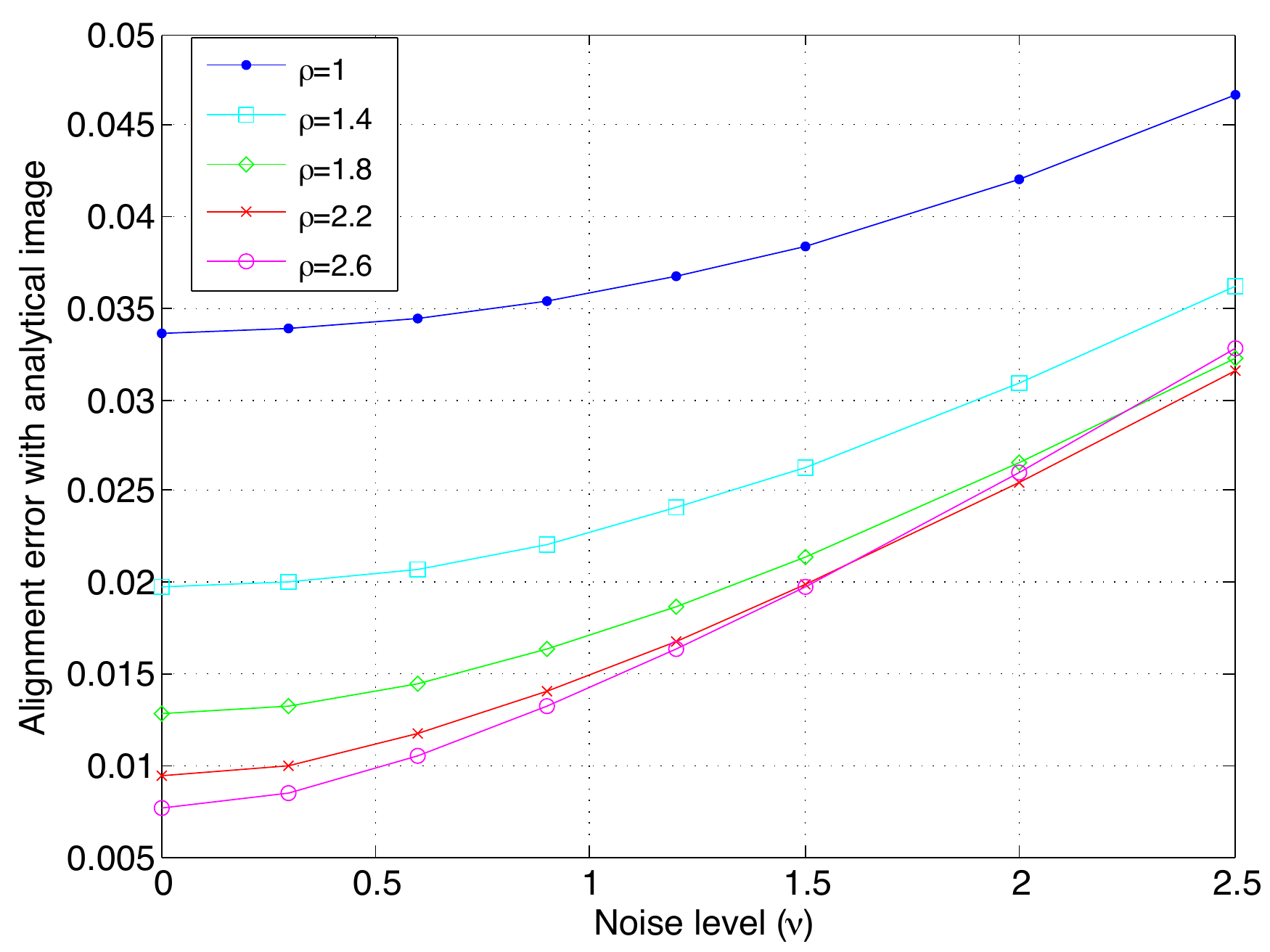}}
     \subfigure[]
       {\label{fig:rd_theo_2d_nu}\includegraphics[height=5cm]{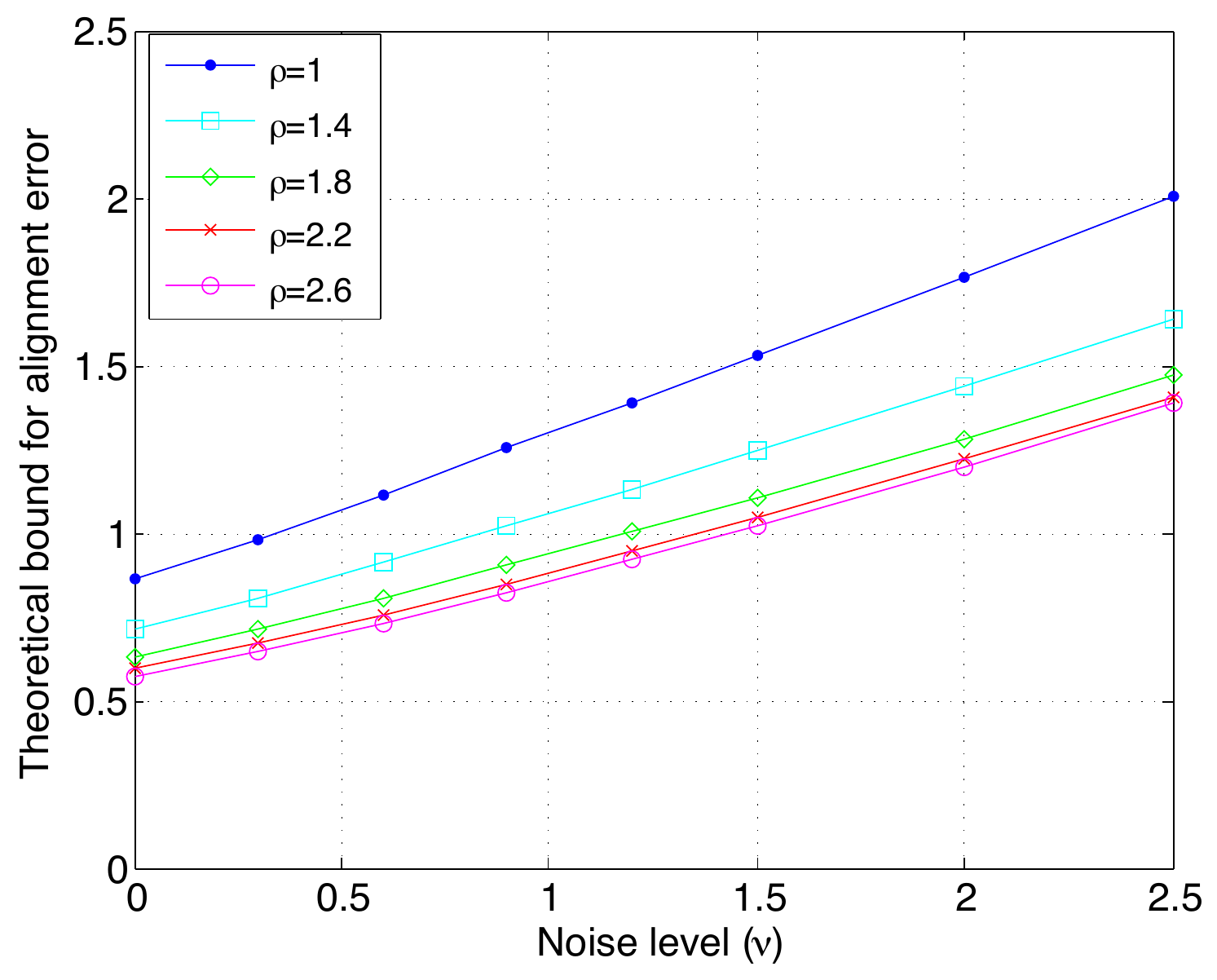}}
     \subfigure[]
       {\label{fig:rd_expnum_2d_rho}\includegraphics[height=5cm]{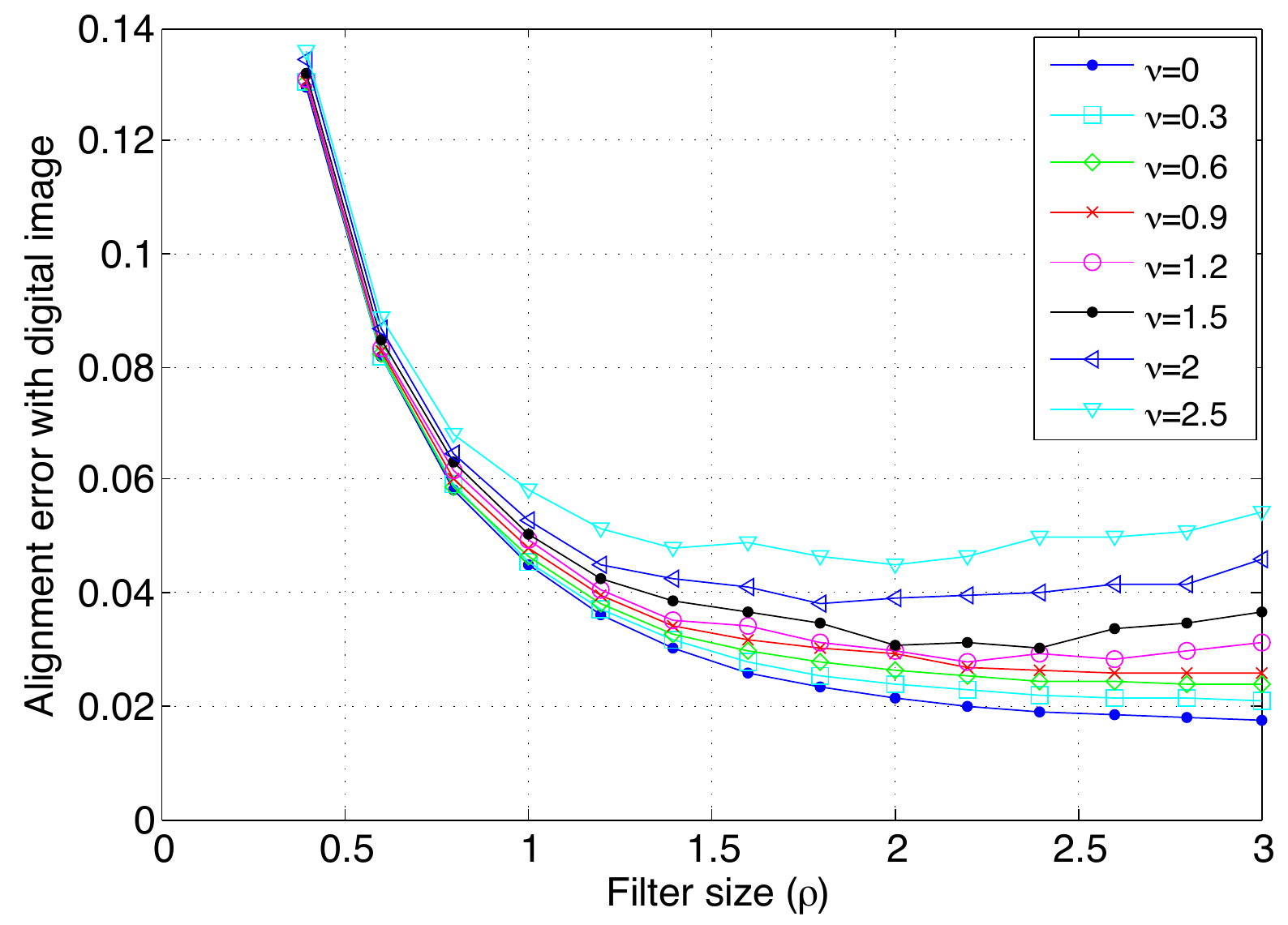}}
     \subfigure[]
       {\label{fig:rd_exp_2d_rho}\includegraphics[height=5cm]{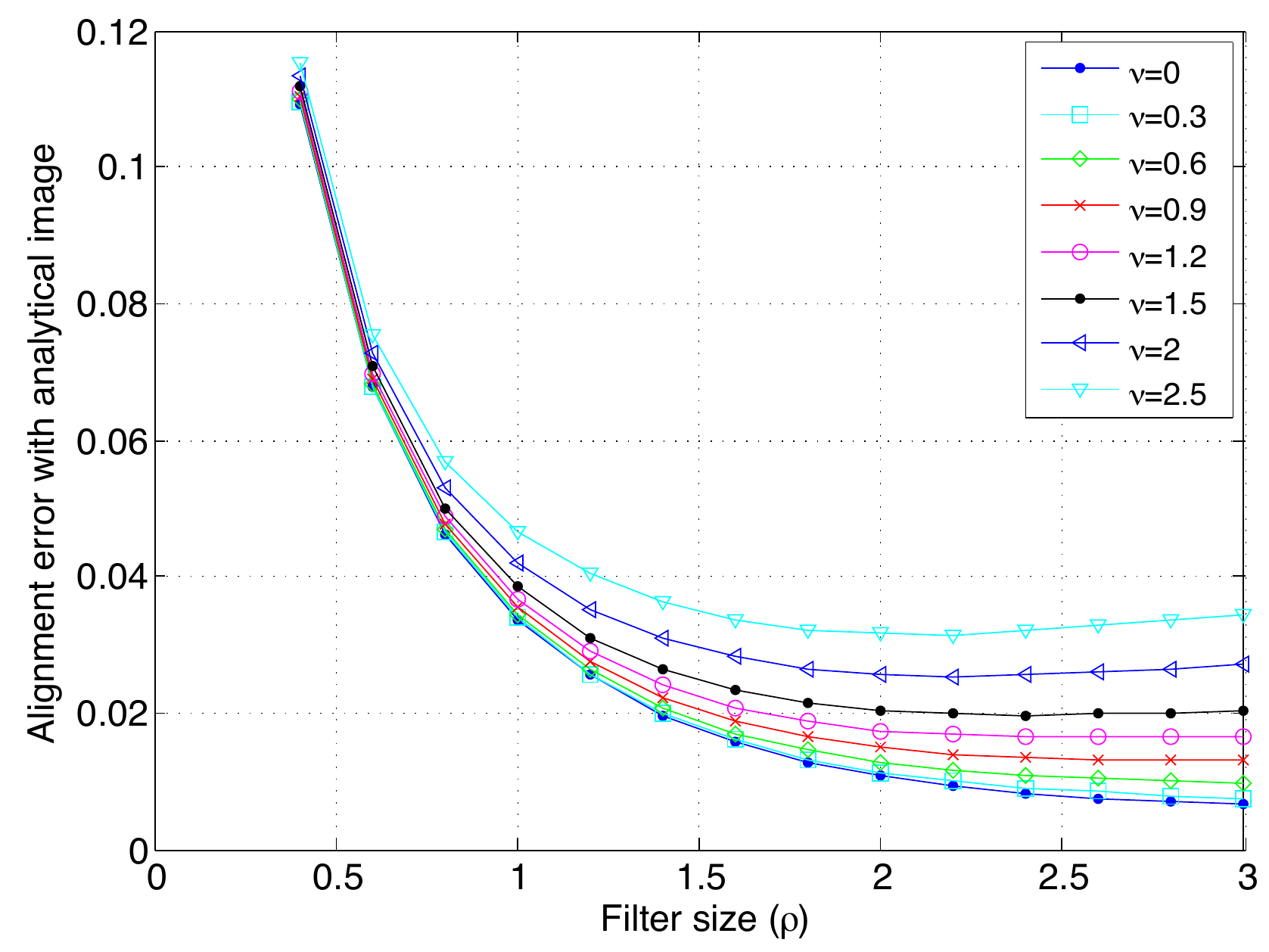}}
     \subfigure[]
       {\label{fig:rd_theo_2d_rho}\includegraphics[height=5cm]{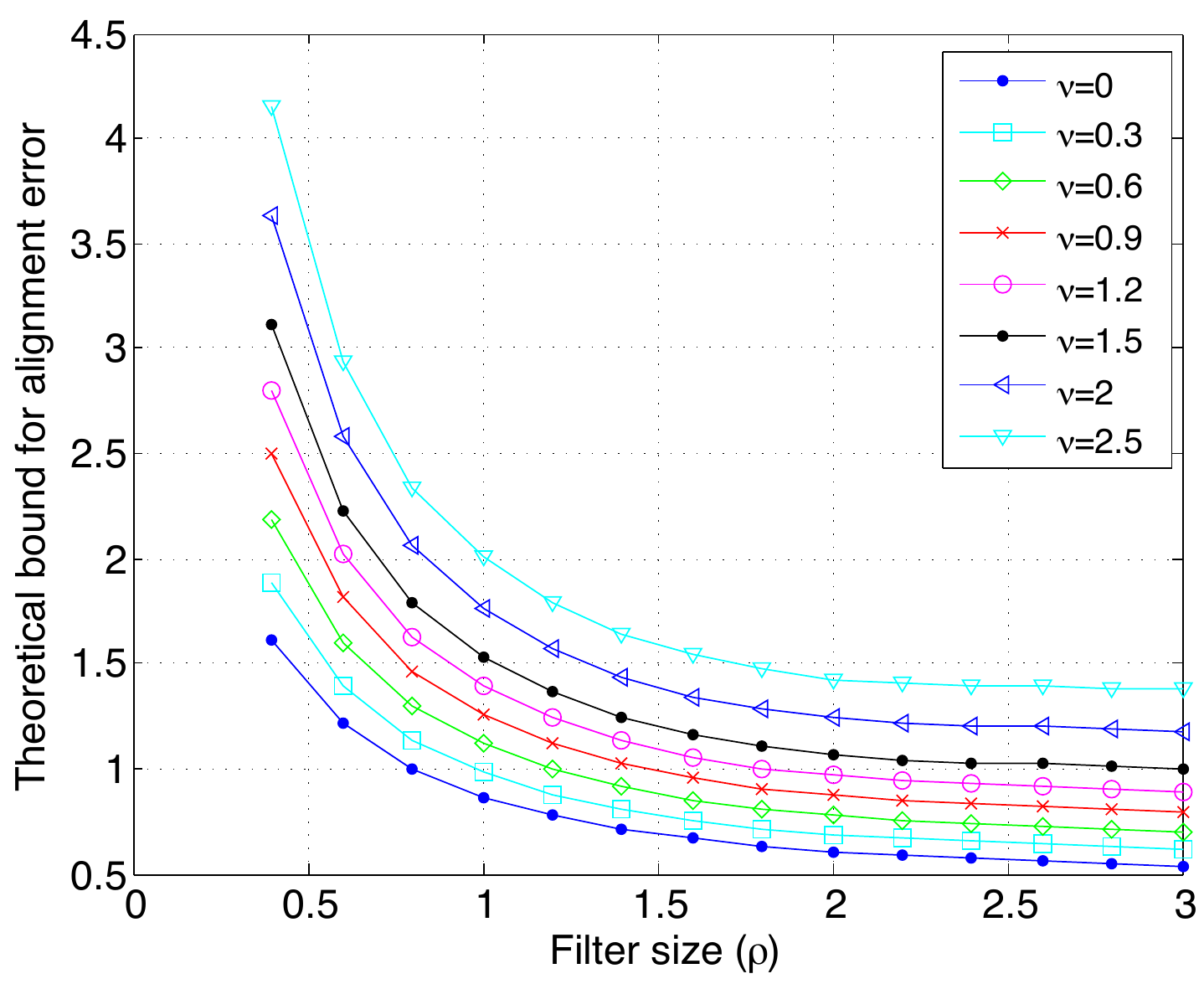}}
 \end{center}
 \caption{Alignment errors of real images for 2-D manifolds generated by translations.}
 \label{fig:real_2d}
\end{figure}

\begin{figure}[h!]
\begin{center}
     \subfigure[]
       {\label{fig:rd_expnum_3d_nu}\includegraphics[height=5cm]{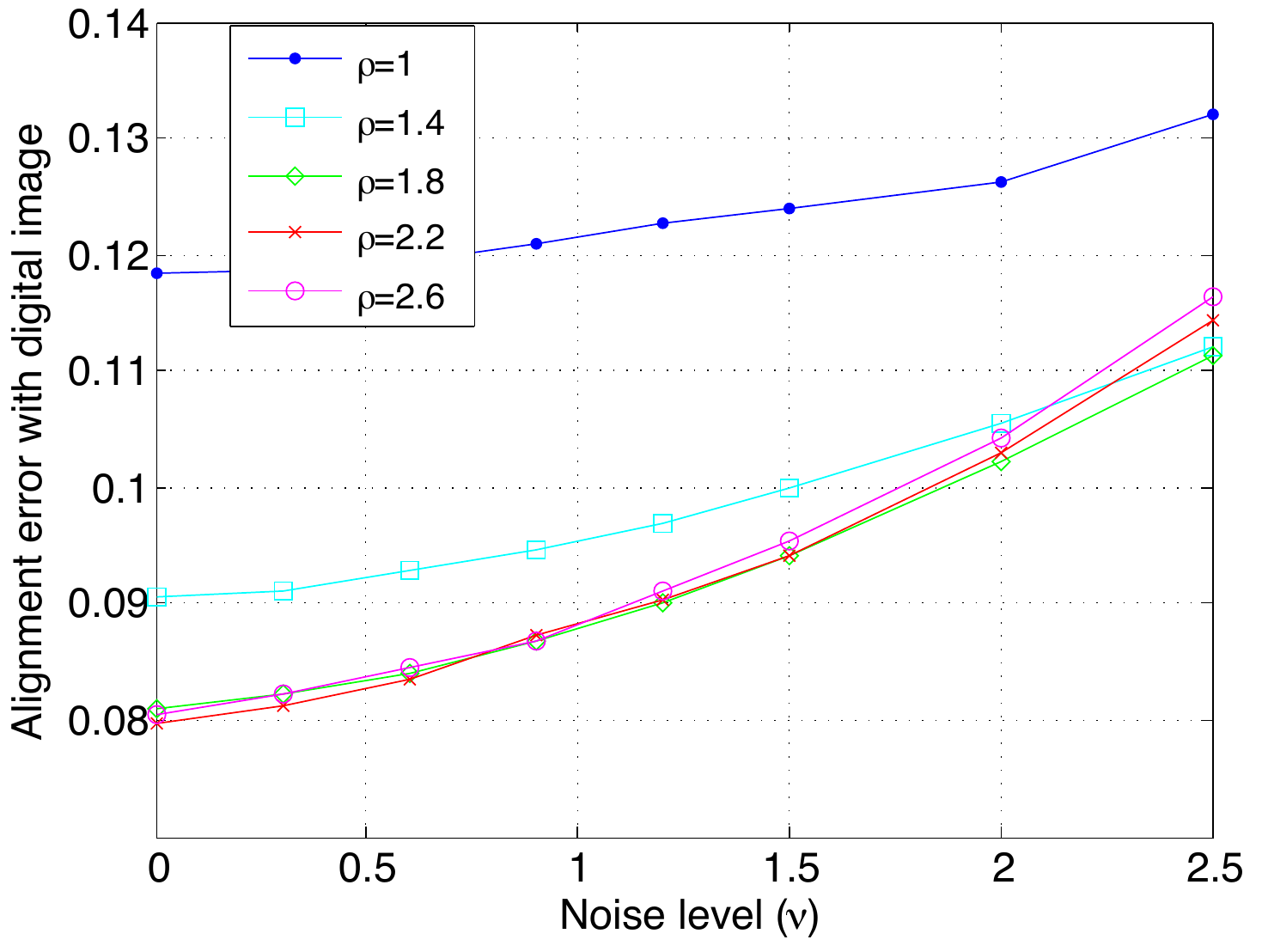}}
     \subfigure[]
       {\label{fig:rd_exp_3d_nu}\includegraphics[height=5cm]{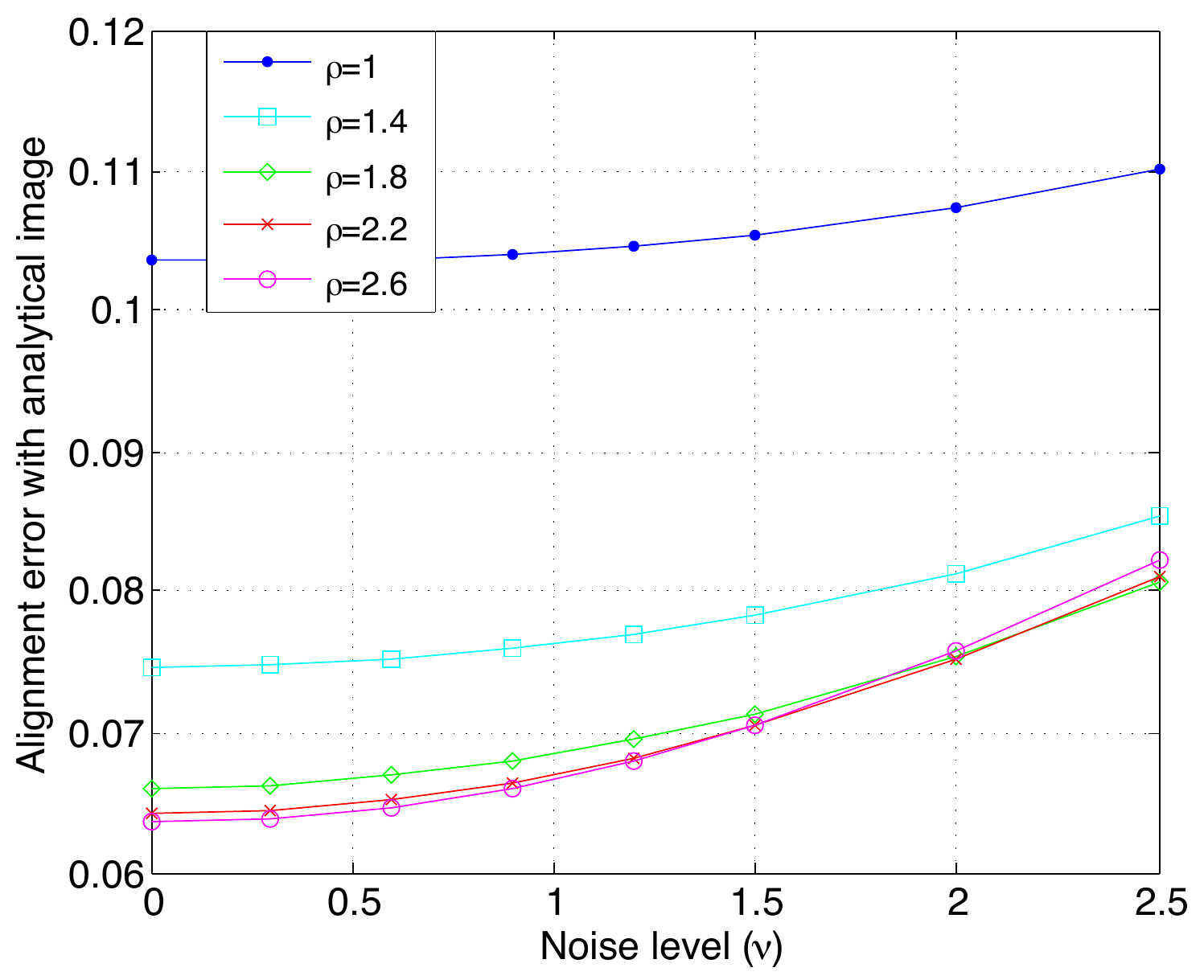}}
     \subfigure[]
       {\label{fig:rd_theo_3d_nu}\includegraphics[height=5cm]{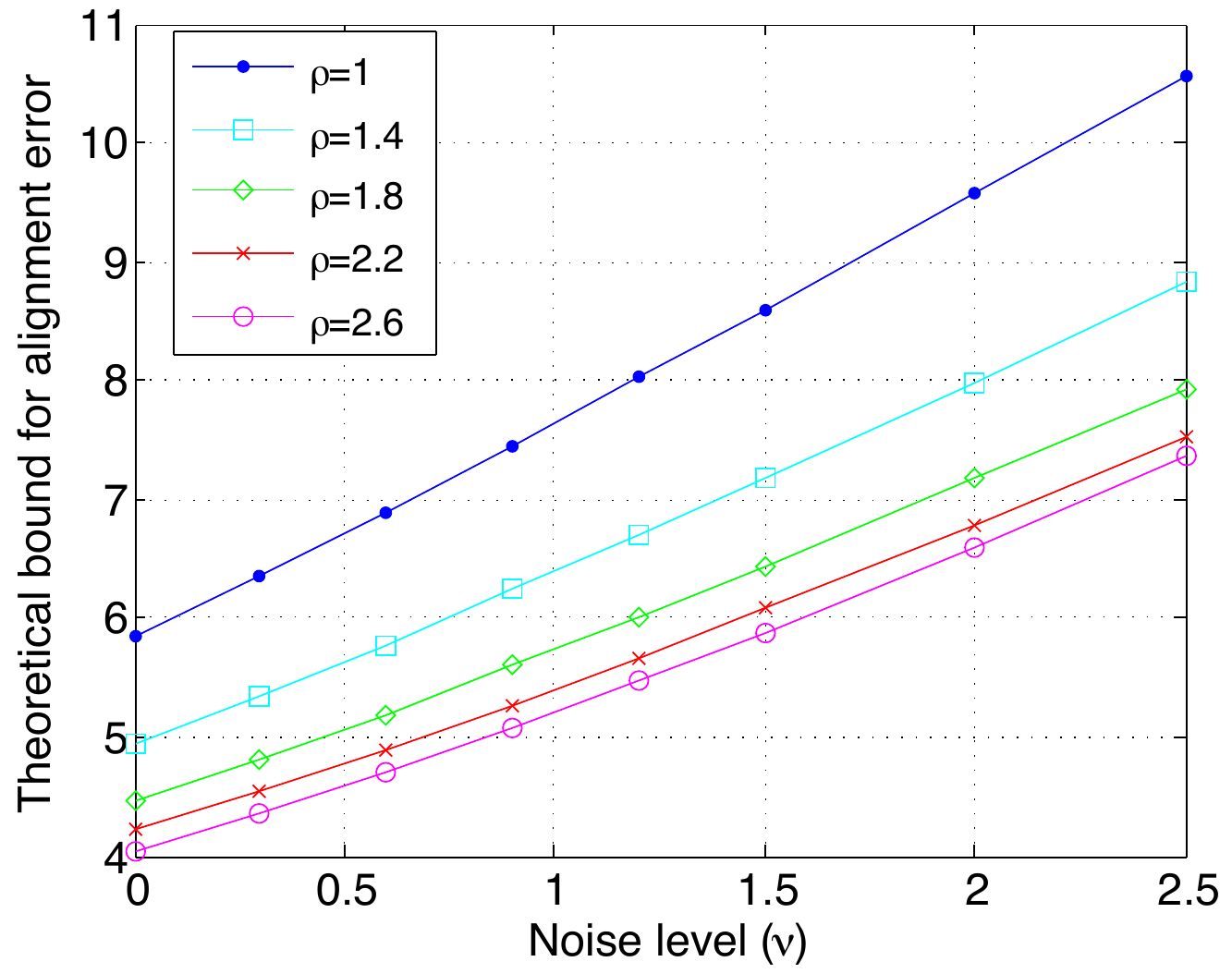}}
     \subfigure[]
       {\label{fig:rd_expnum_3d_rho}\includegraphics[height=5cm]{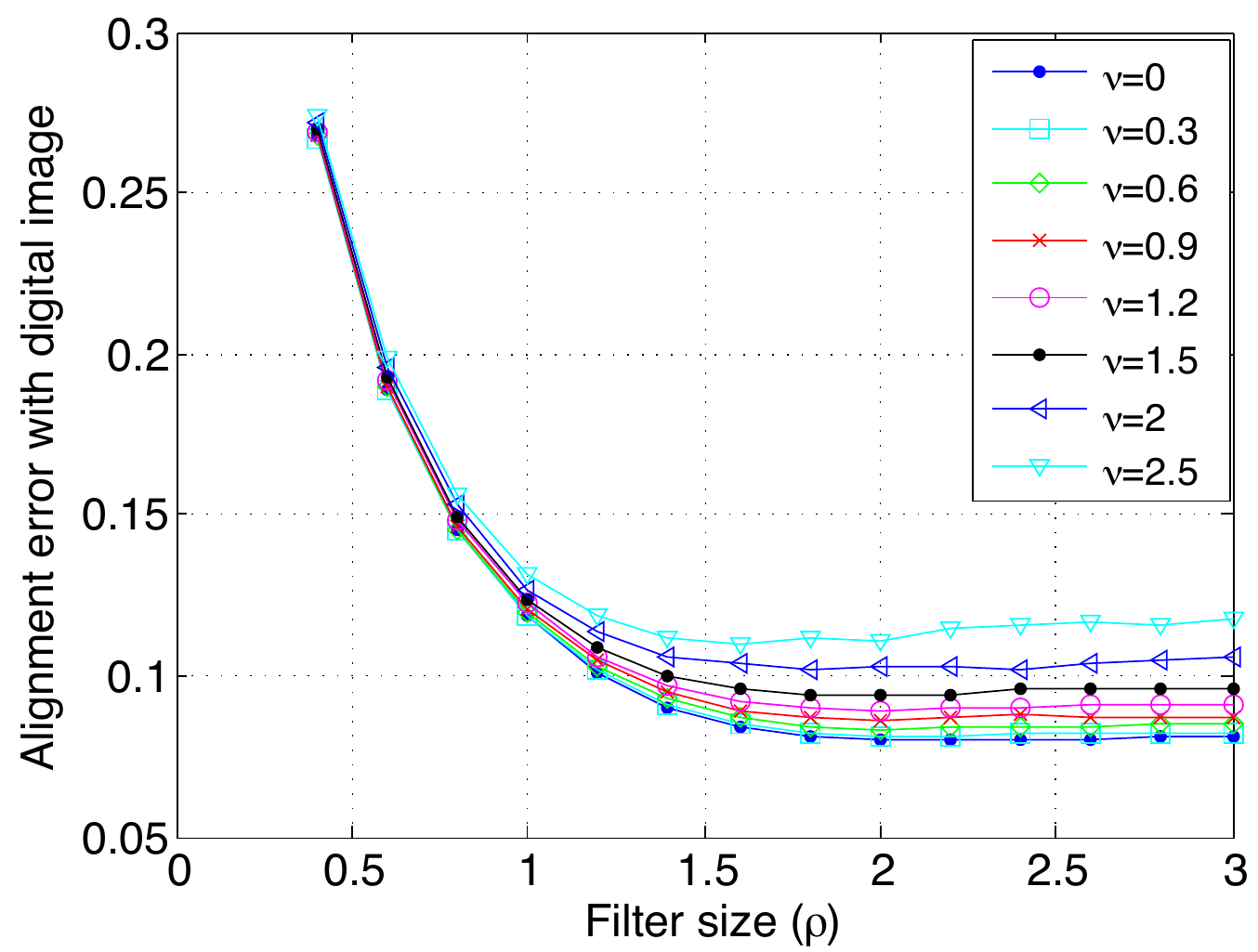}}
     \subfigure[]
       {\label{fig:rd_exp_3d_rho}\includegraphics[height=5cm]{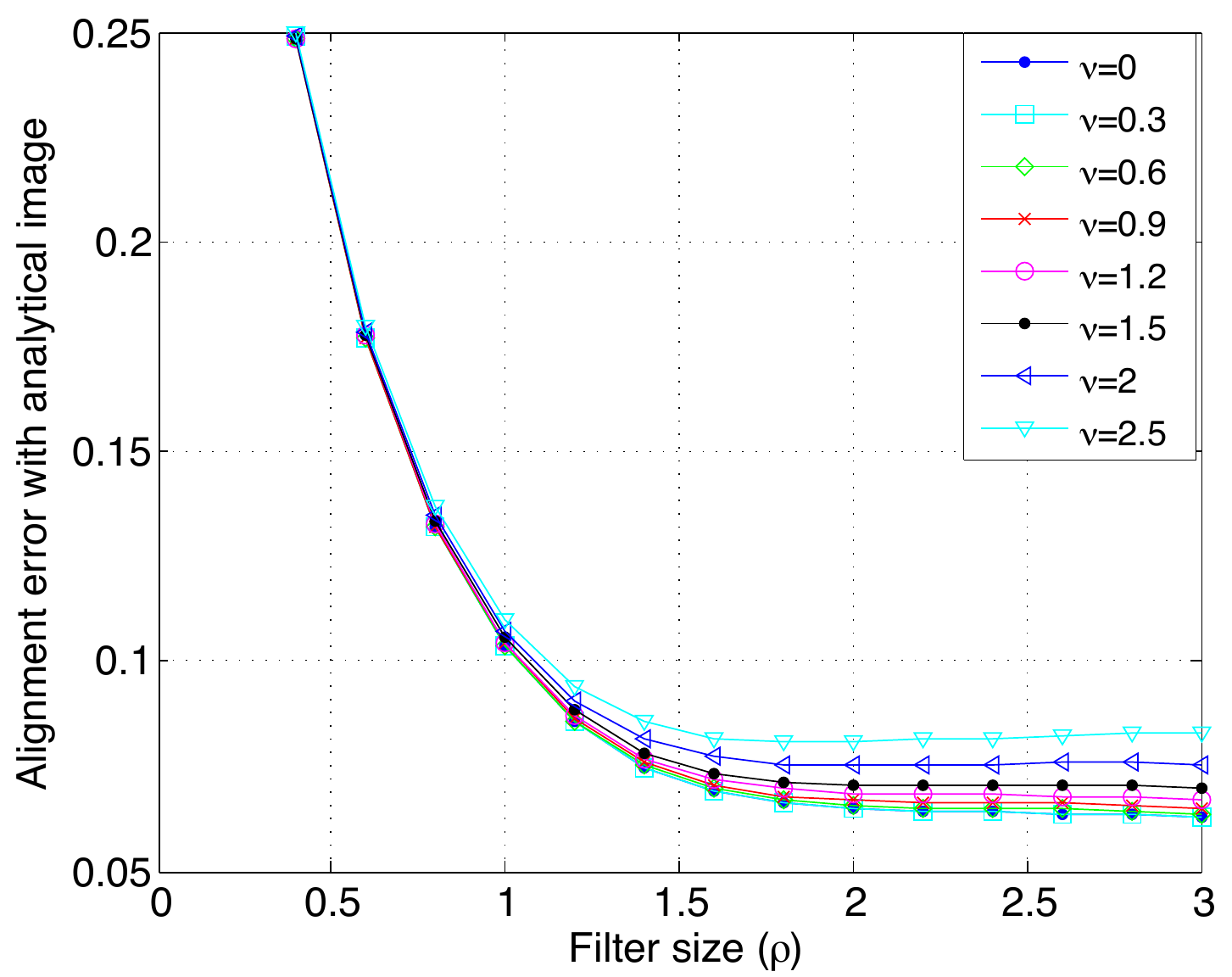}}
     \subfigure[]
       {\label{fig:rd_theo_3d_rho}\includegraphics[height=5cm]{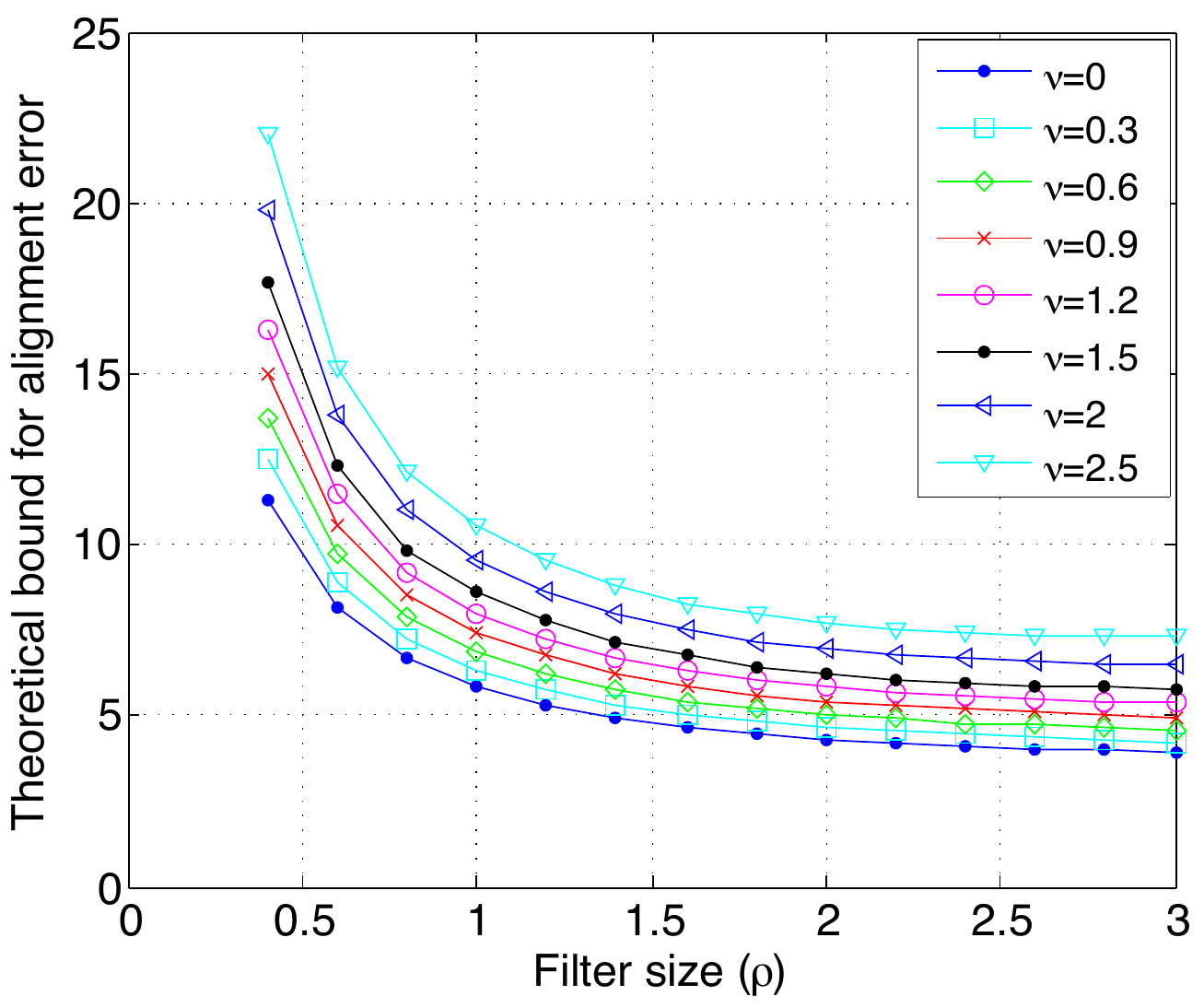}}
 \end{center}
 \caption{Alignment errors of real images for 3-D manifolds generated by translations and rotations.}
 \label{fig:real_3d}
\end{figure}

\begin{figure}[h!]
\begin{center}
     \subfigure[]
       {\label{fig:rd_expnum_4d_nu}\includegraphics[height=5cm]{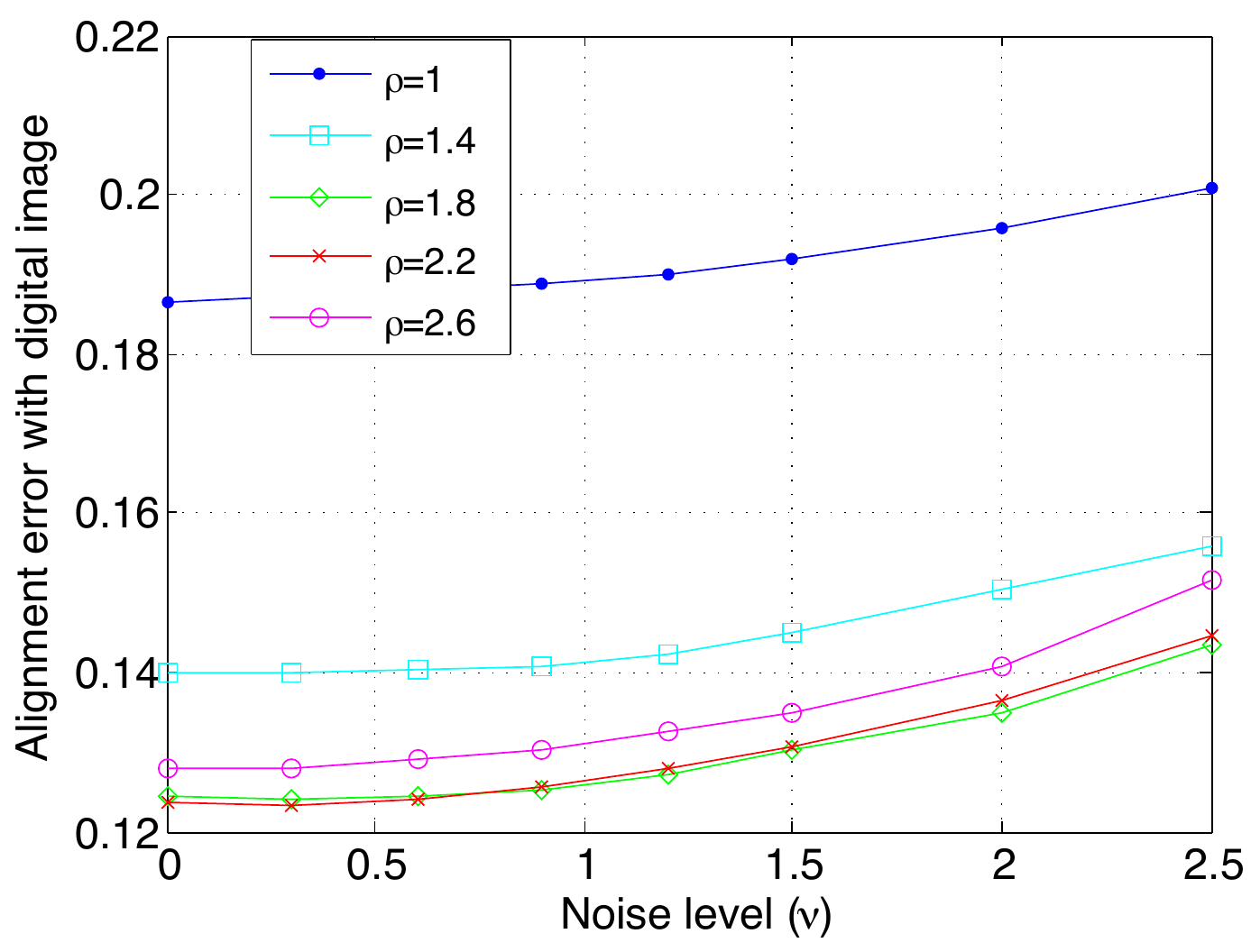}}
     \subfigure[]
       {\label{fig:rd_exp_4d_nu}\includegraphics[height=5cm]{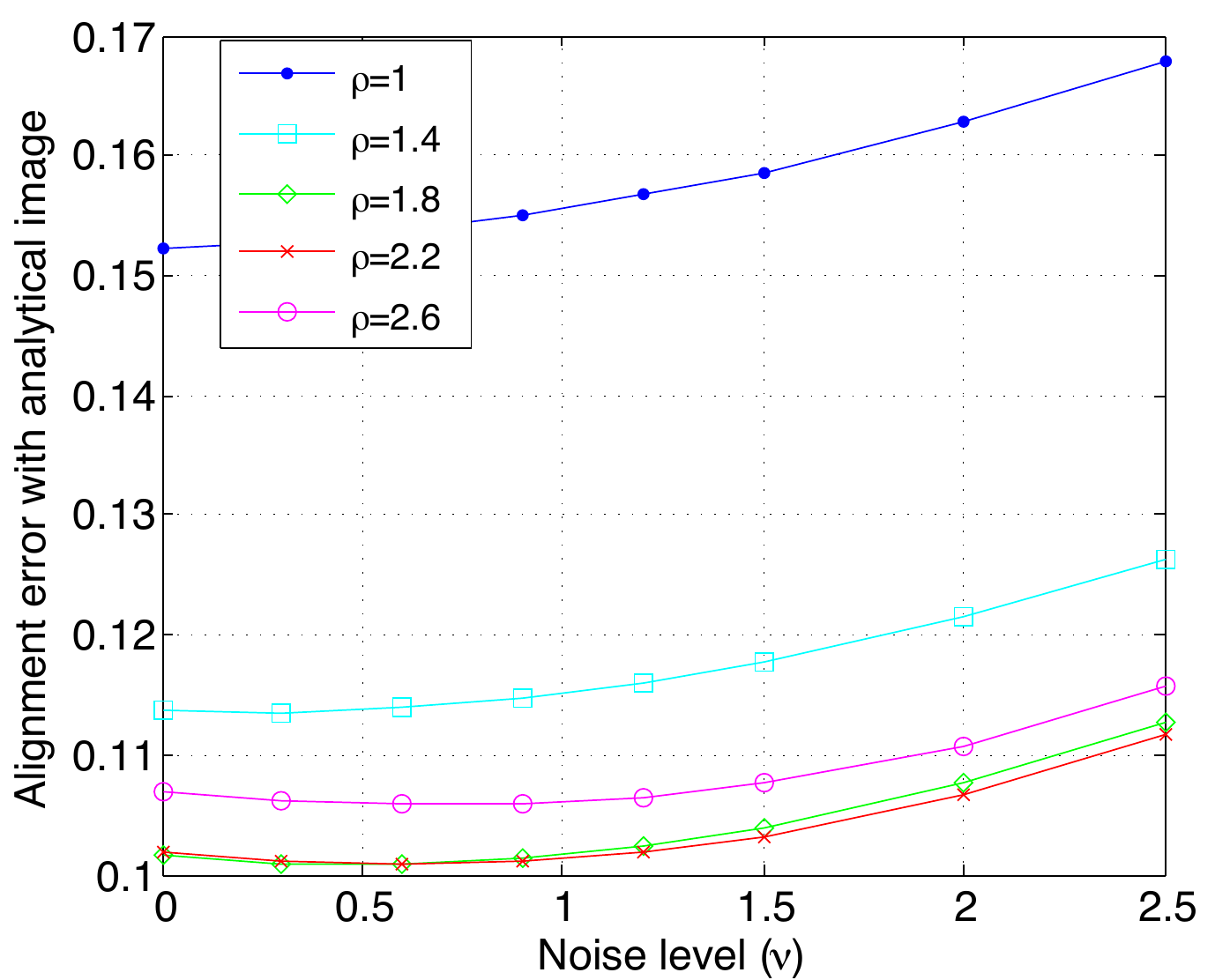}}
     \subfigure[]
       {\label{fig:rd_theo_4d_nu}\includegraphics[height=5cm]{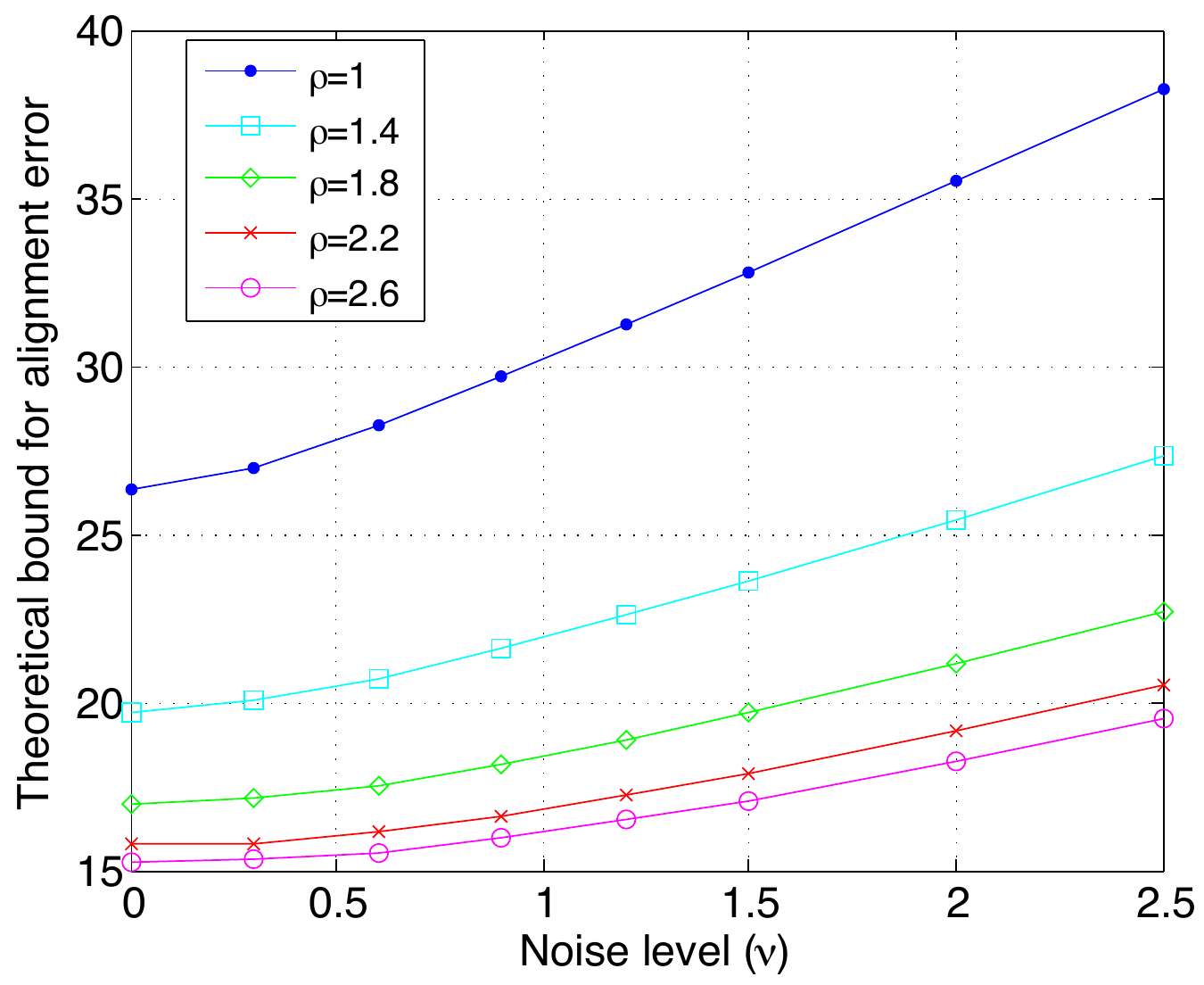}}
     \subfigure[]
       {\label{fig:rd_expnum_4d_rho}\includegraphics[height=5cm]{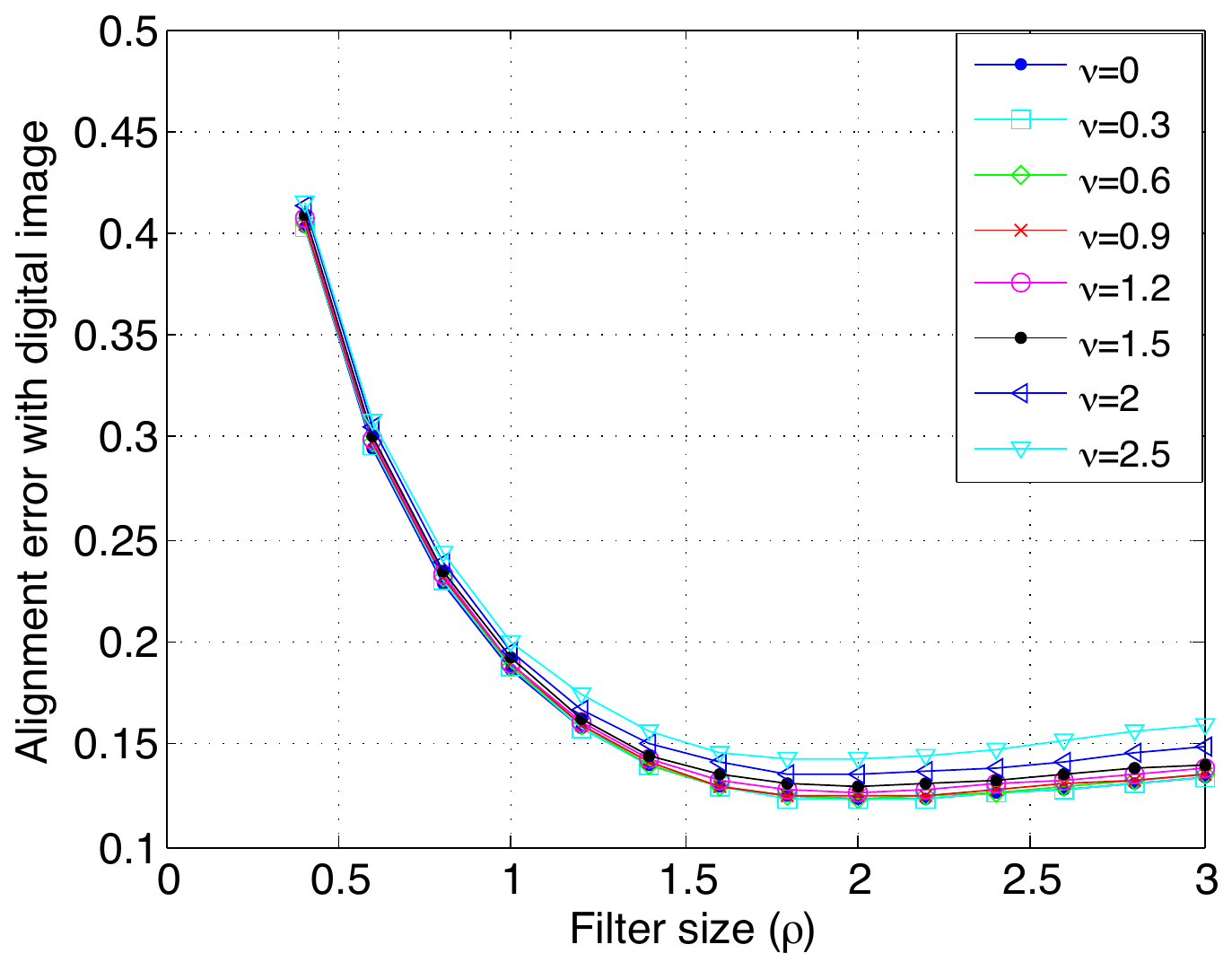}}
     \subfigure[]
       {\label{fig:rd_exp_4d_rho}\includegraphics[height=5cm]{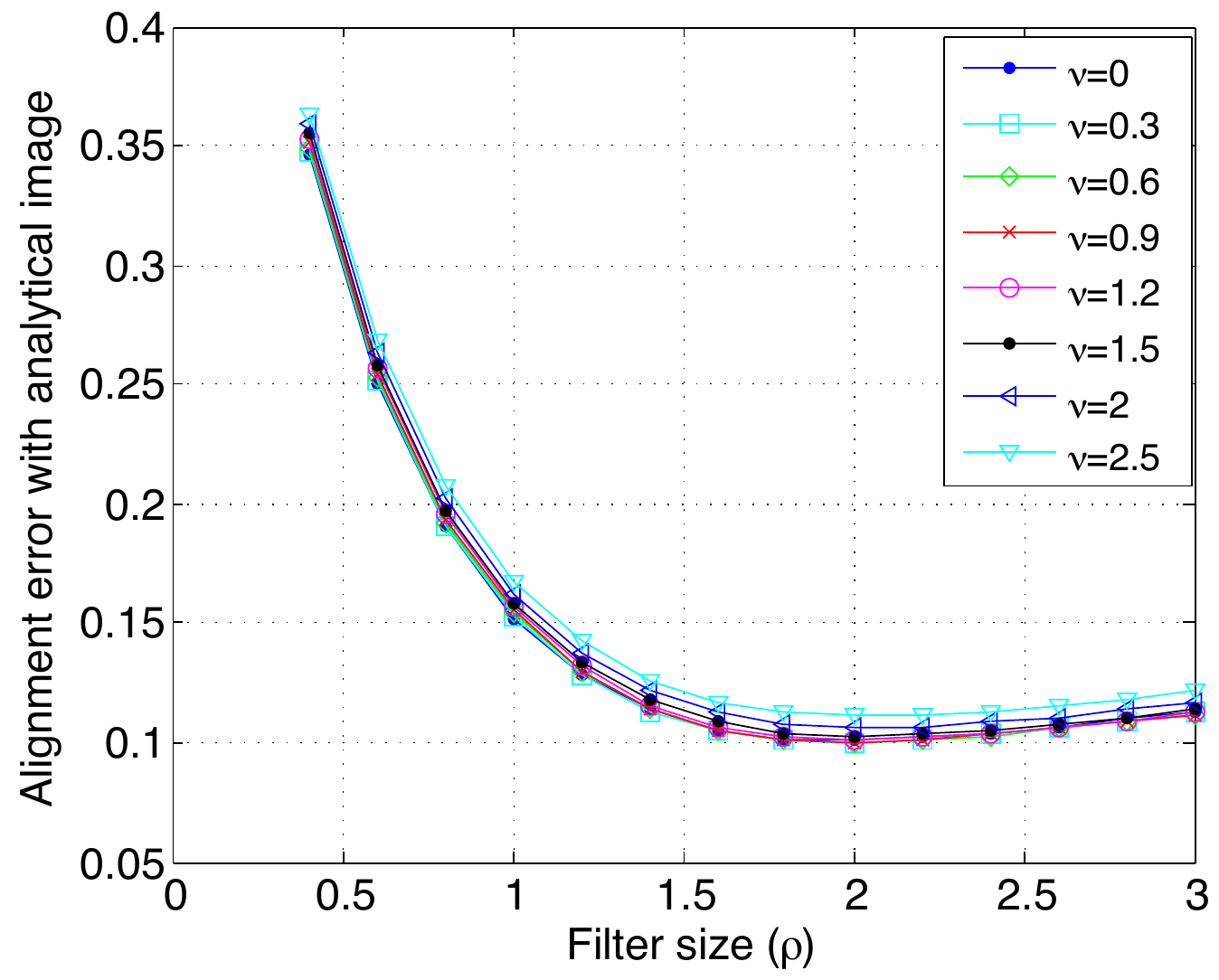}}
     \subfigure[]
       {\label{fig:rd_theo_4d_rho}\includegraphics[height=5cm]{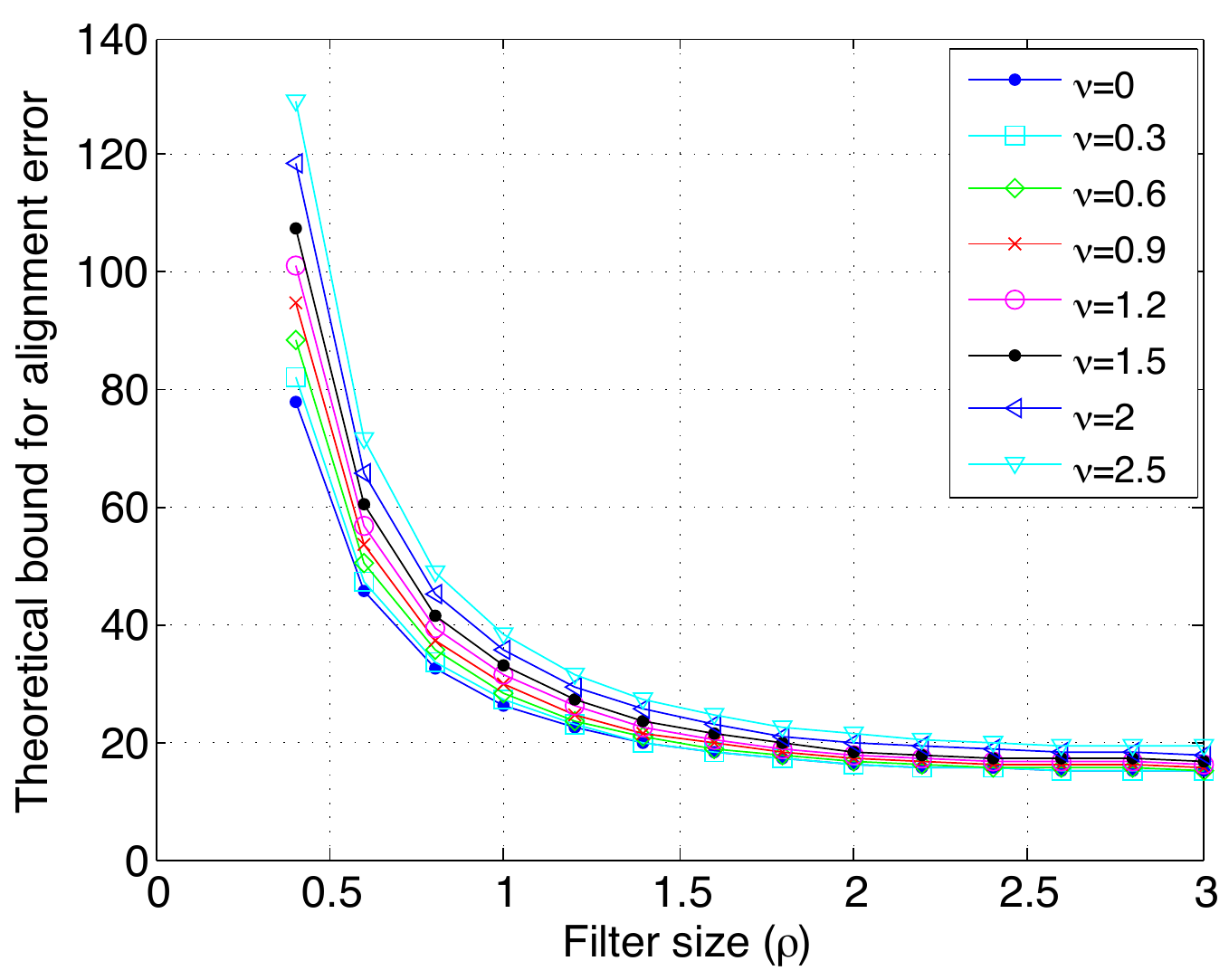}}
 \end{center}
 \caption{Alignment errors of real images for 4-D manifolds generated by translations, rotations, and scale changes.}
 \label{fig:real_4d}
\end{figure}

The results of the experiment show that the behavior of the alignment error for digital image representations is similar to the behavior of the error obtained with the analytical approximations of the images in $\mathcal{D}$. They mostly agree with the theoretical curves as well. The plots confirm that the increase in the alignment error with the noise level approaches an approximately linear rate at large values of the noise level as predicted by the theoretical results. The variation of the error with filtering is also in agreement with Theorem \ref{thm:dep_alerrbnd}, and different transformation models lead to different behaviors for the alignment error as in the previous set of experiments. Meanwhile, it is observable that the dependence of the alignment error $\hatalerrbnd$ on the filter size $\rho$ in these experiments is mostly determined by its first component $\hatalerrbnd_1$ related to manifold nonlinearity, even at large filter sizes. This is in contrast to the results obtained in the first setup with synthetically generated random patterns. The difference between the two setups can be explained as follows. Real images generally contain more high-frequency components than synthetical images generated in the smooth dictionary $\mathcal{D}$. These are captured with fine, small-scale atoms in the analytical approximations (the smallest atom scale used in this setup is $0.05$, while it is $0.3$ in the previous setup). The high-frequency components increase the manifold nonlinearity, which causes the error $\hatalerrbnd_1$ to be the determining factor in the overall error. In return, the positive effect of filtering that reduces the alignment error is more prominent in these experiments, while the non-monotonic variation of the error with the filter size is still observable at large noise levels or for the transformation model (\ref{eq:Mp_4d}) involving a scale change. The comparison of the two experimental setups shows that the exact variation of the error with filtering is influenced by the frequency characteristics of the reference patterns.

The plots in panels (d)-(f) of the figures also show that, at small filter sizes, experimental errors are relatively high and very similar for different noise levels, while this is not the case in the theoretical plots. This suggests that numerical errors in the estimation of the tangent vectors with finite differences must have some influence on the overall error in practice, which is not taken into account in the theoretical bound. This error is higher for images with stronger high-frequency components and diminishes with smoothing (see, e.g., \cite{Brandt94} for the effect of smoothing on the bias in the estimation of image derivatives with finite difference methods, and the study in \cite{Wakin05}, which shows that iterative smoothing is useful for the registration of non-differentiable images). Lastly, one can observe that the alignment errors obtained with digital images are slightly larger than the alignment errors given by the analytic approximations of the images. This can be explained by the difference in the numerical computation of the tangent vectors in these two experimental settings. The analytic representation of the images in terms of parametric Gaussian atoms permits a more accurate computation of the tangent vectors, while the numerical interpolations employed in the computation of the tangents in the digital setting create an additional error source.

The overall conclusions of the experiments can be summarized as follows. The theoretical  alignment error upper bound given in Theorem \ref{thm:bnd_alignerrTD} gives a numerically pessimistic estimate of the alignment error as it is obtained with a worst-case analysis. However, it reflects well the actual dependence of the true alignment error both on the noise level and the filter size, and the results confirm the approximate variation rates given in Theorem \ref{thm:dep_alerrbnd}. The theoretical upper bounds can be used in the determination of  appropriate filter sizes in hierarchical image registration with tangent distance.

\subsection{Image classification}

We now experimentally study the image classification performance when manifold distances are computed with registration based on the tangent distance method.

In the first experiment, we classify a data set of synthetic images. We experiment on two classes of images. The reference pattern of each class consists of 20 randomly chosen Gaussian atoms such that 16 of the atoms are common between the two classes and 4 atoms are specific to each class. This configuration has the purpose of simulating a setting where the distinction between different classes stems from class-specific features, meanwhile different classes have some common features as well, which poses a challenge for classification. We then generate a set of test patterns that lie between the transformation manifolds of the two reference patterns. The test patterns are generated such that their true class labels are given by the class label of the closer manifold as in (\ref{eq:defn_true_classlab}). We then classify the test patterns with the tangent distance method by estimating the transformation parameters in one step using the low-pass filtered versions of the reference and test patterns. The class labels of the test patterns are then estimated as in (\ref{eq:est_classlab_td_multis}). We conduct the experiment on the transformation models in (\ref{eq:Mp_2d})-(\ref{eq:Mp_4d}) and test the classification accuracy at different filter sizes. In Figures \ref{fig:misclass_rate_rand_2D}, \ref{fig:misclass_rate_rand_3D} and \ref{fig:misclass_rate_rand_4D}, the percentage of misclassified test patterns is plotted with respect to the filter size, for these three transformation models respectively. Each plot is obtained by averaging the results of 400 repetitions of the experiment with randomly generated reference and test patterns. In order to interpret the variation of the experimental misclassification rate with the filter size in light of the results in Section \ref{ssec:class_td}, we define a  function 
\begin{equation}
\MderUB_m \, \hatMsecderUB_m \ \lambdamin^{-1} \ \big( [ \hatGij^m (\tparref^m) ] \big) 
\left( \half \,  \sqrt{\tr( [ \hatGij^m (\tparref^m) ] )} \ \| \hattparopt - \tparref \|_1^2
+  \sqrt{d} \ \| \hatnoiseopt_m \|   \   \| \hattparopt - \tparref \|_1 \right)
\label{eq:misclass_error_metric}
\end{equation}
for the test patterns, where $\| \hatnoiseopt_m \|$ is the distance between the filtered test pattern $\hattargetp$ and the transformation manifold $\M(\hatp^m)$ of the filtered reference pattern representing class $m$. Comparing the function in (\ref{eq:misclass_error_metric}) with the misclassification probability bound in (\ref{eq:misclass_prob_filt}), one can observe that they have the same variation with the filter size $\rho$, while it is easier to compute  (\ref{eq:misclass_error_metric}) experimentally. As it provides a measure for the misclassification probability, we call the expression in (\ref{eq:misclass_error_metric}) the ``misclassification likeliness'' function. The average value of the misclassification likeliness (\ref{eq:misclass_error_metric}) is plotted in Figures \ref{fig:misclass_prob_rand_2D}, \ref{fig:misclass_prob_rand_3D} and \ref{fig:misclass_prob_rand_4D}, respectively for the transformation models in (\ref{eq:Mp_2d})-(\ref{eq:Mp_4d}). Comparing panels (a) and (b) of Figures \ref{fig:misclass_results_rand_2D} - \ref{fig:misclass_results_rand_4D}, we observe that the variation of the experimental misclassification probability with filtering agrees with that of the analytical misclassification likeliness (\ref{eq:misclass_error_metric}). This shows that the misclassification probability upper bound in (\ref{eq:misclass_prob_filt}) captures well the behavior of the actual misclassification probability. Furthermore, as the misclassification likeliness is linearly proportional to the alignment error bound, we observe that the classification performance of the tangent distance method is indeed closely related to its alignment performance. The experimental results confirm that the misclassification probability has a non-monotonic variation with the filter size as predicted by the theoretical results of Section \ref{ssec:class_td}, and the optimal filter size minimizing the misclassification probability is in the vicinity of the filter size that minimizes the misclassification likeliness.

\begin{figure}[t]
\begin{center}
     \subfigure[]
       {\label{fig:misclass_rate_rand_2D}\includegraphics[height=5cm]{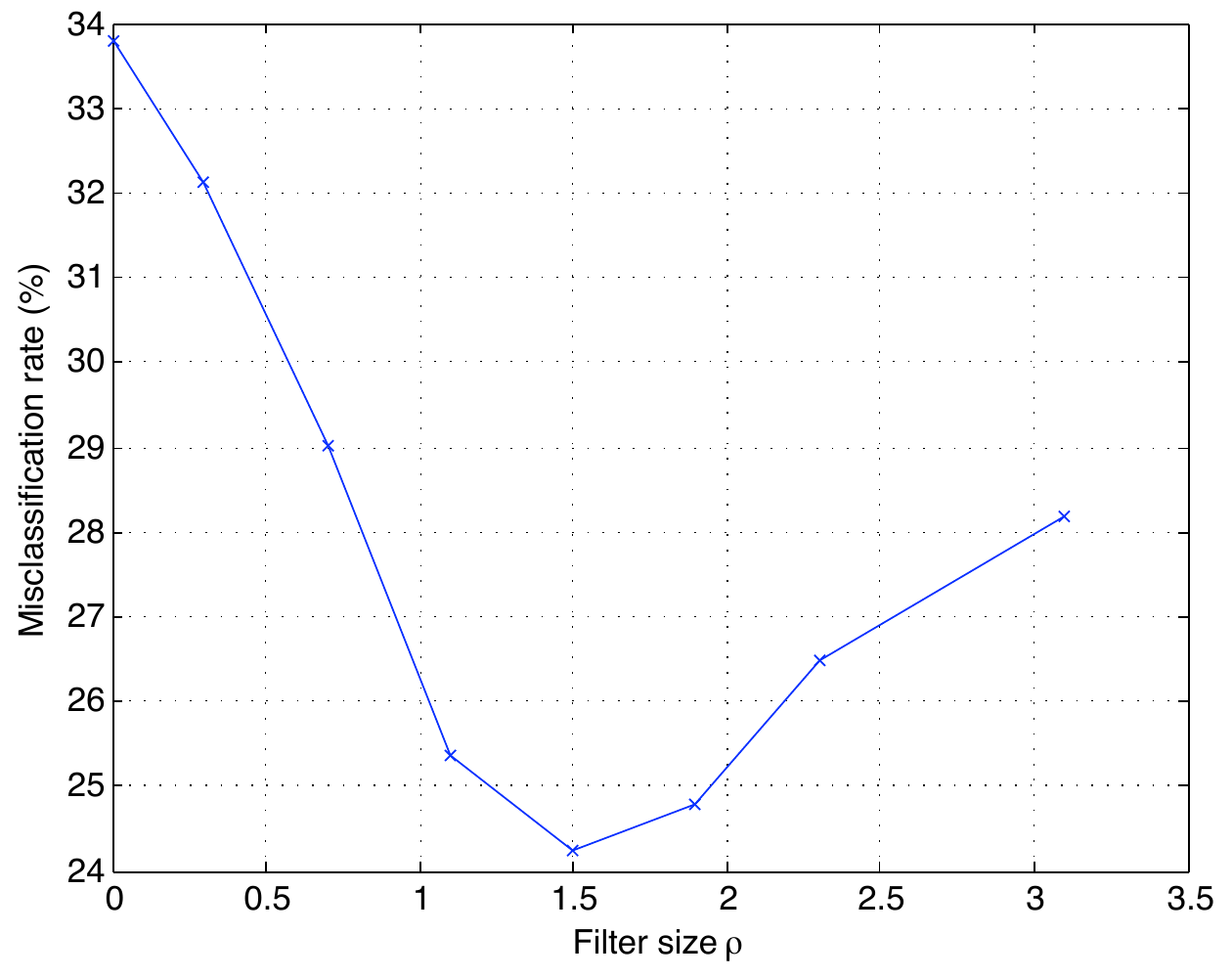}}
     \subfigure[]
       {\label{fig:misclass_prob_rand_2D}\includegraphics[height=5cm]{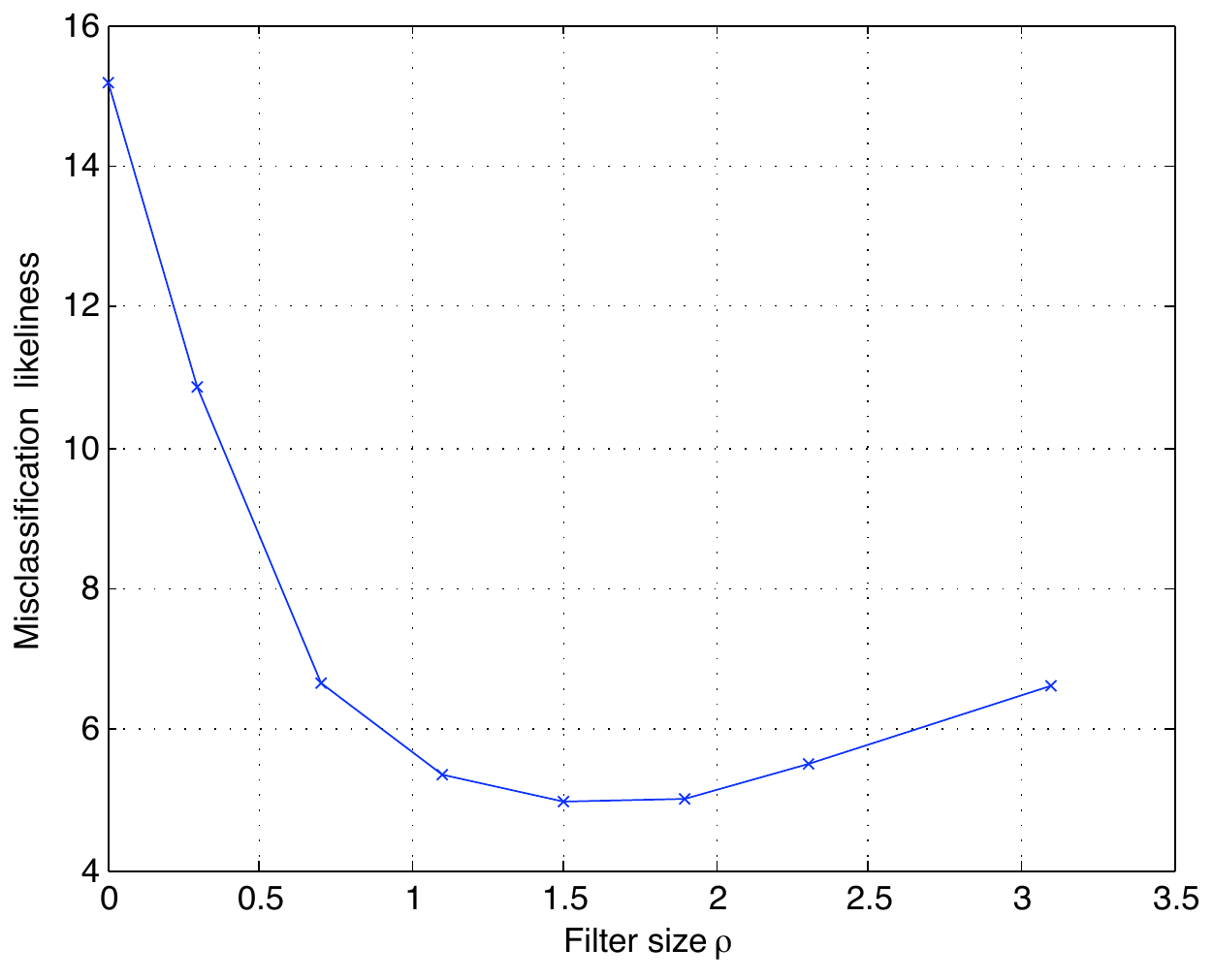}}
 \end{center}
 \caption{Classification results for random patterns and 2-D manifolds generated by translations.}
 \label{fig:misclass_results_rand_2D}
\end{figure}

\begin{figure}[]
\begin{center}
     \subfigure[]
       {\label{fig:misclass_rate_rand_3D}\includegraphics[height=5cm]{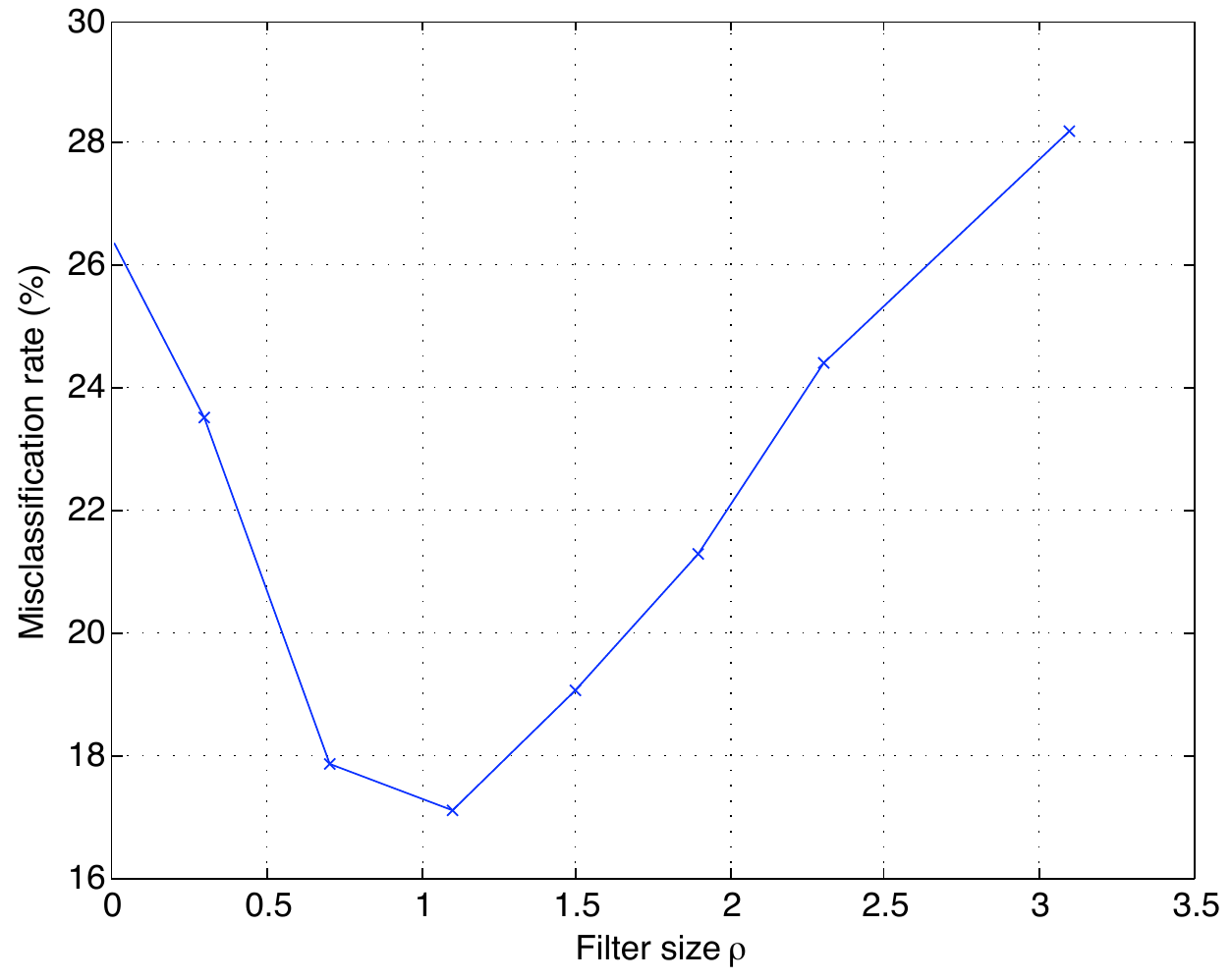}}
     \subfigure[]
       {\label{fig:misclass_prob_rand_3D}\includegraphics[height=5cm]{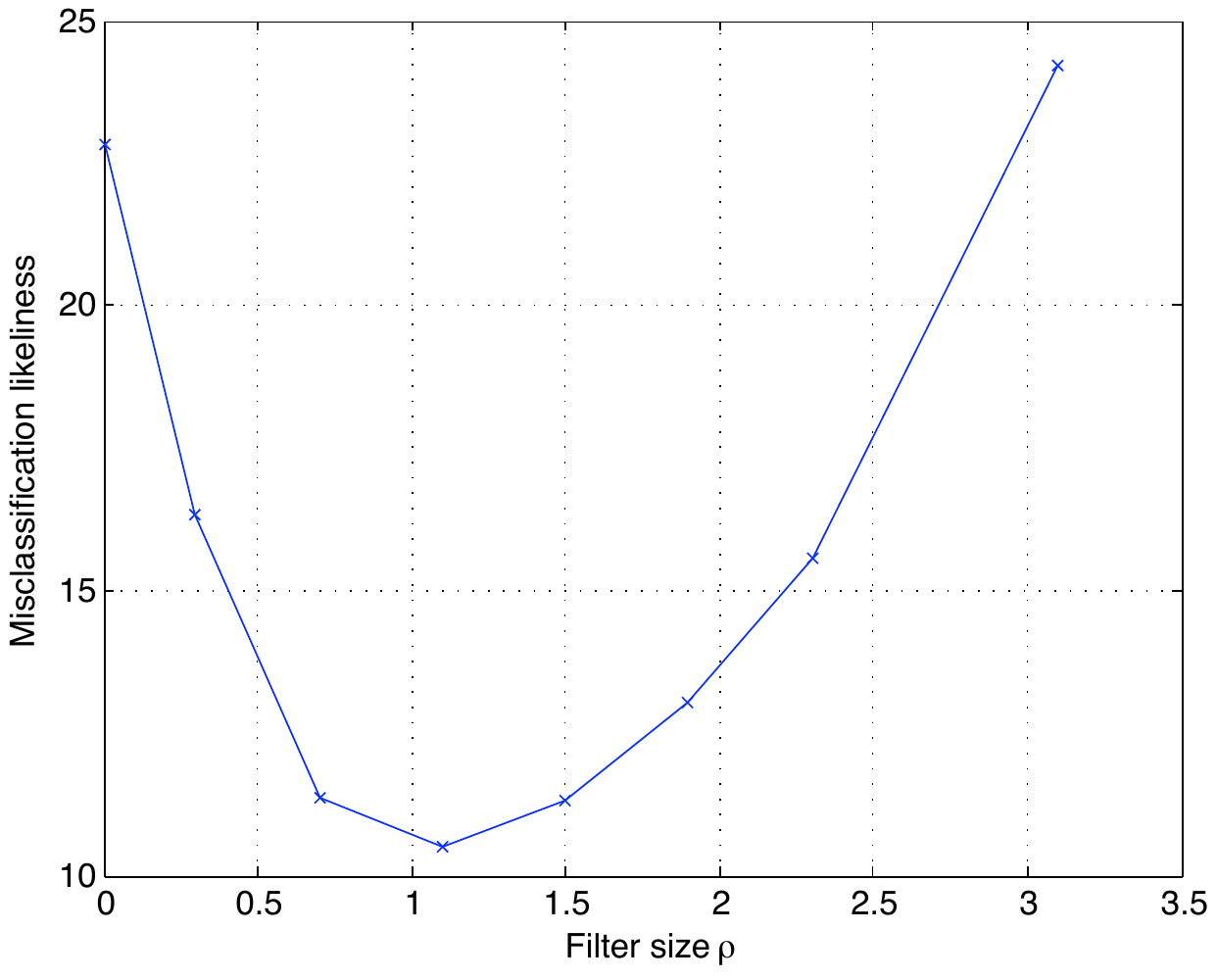}}
 \end{center}
 \caption{Classification results for random patterns and 3-D manifolds generated by translations and rotations.}
 \label{fig:misclass_results_rand_3D}
\end{figure}

\begin{figure}[]
\begin{center}
     \subfigure[]
       {\label{fig:misclass_rate_rand_4D}\includegraphics[height=5cm]{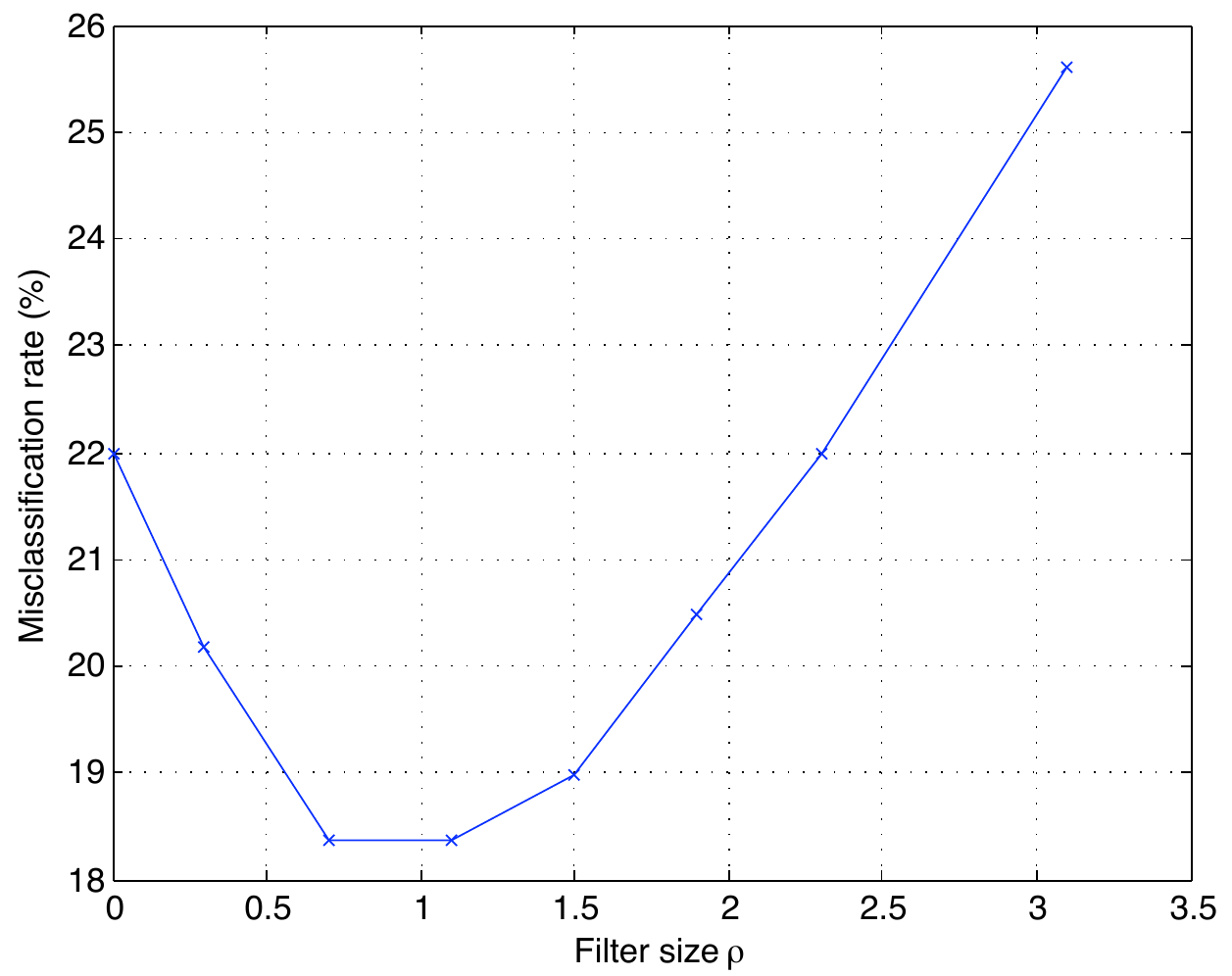}}
     \subfigure[]
       {\label{fig:misclass_prob_rand_4D}\includegraphics[height=5cm]{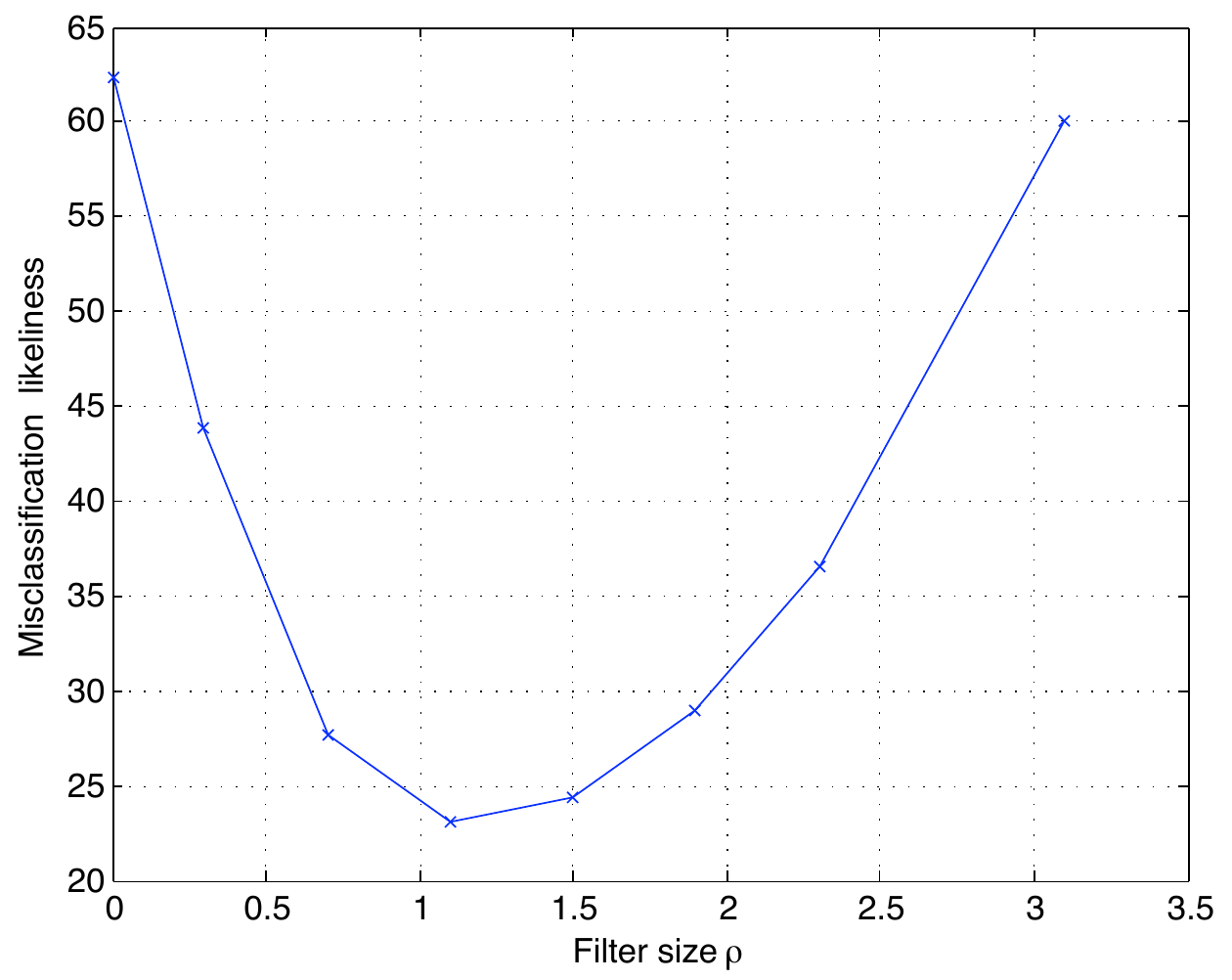}}
 \end{center}
 \caption{Classification results for random patterns and 4-D manifolds generated by translations, rotations, and scale changes.}
 \label{fig:misclass_results_rand_4D}
\end{figure}

Next, we study the classification performance of the tangent distance method on a data set of handwritten digit images taken from the MNIST database \cite{lecun98}. We experiment on the images of the 2, 3, 5, 8, and 9 digits, each of which represents a different class. We randomly choose a reference image among the training samples of each class. The test images are formed by applying a random geometric transformation on randomly selected test samples in the database. We classify the test images by estimating their distance to the transformation manifolds of the reference images with the tangent distance method for different filter sizes as in (\ref{eq:est_classlab_td_multis}). The results obtained for the geometric transformation models in (\ref{eq:Mp_2d})-(\ref{eq:Mp_4d}) are presented respectively in Figures \ref{fig:misclass_results_digit_2D}-\ref{fig:misclass_results_digit_4D}. Panels (a) and (b) of the figures show the experimental misclassification probability and the misclassification likeliness function (\ref{eq:misclass_error_metric}), which are the average of 1000 repetitions of the experiment with different reference and test images. The behavior of the experimental misclassification probability as a function of the filter size is seen to be similar to that of the misclassification likeliness. Meanwhile, in contrast to the results obtained on synthetic smooth patterns (Figures \ref{fig:misclass_results_rand_2D}-\ref{fig:misclass_results_rand_4D}), the best classification performances are obtained at large filter sizes for the digit images. This is in line with the results of the image alignment experiments with real images in Section \ref{ssec:exp_align_real}, where the error resulting from  manifold nonlinearity has been seen to be the determining factor in the overall behavior of the alignment error. Indeed, the high-frequency components may be prominent in real images. Since the digit images used in the experiments of Figures \ref{fig:misclass_results_digit_2D}-\ref{fig:misclass_results_digit_4D} also have quite nonlinear manifolds as a result of their frequency characteristics, their misclassification rate, as well as their alignment error, reaches its minimum value at large values of the filter size.

\begin{figure}[t]
\begin{center}
     \subfigure[]
       {\label{fig:misclass_rate_digit_2D}\includegraphics[height=5cm]{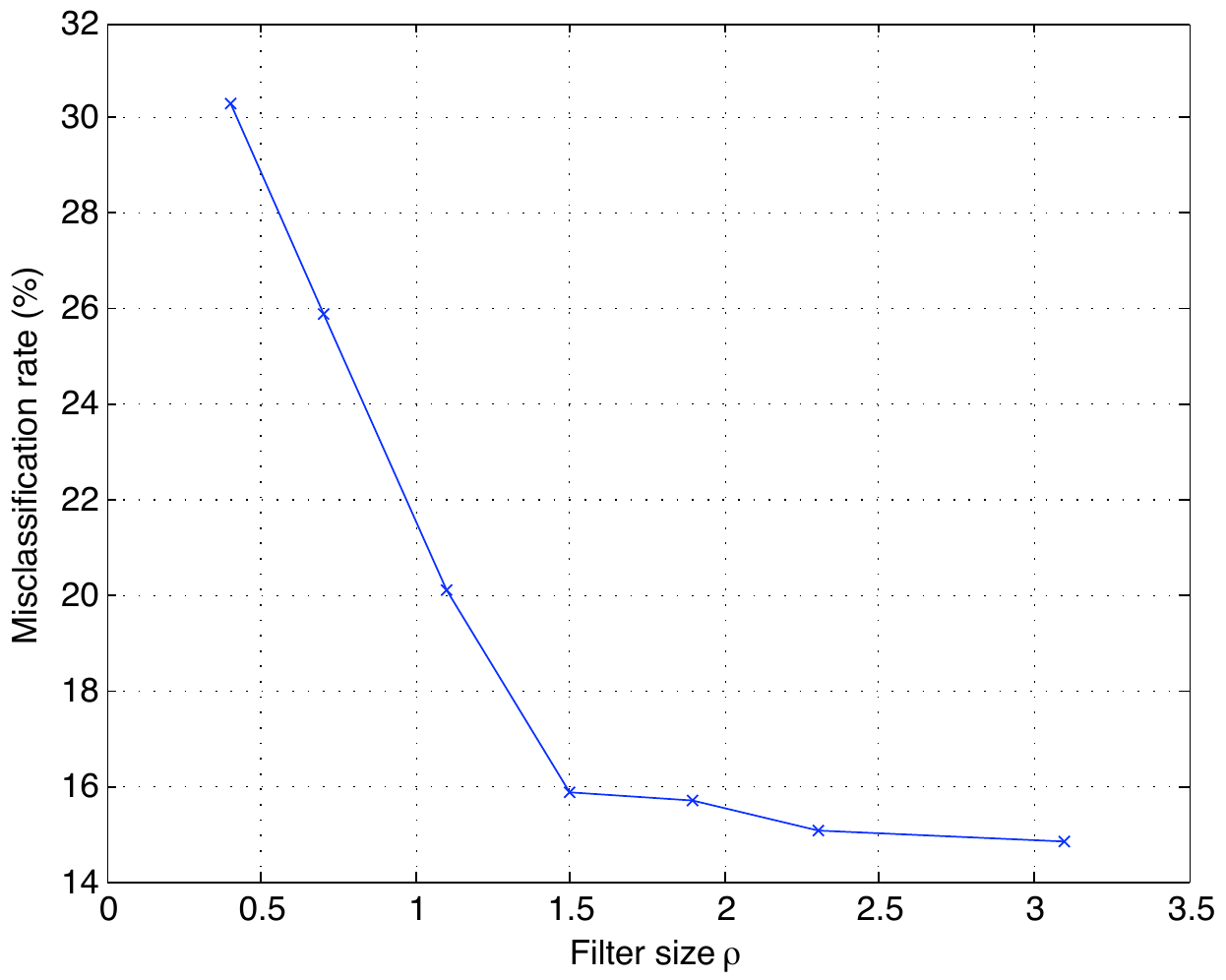}}
     \subfigure[]
       {\label{fig:misclass_prob_digit_2D}\includegraphics[height=5cm]{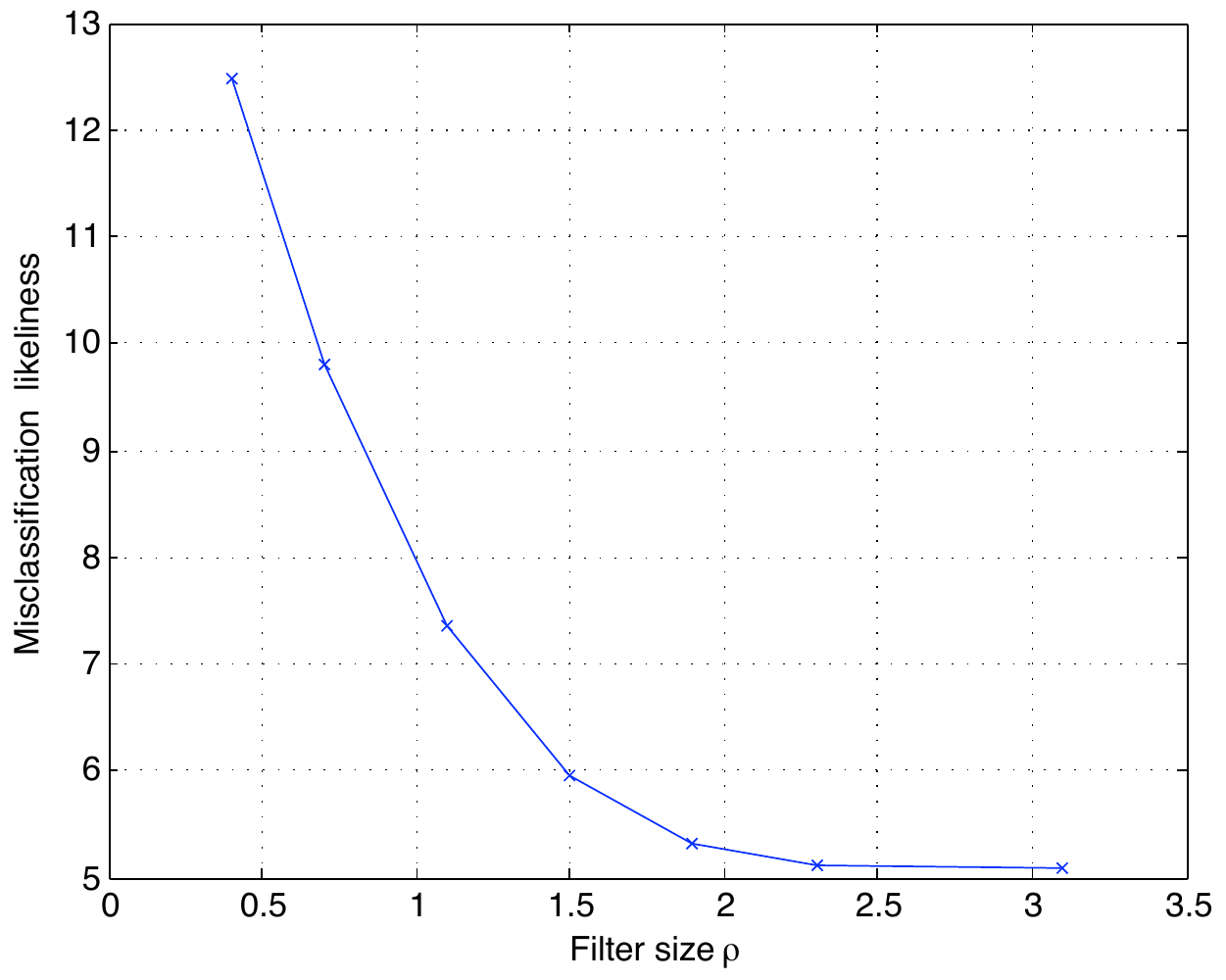}}
 \end{center}
 \caption{Classification results for digit images and 2-D manifolds generated by translations.}
 \label{fig:misclass_results_digit_2D}
\end{figure}

\begin{figure}[]
\begin{center}
     \subfigure[]
       {\label{fig:misclass_rate_digit_3D}\includegraphics[height=5cm]{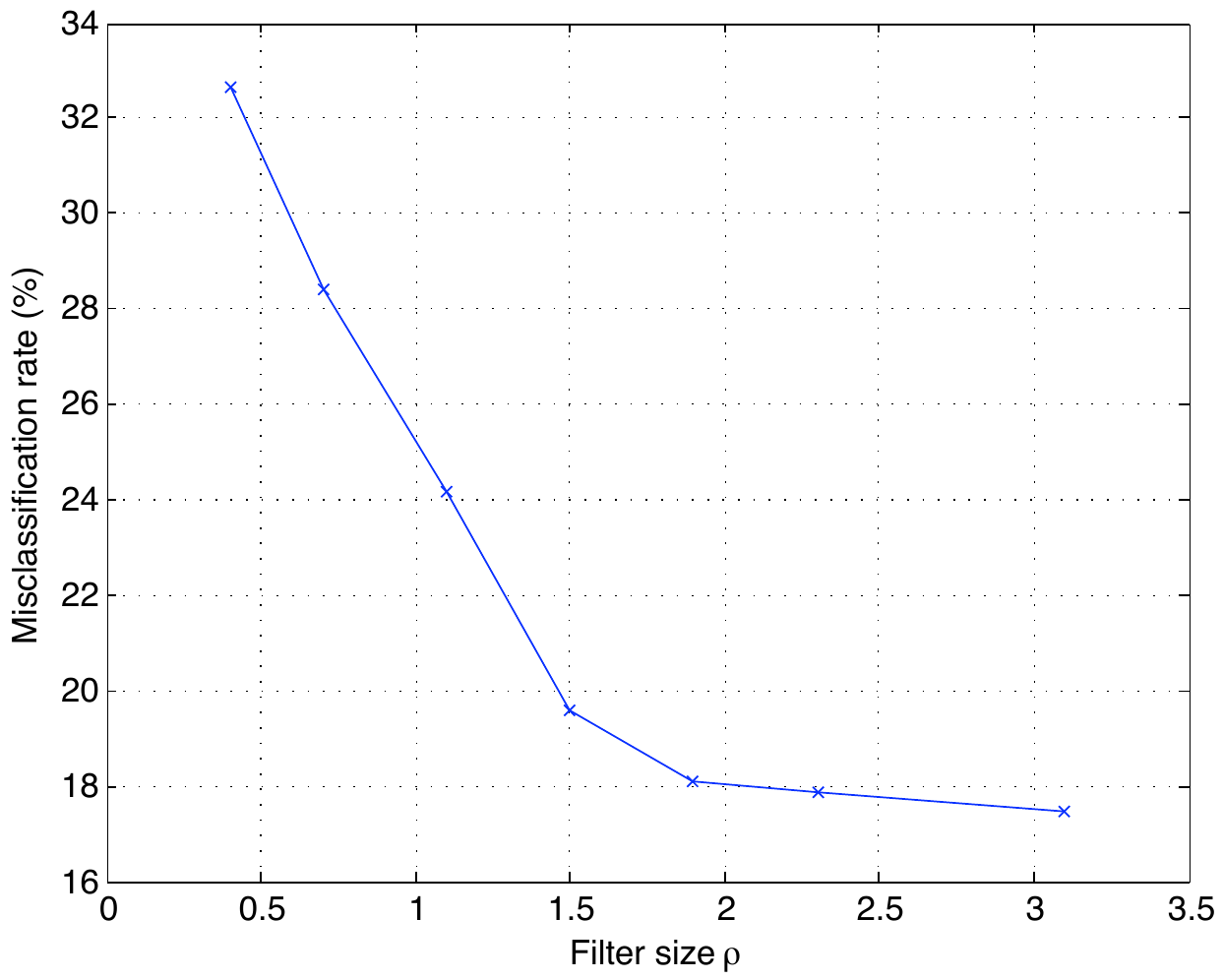}}
     \subfigure[]
       {\label{fig:misclass_prob_digit_3D}\includegraphics[height=5cm]{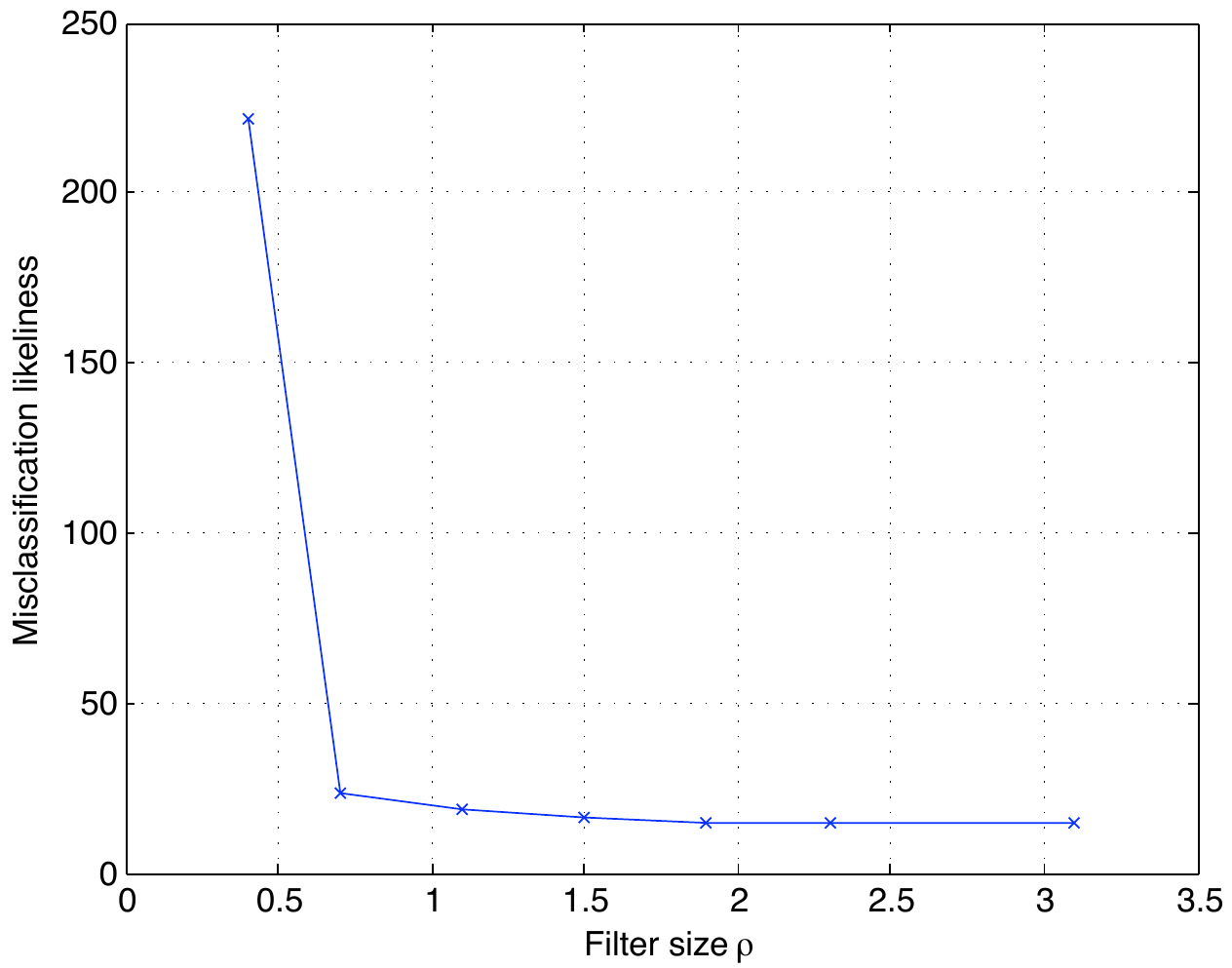}}
 \end{center}
 \caption{Classification results for digit images and 3-D manifolds generated by translations and rotations.}
 \label{fig:misclass_results_digit_3D}
\end{figure}

\begin{figure}[]
\begin{center}
     \subfigure[]
       {\label{fig:misclass_rate_digit_4D}\includegraphics[height=5cm]{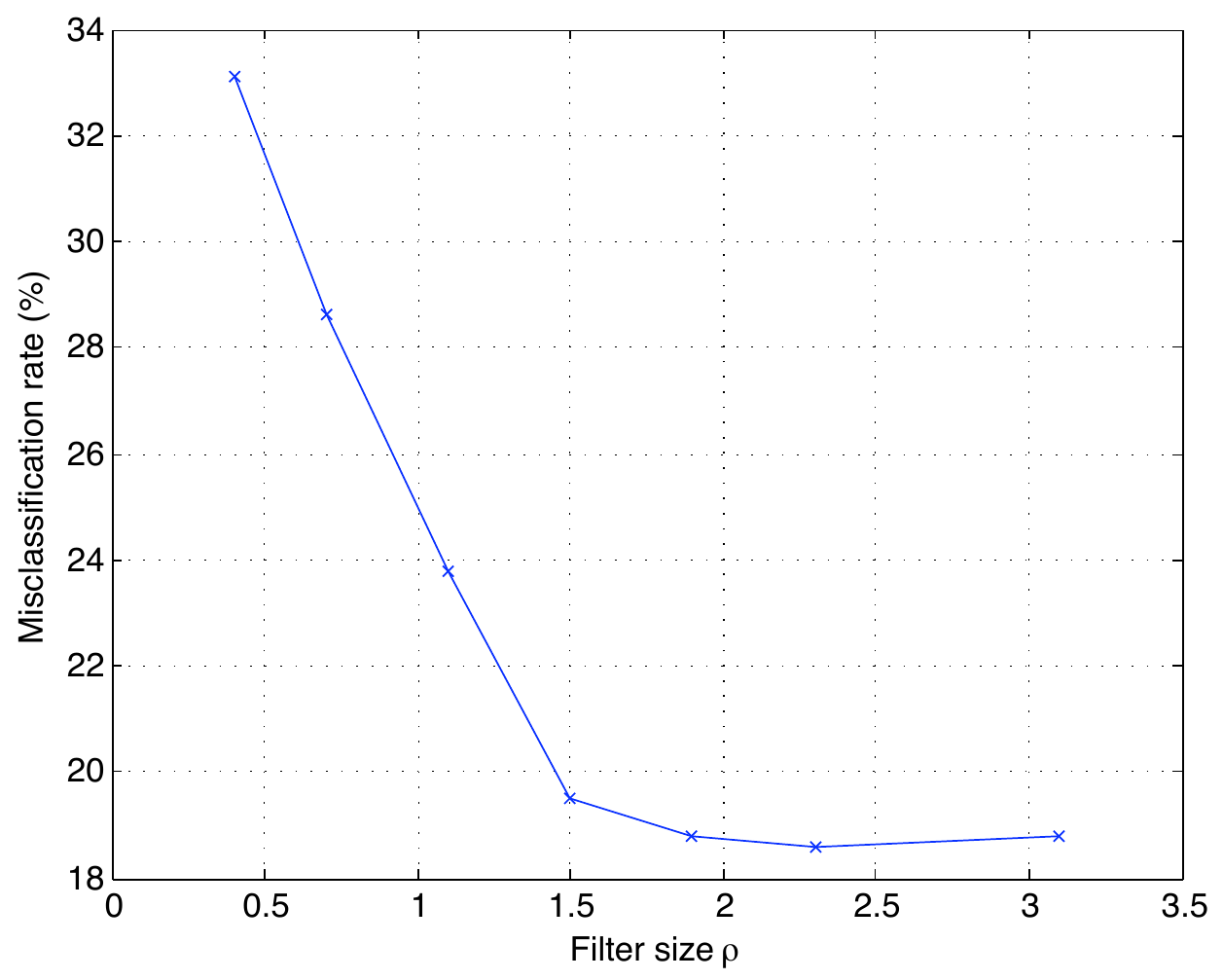}}
     \subfigure[]
       {\label{fig:misclass_prob_digit_4D}\includegraphics[height=5cm]{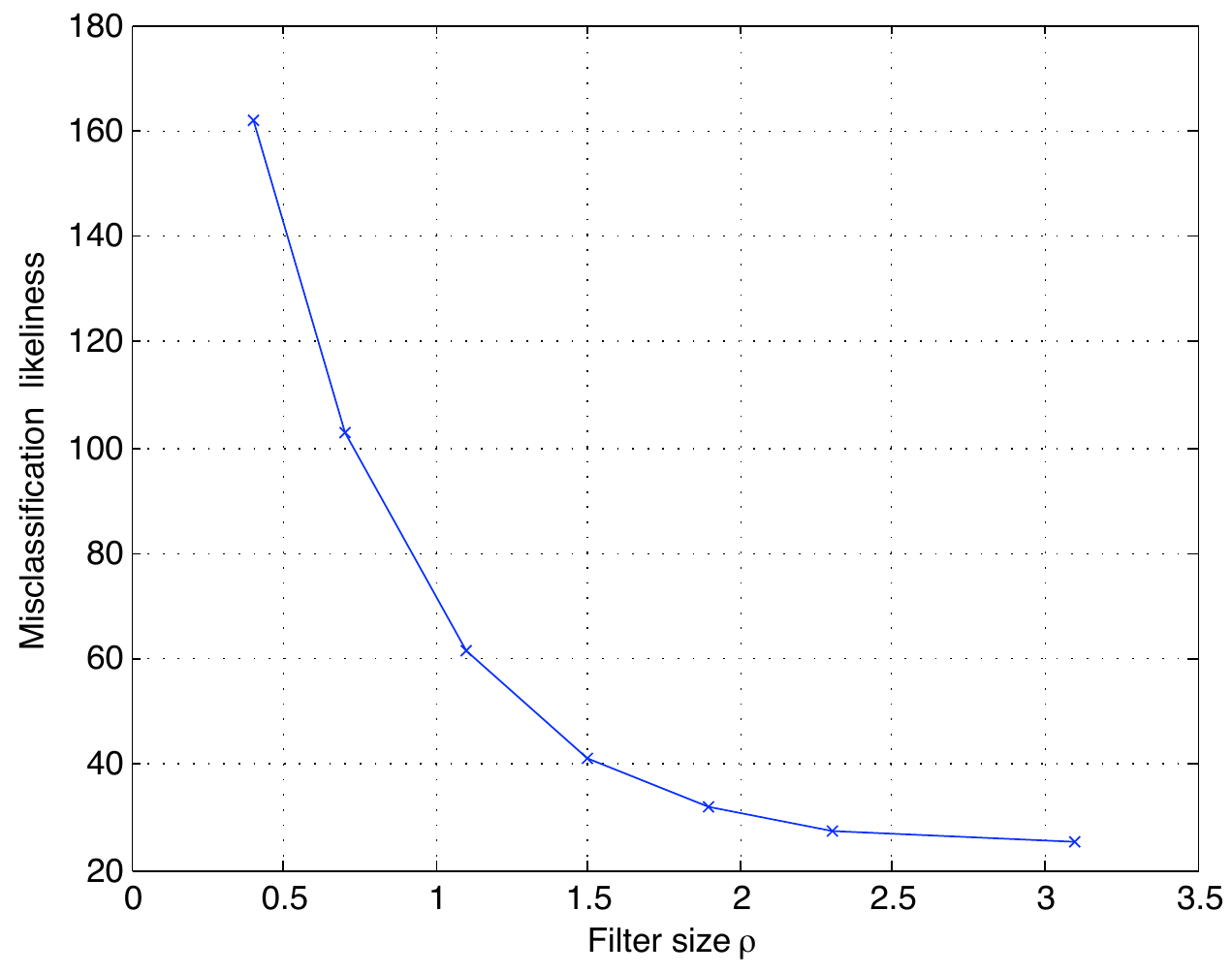}}
 \end{center}
 \caption{Classification results for digit images and 4-D manifolds generated by translations, rotations, and scale changes.}
 \label{fig:misclass_results_digit_4D}
\end{figure}

\section{Discussion of Related Work}
\label{ch:tan_dist:sec:discussion}
 
%We have derived an upper bound for the alignment error of the tangent distance method for generic transformation models. Our analysis shows that the alignment error decreases with the filter size $\rho$ for small values of $\rho$. However, in the presence of noise, the error starts increasing with $\rho$ at relatively large filter sizes and there exists an optimal value of $\rho$ that minimizes the alignment error. %We have also shown that the alignment error increases with the noise level of the target image.

Although the tangent distance method is frequently used in image registration and image analysis applications, its performance has not been theoretically studied for general transformation models to the best of our knowledge. A brief overview of the related literature is as follows. 

%There are several works  that present a performance analysis of image registration, which, however, study the effect of noise on alignment accuracy at a single scale. For instance, studies such as \cite{Robinson04} and \cite{Yetik06} derive Cr\'amer-Rao lower bounds of generic registration methods for several geometric transformation models. 

%We now discuss some previous studies about the performance of image registration. 

We begin with the works that analyze the dependence of the alignment error on noise. First, the study in \cite{Robinson04} derives the Cr\'amer-Rao lower bound (CRLB) for the registration of two images that differ by a 2-D translation. The CRLB gives a general lower bound for the MSE of any estimator; therefore, the lower bounds derived in \cite{Robinson04} are valid for all registration algorithms  that aim to recover the translation between two images. A Gaussian noise model is assumed in  \cite{Robinson04}, and the CRLB of a translation estimator is shown to be proportional to the noise variance. One can consider the noise standard deviation in the analysis in \cite{Robinson04} to be proportional to our noise level parameter $\noiselev$, which implies that the alignment error has a lower bound of $O(\noiselev)$. Then, the study in \cite{Yetik06} explores the CRLB of registration for a variety of geometric transformation models and shows that the linear variation of the CRLB with the noise level derived in \cite{Robinson04} for translations can be generalized to several other models such as rigid, shear and affine transformations. Being a generic bound valid for any estimator, the Cr\'amer-Rao lower bound is also valid for the tangent distance method. In our main result Theorem \ref{thm:dep_alerrbnd}, the second component $\hatalerrbnd_2$ of the alignment error, which is related to image noise, increases at a rate of $O(\noiselev)$ with the noise level $\noiselev$ for any geometric transformation model. Therefore, the results in  \cite{Robinson04} and \cite{Yetik06} are consistent with ours.\footnote{Note that we only concentrate on the alignment error caused by the linearization of the manifold in this work. In general, the upper bound on the alignment error due to the change in the actual projection onto the manifold as a result of noise can be above the linear rate $O(\noiselev)$, e.g., as shown in \cite{Vural12}, \cite{Sabater11}.}  Finally, let us remark the following about the variation of $\hatalerrbnd_2$ with the filter size.  The studies  \cite{Robinson04} and \cite{Yetik06} show that the CRLB of transformation estimators increases when the magnitudes of the spatial derivatives of patterns decrease. Since low-pass filtering reduces the magnitudes of spatial derivatives, it increases the MSE of estimators that compute the transformation parameters between an image pair. Similar results can be found in our previous work \cite{Vural12}, where we show that the error due to noise in the estimation of 2-D translations with descent-type algorithms is amplified with filtering (however, this previous study does not assume a linearization of the manifold and focuses merely on the perturbation on the global minimum of the alignment objective function caused by the noise). Our main result in this paper, which indicates that the error component $\hatalerrbnd_2$ associated with image noise increases with filtering, is in line with these previous works.

Next, the scope of the previous studies that examine the effect of manifold linearizations (e.g., \cite{Robinson04}, \cite{Brandt94}, \cite{Pham05}) is confined to the context of gradient-based optical flow estimation. Indeed, block-based optical flow estimation methods can be regarded as the restriction of the tangent distance method to  estimate 2-D translations between image patches. Our study differs from these analyses in that it considers arbitrary transformation models while characterizing the influence of the image noise on the alignment performance in a multiscale setting (by including the effect of filtering in the analysis). We now briefly discuss some of these results in relation with our work. 

%Let us now compare our results with some previous analyses on the performance of gradient-based methods in optical flow computation, which can be regarded as the restriction of the tangent distance method to  estimate 2-D translations between image patches. 

The work  \cite{Robinson04} studies the bias on gradient-based estimators, which employ a first-order approximation of the image intensity function. The bias is the difference between the expectation of the translation parameter estimates and the true translation parameters, and it results from the first-order approximation of the image intensity function. It is therefore associated with the first error term $\hatalerrbnd_1$ in Theorem \ref{thm:dep_alerrbnd} in our analysis. Note that the second error term $\hatalerrbnd_2$ results from image noise and is related to the variance of the estimator when a zero-mean random noise model is assumed. It is shown in \cite{Robinson04} that the bias is more severe if the image has larger bandwidth, i.e., if it has stronger high-frequency components. Hence, as smoothing the images with a low-pass filter reduces the image bandwidth, it decreases the bias. The studies in \cite{Kearney87} and \cite{Brandt94} furthermore report that smoothing diminishes the systematic error in the estimation of the image gradients from finite differences in optical flow computation, as it reduces the second and higher-order derivatives of the image intensity function. 
The results in \cite{Robinson04} are consistent with our analysis, which shows that the component of the alignment error associated with manifold nonlinearity decreases with the filter size $\rho$. Our result is however valid not only for translations, but for other transformation models as well. Moreover, it provides an exact rate of decrease for the error, which is given by $O\left( (1+\rho^2)^{-1/2} \right) $ for translations, and $O\left(1 + (1+\rho^2)^{-1/2} \right) $ for other transformation models. 
%The analysis in \cite{Robinson04} reports that the bias due to series truncation has a polynomial dependence on the amount of translation. In the bound given in Theorem \ref{thm:bnd_alignerrTD}, the alignment error term $\alerrbnd_1$ associated with manifold nonlinearity is seen to be proportional to the square $\| \tparopt - \tparref  \|_{1}^2$ of the distance between the transformation parameters. This quadratic dependence is due to the fact that we have used a second-order approximation of the transformation manifold; a higher-order approximation clearly yields a polynomial dependence of higher-degree as obtained in \cite{Robinson04}. 

Finally, the analysis in \cite{Lefebure01} studies the convergence of multiscale gradient-based registration methods where the image pair is related with a 2-D translation. It is shown that, for sufficiently small translations, coarse-to fine gradient-based registration algorithms converge to the globally optimal solution if the images are smoothed with ideal low-pass filters such that the filter bandwidth is doubled in each stage of the pyramid. However, this convergence guarantee is limited to an ideal noiseless setting where the target image is exactly a translated version of the reference image, whereas the convergence guarantee derived in our study is valid for also noisy settings and arbitrary geometric transformation models.

\section{Conclusion}
\label{ch:tan_dist:sec:conclusion}

We have presented a first complete performance analysis of the tangent distance method, which uses a first-order approximation of the transformation manifold in the estimation of the geometric transformation between a pair of images. We have first derived an upper bound for the alignment error and analyzed its variation with the noise level and the size of the low-pass filter used for smoothing the images in hierarchical registration algorithms. We have shown that the alignment error generally has a non-monotonic variation with the filter size due to the effects of smoothing on the image noise and the transformation manifold curvature. We have then used these results in order to establish some convergence guarantees for the hierarchical tangent distance algorithm. We have also derived some guidelines to choose the filter sizes optimally throughout the  algorithm. Our results show that, in order to optimize the performance of the hierarchical alignment method, the initial filter size in the beginning of the algorithm should increase with the amount of transformation and decrease with the noise level. The optimal geometric decay factor of the filter size (which is usually taken as $1/2$ in practice) then increases with the manifold curvature, the amount of transformation and the noise level. Finally, we have studied the classification performance of the tangent distance method and shown that the classification accuracy is expected to vary similarly to the alignment error. Our treatment is generic and valid for arbitrary geometric transformation models, and the theoretical results are confirmed by experiments. The presented study provides important insights for the understanding of multiscale registration methods that are based on manifold linearizations, and is helpful for optimizing the performance of such methods in image registration and image analysis applications.

\bibliographystyle{IEEEbib}
\bibliography{refs}

\appendix

%%%%% APPENDIX: PROOF OF THEOREM 1
\section{Proof of Theorem \ref{thm:bnd_alignerrTD}}
\label{app:pf_thm_bnd_alignerrTD}

\begin{proof}
Now we derive the upper bound on the alignment error given in Theorem \ref{thm:bnd_alignerrTD}. First, notice from (\ref{eq:tdreg_soln}) that the difference between the optimal and estimated transformation parameters is given by
\begin{equation}
\label{eq:al_err_v1_pfthm1}
\tparest^i - \tparopt^i =  \invGij(\tparref)  \langle \targetp - \ptranref, \derjptranref  \rangle - (\tparopt^i - \tparref^i ).
\end{equation}
%

%%%%%%%

%%%% LAURENT'S SUGGESTIONS
Now, given $g\in L^2(\Rsq)$, consider the function $h_g(\tpar)=\langle  \ptran, g  \rangle$. Applying the Taylor expansion of $h_g(\tpar)$ around the reference point $\tpar=\tparref$, we have
\begin{equation*}
h_g(\tpar) = \langle   \ptran, g \rangle =  \langle  \ptranref, g \rangle 
		+   \langle   \deriptranref, g \rangle   (\tpar^i -\tparref^i ) 
		+ \half  \langle  \partial_{ij} p_{\zeta_g}, g \rangle  (\tpar^i -\tparref^i ) (\tpar^j -\tparref^j )   
\end{equation*}
where $\zeta_g = \tparref + c_g (\tpar - \tparref)$ for some $c_g \in (0,1)$ that depends on $g$. 

Decomposing the target pattern as $\targetp = \ptranopt + \noiseopt $, and taking  $\tpar= \tparopt$ and $g=\derkptranref$ in the above equation, we obtain
\begin{equation*}
 \langle   \ptranopt, \derkptranref \rangle  = \langle  \ptranref, \derkptranref \rangle 
		+   \langle   \deriptranref, \derkptranref \rangle   (\tparopt^i -\tparref^i ) 
		+ \half  \langle  \partial_{ij} p_{\zeta_k}, \derkptranref \rangle  (\tparopt^i -\tparref^i ) (\tparopt^j -\tparref^j ).   
\end{equation*}
Defining 
\begin{equation}
\label{eq:lintranopt}
\lintranopt =  \deriptranref (\tparopt^i - \tparref^i) 
\end{equation}
and 
\begin{equation*}
\quadtranoptzetak = \half \partial_{ij} p_{\zeta_k} (\tparopt^i - \tparref^i) (\tparopt^j - \tparref^j),
\end{equation*}
one can rewrite the above equation as
\begin{equation*}
 \langle   \ptranopt, \derkptranref \rangle  = \langle  \ptranref, \derkptranref \rangle 
		+   \langle  \lintranopt , \derkptranref \rangle    
		+   \langle  \quadtranoptzetak, \derkptranref \rangle.   
\end{equation*}
From the expression of the alignment error in (\ref{eq:al_err_v1_pfthm1}), we get
\begin{equation*}
\begin{split}
\tparest^i - \tparopt^i &=  \invGik(\tparref)  \langle  \ptranopt + \noiseopt - \ptranref, \derkptranref  \rangle - (\tparopt^i - \tparref^i ) \\
				 &=  \invGik (\tparref)  \langle    \lintranopt  , \derkptranref \rangle 
				+  \invGik (\tparref)  \langle    \quadtranoptzetak  , \derkptranref \rangle 	
		+   \invGik (\tparref)   \langle  n , \derkptranref \rangle    
		- (\tparopt^i - \tparref^i ).  					 
\end{split}
\end{equation*}
However, the first and last terms in the above equation cancel each other as
\begin{equation*}
\begin{split}
 \invGik (\tparref)  \langle    \lintranopt  , \derkptranref \rangle 
	& = \invGik (\tparref)  \langle    \deriptranref (\tparopt^i - \tparref^i)  , \derkptranref \rangle 	 = \invGik (\tparref)  \langle    \deriptranref   , \derkptranref \rangle 	(\tparopt^i - \tparref^i) \\
 	& = \invGik (\tparref)  \Gik(\tparref)	(\tparopt^i - \tparref^i) 
	 =  (\tparopt^i - \tparref^i), 
\end{split}
\end{equation*}
yielding
\begin{equation*}
\begin{split}
\tparest^i - \tparopt^i  &=   \invGik (\tparref)  \langle    \quadtranoptzetak  , \derkptranref \rangle 	
		+   \invGik (\tparref)  \langle  n , \derkptranref \rangle.    				 
\end{split}
\end{equation*}

The norm of the alignment error can thus be upper bounded as 
\begin{equation}
\label{eq:align_err_bnd_v1}
\| \tparest - \tparopt  \| \leq \left\|  [ \invGij (\tparref) ] \, [\langle  \quadtranoptzetaj   ,  \ \derjptranref  \rangle]  \right\|
				+ \left\|   [ \invGij (\tparref) ] \, [\langle   \noiseopt  ,  \ \derjptranref  \rangle] \right\|
\end{equation}
where $[ \invGij (\tparref) ] $ is the matrix representation of the inverse metric and $[\langle  \quadtranoptzetaj   ,  \ \derjptranref  \rangle]$  and $ [\langle   \noiseopt  ,  \ \derjptranref  \rangle]$ respectively denote the $d \times 1$ vectors that contain $\langle  \quadtranoptzetaj   ,  \ \derjptranref  \rangle$ and $\langle   \noiseopt  ,  \ \derjptranref  \rangle$ in the $j$-th entry.

We proceed with finding an upper bound for the two terms in the above expression. The first term can be bounded as
\begin{equation}
\label{eq:quad_proj_vint}
\begin{split}
\left\|  [ \invGij (\tparref) ] \, [\langle  \quadtranoptzetaj   ,  \ \derjptranref  \rangle]  \right\|
	&\leq  \lambdamax \left(  [ \invGij(\tparref) ] \right) \  \big\|   [\langle  \quadtranoptzetaj   ,  \ \derjptranref  \rangle]   \big\| \\
	&= \lambdamin^{-1} \left(  [ \Gij(\tparref) ] \right) \  \big\|   [\langle  \quadtranoptzetaj   ,  \ \derjptranref  \rangle]   \big\|
\end{split}
\end{equation}
where $ \lambdamax (\cdot) $ and $ \lambdamin (\cdot) $ denote respectively the maximum and minimum eigenvalues of a matrix. We have
\begin{equation*}
\begin{split}
\big\|   [\langle  \quadtranoptzetaj   ,  \ \derjptranref  \rangle]   \big\| 
	&= \left(  \sumjd |  \langle \quadtranoptzetaj , \ \derjptranref    \rangle |^2  \right)^{1/2}
	\leq   \left(  \sumjd  \| \quadtranoptzetaj \|^2  \|  \derjptranref \|^2   \right)^{1/2} .
	% &= \| \quadtranoptzetaj \|   \sqrt{ \tr\big( [\Gij(\tparref) ]  \big) }
\end{split}
\end{equation*}
One can upper bound the norm of the quadratic term as
\begin{equation*}
\begin{split}
\| \quadtranoptzetak \| &=   \left\|  \half \partial_{ij} p_{\zeta_k} (\tparopt^i - \tparref^i) (\tparopt^j - \tparref^j)  \right\|
   		\leq \half  \sumid \sumjd \|  \partial_{ij} p_{\zeta_k}  \| \,  | \tparopt^i - \tparref^i    | \,   | \tparopt^j - \tparref^j    | \\
		& \leq \half \, \MsecderUB \  \| \tparopt - \tparref    \|_1^2  
\end{split}
\end{equation*}
which gives
\begin{equation*}
\big\|   [\langle  \quadtranoptzetaj   ,  \ \derjptranref  \rangle]   \big\| 
	  \leq    \half \, \MsecderUB \  \| \tparopt - \tparref    \|_1^2  \,   \sqrt{ \tr\big( [\Gij(\tparref) ]  \big) }.
\end{equation*}

Using this in (\ref{eq:quad_proj_vint}) we obtain
\begin{equation}
\label{eq:bnd_proj_quad}
\left\|  [ \invGij (\tparref) ] \, [\langle  \quadtranoptzetaj   ,  \ \derjptranref  \rangle]  \right\|  
\leq  \half \,  \  \MsecderUB  \ \lambdamin^{-1} \  \left(  [ \Gij(\tparref) ] \right) \sqrt{ \tr\big( [\Gij(\tparref) ]  \big) }  \  \| \tparopt - \tparref    \|^2_{1}.
\end{equation}
Having thus obtained an upper bound for the first additive term in (\ref{eq:align_err_bnd_v1}), we now continue with the second term $ \left\|   [ \invGij (\tparref) ] \, [\langle   \noiseopt  ,  \ \derjptranref  \rangle] \right\|$. First, remember from (\ref{eq:noiseopt_orth}) that the noise component $\noiseopt$ is orthogonal to the tangent space $T_{\tparopt} \M(p)$ at $\ptranopt$. The term $  \invGij(\tparref)  \langle   \noiseopt  ,  \ \derjptranref  \rangle $ gives the coordinates of the projection of $\noiseopt$ onto the tangent space $T_{\tparref} \M(p)$ at $\ptranref$. Due to manifold curvature, there is a nonzero angle between these two tangent spaces; therefore, the orthogonal projection of $\noiseopt$ onto $T_{\tparref} \M(p)$ is a nonzero vector in general. In the following, we derive an upper bound for the magnitude of this projection by looking at the change in the tangent vectors between the two manifold points $\ptranref$ and $\ptranopt$. Let us define
\begin{equation*}
\deltaderi := \deriptranref - \deriptranopt
\end{equation*}
which gives the change in the $i$-th tangent vector between the points  $\ptranref$ and $\ptranopt$. We have
\begin{equation}
\label{eq:bnd_proj_noise_v1}
\begin{split}
 \invGij(\tparref)  \langle   \noiseopt  ,  \ \derjptranref  \rangle 
 		&=  \invGij(\tparref)  \langle   \noiseopt  ,  \ \derjptranopt  \rangle
		 +   \invGij(\tparref)  \langle   \noiseopt  ,  \ \deltaderj  \rangle \\
		& =   \invGij(\tparref)  \langle   \noiseopt  ,  \ \deltaderj  \rangle
\end{split}
\end{equation}
since $\langle   \noiseopt  ,  \ \derjptranopt  \rangle = 0$ for all $j = 1, \cdots, d$. We now derive an upper bound for the norm of $\deltaderj$ as follows. Let us define a curve
$
p_{\tpar(t)}: [0, 1] \rightarrow \M(p)
$
such that
$
\tpar(t)=\tparopt + t (\tparref - \tparopt).
$
Hence, $p_{\tpar(0)} = \ptranopt$ and $p_{\tpar(1)} = \ptranref $. For each $i = 1, \cdots, d$ we have
\begin{equation*}
\begin{split}
\deriptranref &= \deriptranopt + \int_0^1 \frac{d \, \partial_{i} p_{\tpar(t)}  }{dt} \ dt 
		  =  \deriptranopt + \int_0^1  \partial_{ij} p_{\tpar(t)} \frac{d \tpar^j(t)}{dt} \ dt \\
		  &=  \deriptranopt + \int_0^1  \partial_{ij} p_{\tpar(t)}  (\tparref - \tparopt)^j   \ dt .
\end{split}
\end{equation*}
We thus get the following upper bound on $\| \deltaderi \|$
\begin{equation*}
\begin{split}
\| \deltaderi \| &= \left \|    \int_0^1  \partial_{ij} p_{\tpar(t)}  (\tparref - \tparopt)^j   \ dt    \right\| 
		    \leq \sumjd  \int_0^1  \left \|  \partial_{ij} p_{\tpar(t)}   \right\|  |  (\tparref - \tparopt)^j | \ dt \\
		   & \leq \ \  \sumjd  \MsecderUB \  |  (\tparref - \tparopt)^j | 
		    = \MsecderUB \ \| \tparref -  \tparopt \|_1.
\end{split} 
\end{equation*}
It follows from (\ref{eq:bnd_proj_noise_v1}) that 
\begin{equation*}
\begin{split}
 \left\|   [ \invGij (\tparref) ] \, [\langle   \noiseopt  ,  \ \derjptranref  \rangle] \right\|
	= \left \|   [ \invGij(\tparref) ] [\langle   \noiseopt  ,  \ \deltaderj  \rangle]  \right\|  
	\leq \ \lambdamin^{-1}  \left(  [ \Gij(\tparref) ] \right)   \big \|    [\langle   \noiseopt  ,  \ \deltaderj  \rangle]  \big\|  
\end{split}
\end{equation*}
where
\begin{equation*}
\big \|    [\langle   \noiseopt  ,  \ \deltaderj  \rangle]  \big\|  
		= \left(  \sumjd |  \langle \noiseopt  ,  \ \deltaderj \rangle  |^2   \right)^{1/2}
		\leq \  \left(  \sumjd   \noiselev ^2 \, \| \deltaderj \|^2    \right)^{1/2}.
\end{equation*}
Using the bound $\| \deltaderj \|  \leq \MsecderUB \ \| \tparref -  \tparopt \|_1$ above we get
\begin{equation*}
\big \|    [\langle   \noiseopt  ,  \ \deltaderj  \rangle]  \big\|   \leq \, \MsecderUB \, \sqrt{d}  \ \noiselev \ \| \tparref -  \tparopt \|_1
\end{equation*}
which gives
\begin{equation}
\label{eq:bnd_proj_noise}
 \left\|   [ \invGij (\tparref) ] \, [\langle   \noiseopt  ,  \ \derjptranref  \rangle] \right\|
	\leq  \MsecderUB \, \sqrt{d} \ \noiselev \ \lambdamin^{-1} \left(  [ \Gij(\tparref) ] \right)    \ \| \tparref -  \tparopt \|_1.
\end{equation}
This finishes the derivation of the upper bound on the norm of the projection of the noise component on $T_{\tparref} \M(p)$. We finally put together the results (\ref{eq:bnd_proj_quad}) and (\ref{eq:bnd_proj_noise}) in (\ref{eq:align_err_bnd_v1}) and get the stated bound on the norm of the alignment error $\| \tparest - \tparopt \|$
\begin{equation*}
\| \tparest - \tparopt \| \leq  \ \MsecderUB \ \lambdamin^{-1} \ \big( [ \Gij (\tparref) ] \big) 
\left( \half \,  \sqrt{\tr( [ \Gij (\tparref) ] )} \ \| \tparopt - \tparref  \|_{1}^2
+  \sqrt{d} \ \noiselev  \   \|  \tparopt - \tparref   \|_1
\right)
\end{equation*}
which concludes the proof. 
\end{proof}

%%%%% APPENDIX: PROOF OF THEOREM 2
\section{Proof of Theorem \ref{thm:dep_alerrbnd}}
\label{app:pf_thm_dep_alerrbnd}

In order to prove Theorem  \ref{thm:dep_alerrbnd}, we need to analyze the variation of $\hatalerrbnd$ with filtering and noise.  We begin with examining the dependence of each term in the expression of $\hatalerrbnd$ 
\begin{equation}
\label{eq:hatalerrbnd_app}
\hatalerrbnd = \hatMsecderUB \ \lambdamin^{-1} \ \big( [ \hatGij (\tparref) ] \big) 
\left( \half \,  \sqrt{\tr( [ \hatGij (\tparref) ] )} \ \| \hattparopt - \tparref  \|_{1}^2
+  \sqrt{d} \ \| \hatnoiseopt \|  \   \|  \hattparopt - \tparref   \|_1
\right)
\end{equation}
on the filter size $\rho$ and the initial noise level $\noiselev$ of the unfiltered target image. First, the curvature parameter $\hatMsecderUB$ of the smoothed manifold is given by
\begin{equation*}
\hatMsecderUB = \max_{i,j = 1, \cdots, d}  \  \sup_{\tpar \in \tpardom} \| \hatderijptran  \|
\end{equation*}
where $\hatderijptran$ denotes the second order derivative of the manifold of the smoothed pattern. Hence, if a uniform estimate that is valid for all $\tpar$ and $(i,j)$ can be found for the rate of variation of  $\| \hatderijptran  \|$ with the filter size $\rho$, the curvature parameter $\hatMsecderUB$ then also has the same order of variation with $\rho$. 

Next, the metric tensor of the smoothed manifold is given by $\hatGij(\tparref) = \langle \hatderiptranref, \hatderjptranref \rangle$, and its trace is
\[
\tr\big( [ \hatGij (\tparref) ] \big) = \sum_{i=1}^d  \| \hatderiptranref \|^2.
\]
Therefore, if the variation of $\| \hatderiptranref \|^2$ with the filter size $\rho$ can be characterized uniformly (in a way that is valid for all $\tparref$ and $i$), the trace $\tr\big( [ \hatGij (\tparref) ] \big)$ of the metric tensor will also have the same order of variation with $\rho$ as $\| \hatderiptranref \|^2$. %Since the trace is given by the sum of the eigenvalues, one can reasonably expect the smallest eigenvalue $\lambdamin \ \big( [ \hatGij (\tparref) ] \big)$ to have the same variation with $\rho$ as well. 

The smallest eigenvalue $\lambdamin \ \big( [ \hatGij (\tparref) ] \big)$ of the metric tensor is also expected to have the same variation with $\rho$. This can be observed, for instance, by decomposing the metric tensor into its diagonal and off-diagonal components and regarding the off-diagonal component as a perturbation on the diagonal one. The smallest eigenvalue $\lambdamin \ \big( [ \hatGij (\tparref) ] \big)$ can then be lower bounded as in \cite{Rohn98} in terms of the smallest diagonal element $\min_i \| \hatderiptranref \|^2 $ and the spectral radius of the off-diagonal component of the metric tensor consisting of the terms $ \langle \hatderiptranref, \hatderjptranref \rangle $, which is a simple application of the Gershgorin circle theorem. As  the variation of the off-diagonal elements is upper bounded by the variation of the diagonal elements due to Cauchy-Schwarz inequality, the smallest eigenvalue $\lambdamin \ \big( [ \hatGij (\tparref) ] \big)$ decays with $\rho$ at the same rate as $\| \hatderiptranref \|^2$.

Finally, the norm $\| \hatnoiseopt \|$ of the noise component of $\hattargetp$ depends on both the filter size $\rho$ and the initial noise level $\noiselev$ before filtering. 

We study now Equation (\ref{eq:hatalerrbnd_app}) in more details and derive first a relation between the norms $\| \hatderiptran \|$, $\| \hatderijptran \|$ of the first and second-order manifold derivatives and the norms  $\| \normgrad \hatp  \| $, $\| \normhess \hatp  \| $  of the gradient and Hessian magnitudes of the filtered reference pattern $\hatp$. We state the dependences of $\| \normgrad \hatp  \| $ and $\| \normhess \hatp  \| $ on the filter size $\rho$ in Lemma \ref{lem:var_derivs}, which is then used to obtain the variation of the manifold derivatives $\| \hatderiptran \| $, $\| \hatderijptran \| $ with $\rho$ in Corollary \ref{cor:dep_manderivs}. Next, we establish the dependence of the norm $\| \hatnoiseopt \|$ of the noise component on $\rho$ and $\noiselev$ in Lemma \ref{lem:dep_noiselev}. Finally, all of these results are put together in our main result Theorem \ref{thm:dep_alerrbnd}, where we present the rate of variation of the alignment error bound $\hatalerrbnd$ with the filter size $\rho$ and the initial noise level $ \noiselev  $ of the target image.

\subsection{Analysis of $\| \hatderiptran \|$ and $\| \hatderijptran \|$}
\label{app:tan_dist_anly_manderiv}

Let us begin with the computation of the terms $\| \hatderiptran \|$ and $\| \hatderijptran \|$. First, from the relation (\ref{eq:rel_geom_trans}), we have
\begin{equation*}
\ptran(X) = p(X') 
\end{equation*}
where $X'= \coordch_{\tpar}(X)$. Let us denote the transformed coordinates as  $X'=[x' \ y']^T$ and write the derivatives of the transformed coordinates with respect to the transformation parameters as
\begin{equation*}
\derixp = \frac{\partial x'}{\partial \tpari},
\qquad
\deriyp = \frac{\partial y'}{\partial \tpari},
\qquad
\derijxp = \frac{\partial^2 x'}{\partial \tpari \, \partial \tparj},
\qquad
\derijyp = \frac{\partial^2 y'}{\partial \tpari \, \partial \tparj}.
\end{equation*}
Also, let
\begin{equation*}
\begin{split}
\derxp(X') &= \frac{\partial \, p(X)}{\partial x} \bigg|_{X=X'} \ \ ,
\qquad 
\deryp(X') = \frac{\partial \, p(X)}{\partial y} \bigg|_{X=X'} \\
\derxxp(X') &= \frac{\partial^2 \, p(X)}{\partial x^2 }\bigg|_{X=X'} \ \ ,
\qquad 
\derxyp(X') = \frac{\partial^2 \, p(X)}{\partial x \, \partial y}\bigg|_{X=X'} \ \ ,
\qquad 
\deryyp(X') = \frac{\partial^2 \, p(X)}{\partial y^2} \bigg|_{X=X'} 
\end{split}
\end{equation*}
denote the partial derivatives of the reference pattern $p$ evaluated at the point $X'$. Then, the derivatives of the manifold $\M(p)$ at $\ptran$ are given by
\begin{equation*}
\begin{split}
\deriptran(X) &= \derxp(X') \derixp +  \deryp(X') \deriyp \\
\derijptran(X) &= \derxxp(X') \, \derixp \derjxp + \derxyp(X') \, (\derixp \derjyp + \derjxp \deriyp) + \deryyp(X') \, \deriyp \derjyp \\
		& \quad + \derxp(X') \, \derijxp + \deryp(X') \, \derijyp.
\end{split}
\end{equation*}
One can generalize this to the smoothed versions $\hatp$ of the reference pattern as
\begin{equation}
\label{eq:manif_derivs_expr}
\begin{split}
\hatderiptran(X) &= \hatderxp(X') \derixp +  \hatderyp(X') \deriyp \\
\hatderijptran(X) &= \hatderxxp(X') \, \derixp \derjxp + \hatderxyp(X') \, (\derixp \derjyp + \derjxp \deriyp) + \hatderyyp(X') \, \deriyp \derjyp \\
		& \quad + \hatderxp(X') \, \derijxp + \hatderyp(X') \, \derijyp.
\end{split}
\end{equation}
Notice that, in the above equations, the filtering applied on the reference pattern influences only the spatial derivatives of the reference pattern ($\hatderxp$, $\hatderyp$, $\hatderxxp$, $\hatderxyp$, $\hatderyyp$), whereas the derivatives of the transformed coordinates ($\derixp$, $\deriyp$, $\derijxp$, $\derijyp$) depend solely on the transformation model $\tpar$ and are constant with respect to the filter size $\rho$. Therefore, the variation of $\| \hatderiptran \|$ and $\| \hatderijptran \|$ with $\rho$ is mostly determined by the variation of the  spatial derivatives of the pattern with the filter size. We denote the gradient of $\hatp$ as  
\begin{equation*}
\nabla \hatp(X) = [ \hatderxp(X) \ \hatderyp(X)   ]^T
\end{equation*}
and the vectorized Hessian of $\hatp$ as
\begin{equation}
\label{eq:Hess_defn}
(\hess \hatp) (X) = [ \hatderxxp(X) \ \hatderxyp(X) \ \hatderxyp(X) \ \hatderyyp(X) ]^{T}.
\end{equation}
We then define the functions $\normgrad \hatp, \  \normhess \hatp:  \ \Rsq \rightarrow \mathbb{R}$
\begin{equation*}
\normgrad \hatp \, (X) = \| \nabla \hatp(X) \|, 
\qquad \qquad
\normhess \hatp \, (X) = \| (\hess \hatp) (X)  \|
\end{equation*}
which give the $\ell^2$-norms of the gradient and the Hessian of $\hatp$ at $X$. Since we assume that the first and second spatial derivatives of the pattern are square-integrable, the functions $\normgrad \hatp$ and $  \normhess \hatp$ are in $L^2(\Rsq)$. The equations in (\ref{eq:manif_derivs_expr}) show that the first derivatives of the manifold are proportional to the first derivatives of the pattern; and the second derivatives of the manifold depend linearly on both the first and the second derivatives of the pattern $p(X)$. One thus expects the $L^2$-norms of the manifold derivatives to be related to the $L^2$-norms of $ \normgrad \hatp$ and $\normhess \hatp$ as
\begin{equation}
\label{eq:dep_mander_patder}
\begin{split}
\| \hatderiptran \| &= O\left( \| \normgrad \hatp  \| \right)\\
\| \hatderijptran \| &= O\left( \| \normgrad \hatp \| + \| \normhess \hatp \|  \right)
\end{split}
\end{equation}
from the perspective of their dependence on the filter size $\rho$. These relations indeed hold and they are formally shown in Appendix \ref{app:tan_dist:sec:rel_manderiv_patderiv}.

Since we have established the connection between the manifold derivatives and the pattern spatial derivatives, it suffices now to determine how the spatial derivatives $ \| \normgrad \hatp  \| $ and  $\| N_\hess \hatp \|$  depend on the filter size $\rho$. In order to examine this, we adopt a parametric representation of the reference pattern $p$ in an analytic dictionary. Let 
\begin{equation}
\label{eq:dictManifoldExp}
\mathcal{D}  = \{ {\phi}_{\gamma}: \gamma = (\psi, \tau_x, \tau_y, \sigma_x, \sigma_y) \in \Gamma  \} \subset L^2(\mathbb{R}^2)
\end{equation}
be a parametric dictionary manifold such that each atom $\phi_{\gamma}$ in $\mathcal{D}$ is derived from an analytic mother function $\phi$  by a geometric transformation specified by the parameter vector $\gamma$. Here $\psi$ is a rotation parameter, $\tau_x$ and $ \tau_y$ denote translations in $x$ and $y$ directions, and $\sigma_x$ and $ \sigma_y$ represent an anisotropic scaling in $x$ and $y$ directions. The dictionary is defined over the continuous parameter domain $\Gamma$, and an atom $\phi_{\gamma}$ is given by
\begin{equation}
\phi_{\gamma}(X)=\phi(\sigma^{-1} \, \Psi^{-1} \, (X-\tau)),
\end{equation}
where
\begin{equation}
\sigma= \left[
\begin{array}{c c}
 \sigma_x & 0  \\
 0 & \sigma_y
\end{array} \right], \, \, \,
\Psi= \left[
\begin{array}{c c}
 \cos(\psi) & -\sin(\psi)  \\
\sin(\psi) & \cos(\psi)
\end{array} \right], \, \, \,
\tau= \left[
\begin{array}{c}
\tau_x \\
\tau_y
\end{array} \right]
\end{equation}
denote respectively the scale change, rotation and translation matrices defining the atom  $\phi_{\gamma}$. We may consider that the parameter domain $\Gamma$ is defined over the range of parameters $\psi \in [0,2\pi)$, $\tau_x, \tau_y \in \mathbb{R}$, and $\sigma_x, \sigma_y \in \mathbb{R}^+$. It is shown in \cite{Antoine2004} (in the proof of Proposition 2.1.2) that the linear span of a dictionary $\mathcal{D}$ generated with respect to the transformation model in (\ref{eq:dictManifoldExp}) is dense in $L^2(\mathbb{R}^2)$ if the mother function $\phi$ has nontrivial support; i.e., unless $\phi(X)=0$ almost everywhere. 

In our analysis, we select the generating mother function as the Gaussian function $ \phi(X) = e^{-X^T X}$. The Gaussian function has good time-localization properties, it is easy to treat in derivations due to its well-studied properties, and it ensures that $Span(\mathcal{D})$ is dense in $L^2(\mathbb{R}^2)$. Therefore, any pattern $p \in L^2(\mathbb{R}^2)$ can be represented  as the linear combination of a sequence of atoms in $\mathcal{D}$. In the rest of our analysis, we adopt a representation of $p$ in $\mathcal{D}$
\begin{equation}
p(X) = \sumkinf \coefk \, \atomk
\label{eq:pXgaussAtomsInf}
\end{equation}
where $\gamma_k$ are the atom parameters and $\coefk$ are the atom coefficients. Our derivation of the variations of $ \| \normgrad \hatp  \| $ and  $\| N_\hess \hatp \|$ is based on this representation and we use some properties of Gaussian atoms in our analysis.  Nevertheless, the conclusions of our analysis are general and valid for all reference patterns in $L^2(\mathbb{R}^2)$ since any square-integrable pattern can be represented in the Gaussian dictionary $\mathcal{D}$.

Now, applying the Gaussian filter in (\ref{eq:Gausskerdefn}) on the reference pattern in (\ref{eq:pXgaussAtomsInf}), we obtain the filtered pattern as 
\begin{equation*}
 \frac{1}{\pi \rho^2} \,  (\phi_{\rho} * p) (X) =  \frac{1}{\pi \rho^2} \sumkinf \coefk \, (\phi_{\rho}*   \phi_{\gamma_k}) (X) 
\end{equation*}
from the linearity of the convolution operator. In order to evaluate the convolution of two Gaussian atoms, we use the following proposition \cite{WandJones1995}. 

%%% PROP: CONVOLUTION OF GAUSSIANS
\begin{proposition}
Let $\phi(X)=e^{-X^T X}$ be the Gaussian function, and the Gaussian atoms $\phi_{\gamma_1}$ and $\phi_{\gamma_2}$ be given by $\phi_{\gamma_1}(X)=\phi(\sigma_1^{-1} \, \Psi_1^{-1} \, (X-\tau_1))$ and $\phi_{\gamma_2}(X)=\phi(\sigma_2^{-1} \, \Psi_2^{-1} \, (X-\tau_2))$. Then 
\begin{equation}
(\phi_{\gamma_1} * \phi_{\gamma_2})(X) = \frac{\pi | \sigma_1 \sigma_2 |}{| \sigma_3 |}  \phi_{\gamma_3}(X)
\end{equation}
where
\[ \phi_{\gamma_3}(X) = \phi(\sigma_3^{-1} \, \Psi_3^{-1} \, (X-\tau_3)) \]
and the parameters of $\phi_{\gamma_3}$ are given by
\[
 \tau_3= \tau_1 + \tau_2, \, \, \, \, \, \, \, \, \, \, 
 \Psi_3 \, \sigma_3^2 \, \Psi_3^{-1} = \Psi_1 \, \sigma_1^2 \, \Psi_1^{-1} + \Psi_2 \, \sigma_2^2 \, \Psi_2^{-1}. 
 \]
 \label{prop:convGauss}
\end{proposition}
%%%%%

Proposition \ref{prop:convGauss} implies that, when an atom $\phi_{\gamma_k}$ of $p$ is convolved with the Gaussian kernel, it becomes
\begin{equation}
 \frac{1}{\pi \rho^2} \,  (\phi_{\rho} * \phi_{\gamma_k})(X) = \frac{| \sigma_k |}{ | \hat{\sigma}_k | } \phi_{\hat{\gamma}_k} (X) 
\end{equation}
where $ \phi_{\hat{\gamma}_k} (X) =  \phi(\hat {\sigma}_k^{-1} \, \hat{\Psi}_k^{-1} \, (X-\hat{\tau}_k))  $,
\begin{equation}
\hat{\tau}_k = \tau_k, \, \, \, \, \, \,  \hat{\Psi}_k = \Psi_k,  \, \, \, \, \, \, \hat{\sigma}_k  =\sqrt{ \filtsc^2  + \sigma_k^2 } 
\label{eq:defn_filtered_atompar}
\end{equation}
and $\filtsc$ is the scale matrix of the Gaussian filter kernel defined in (\ref{eq:defnfiltsc}). Hence, when $p$ is smoothed with a Gaussian filter, the atom $\atomk$ with coefficient $\coefk$ is replaced by the smoothed atom $ \phi_{\hat{\gamma}_k}(X)$ with coefficient
\begin{equation}
\hatcoefk = \frac{| \sigma_k |}{ | \hat{ \sigma}_k | } \coefk
 = \frac{| \sigma_k |}{\sqrt{ | \filtsc^2 + \sigma_k^2 |} }   \coefk
 = \frac{\sigma_{x,k} \, \sigma_{y,k} }{\sqrt{ (\rho^2+\sigma_{x,k}^2)(\rho^2+\sigma_{y,k}^2 ) }} \coefk
 \label{eq:defn_hatlambdak}
\end{equation}
where $\sigma_k=\diag(\sigma_{x,k} ,\, \sigma_{y,k})$. This shows that the change in the pattern parameters due to filtering can be captured by substituting the scale parameters $\sigma_k$ with $\hat{\sigma}_k$ and replacing the coefficients $\coefk$ with $\hatcoefk$. Then, the smoothed pattern $\hat{p}$ has the following representation in the dictionary $\mathcal{D}$ 
\begin{equation}
\label{eq:form_phat}
\hat{p}(X)= \sumkinf \hatcoefk \, \phi_{\hat{\gamma}_k}(X) .\\
\end{equation}

One can observe from (\ref{eq:defn_hatlambdak}) that the atom coefficients $\hatcoefk$ of the filtered pattern $\hatp$ change with the filter size $\rho$ at a rate
\begin{equation}
\label{eq:dep_coef}
 \hatcoefk = O((1+\rho^2)^{-1}).
\end{equation}
Also, from (\ref{eq:defn_filtered_atompar}), the atom scale parameters of $\hatp$ are given by 
\begin{equation}
\hatsigmaxk= \sqrt{ \sigma_{x,k}^2 + \rho^2 }, 
\qquad \qquad
\hatsigmayk= \sqrt{ \sigma_{y,k}^2 + \rho^2 }
\end{equation}
which have the rate of increase
\begin{equation}
\label{eq:dep_atomscales}
\hatsigmaxk, \ \hatsigmayk = O((1+\rho^2)^{1/2})
\end{equation}
with the filter size $\rho$.

We are now equipped with the necessary tools for examining the variations of $\| \normgrad \hatp  \| $ and $\| \normhess \hatp  \| $ with the filter size $\rho$. We state these in the following lemma.

%%% LEMMA: VARIATION OF GRADIENT AND HESSIAN NORMS
\begin{lemma}
\label{lem:var_derivs}
The norms $\| \normgrad \hatp  \| $ and $\| \normhess \hatp  \| $ of the first and second-order variations of the pattern decrease with the filter size $\rho$ at the following rates 
\begin{equation*}
\begin{split}
\| \normgrad \hatp  \| &=  O((1+\rho^2)^{-1}) \\
\| \normhess \hatp  \| &= O((1+\rho^2)^{-3/2}) .
\end{split}
\end{equation*}
\end{lemma}

The proof of Lemma \ref{lem:var_derivs} is given in Appendix \ref{app:tan_dist:sec:pf_lem_var_derivs}. The above dependences are shown by deriving approximations of $\| \normgrad \hatp  \|$ and $\| \normhess \hatp  \|$ in terms of the atom parameters $\{\gamma_k \}$ and  coefficients $\{ \coefk \}$. Their variations with the filter size $\rho$ are then determined by building on the relations (\ref{eq:dep_atomscales}) and (\ref{eq:dep_coef}). The lemma not only confirms the intuition that the norms of the pattern gradient and Hessian  should decrease with filtering, but also provides expressions for their rate of decay with the filter size $\rho$.

An immediate consequence of Lemma \ref{lem:var_derivs} is the following.

\begin{corollary}
\label{cor:dep_manderivs}
The norms $\| \hatderiptran \| $, $\| \hatderijptran \| $ of the first and second-order manifold derivatives decrease with the filter size $\rho$ at the following rates
\begin{equation*}
\begin{split}
\| \hatderiptran \|  &=  O((1+\rho^2)^{-1}) \\
\| \hatderijptran \|  &= O\left((1+\rho^2)^{-3/2} + (1+\rho^2)^{-1}\right).
\end{split}
\end{equation*}
\end{corollary}

\begin{proof} The corollary follows directly from Lemma \ref{lem:var_derivs} and the relation between the manifold derivatives and the pattern derivatives given in (\ref{eq:dep_mander_patder}).
\end{proof}

Note that for large values of $\rho$, the second additive term of $ O(1+\rho^2)^{-1}$ in $\| \hatderijptran \| $ dominates the first term of $ O(1+\rho^2)^{-3/2}$, therefore $\| \hatderijptran \|  = O((1+\rho^2)^{-1})$ for large $\rho$. However, we keep both additive terms in $\| \hatderijptran \| $ as we will see that the first term is important for characterizing the behavior of the alignment error bound for small values of the filter size. Corollary \ref{cor:dep_manderivs} will be helpful for determining the dependences of the curvature bound $\hatMsecderUB$ and the parameters related to the metric tensor $\hatGij$ on the filter size. We will use it in our main result of Theorem \ref{thm:dep_alerrbnd}.

\subsection{Analysis of $\| \hatnoiseopt \|  $}

%We examine now the variation of $\| \hatnoiseopt \|  $ with both the filter size $\rho$ and the initial noise level $ \noiselev $ of the unfiltered target pattern. 

In the following lemma, we summarize the dependence of the noise level $\| \hatnoiseopt \|  $ in the filtered target pattern, on  the noise level $\noiselev$ in the original target pattern and the size $\rho$ of the smoothing filter.

\begin{lemma}
\label{lem:dep_noiselev}
The distance $\| \hatnoiseopt \|  $ between the filtered target pattern $\hattargetp$ and the transformation manifold $\M(\hatp)$ of the filtered reference pattern $\hatp$ has a rate of variation of
\begin{equation*}
\| \hatnoiseopt \| = O\left( (\noiselev + 1) (1+\rho^2)^{-1/2} \right)
\end{equation*}
with the filter size $\rho$ and the initial noise level $\noiselev$ for geometric transformation models that allow the change of the scale of the pattern $p$. The variation of  $\| \hatnoiseopt \|  $ is however given by
\begin{equation*}
\| \hatnoiseopt \| = O\left( \noiselev  (1+\rho^2)^{-1/2} \right)
\end{equation*}
if the geometric transformation model does not include a scale change.
\end{lemma}

The proof of Lemma \ref{lem:dep_noiselev} is given in Appendix \ref{app:tan_dist:sec:pf_lem_dep_noiselev}. The presented dependences are obtained by deriving a relation between the norm of the noise component $\hatnoiseopt = \hattargetp -  \hatptranopt$ and the filtered version $\filtnoiseopt$ of the initial noise component $\noiseopt = \targetp - \ptranopt $. The lemma states that $\| \hatnoiseopt \|$ decreases with the filter size $\rho$ at a rate of $O\left( (1+\rho^2)^{-1/2}\right)$. Meanwhile, its dependence on the initial noise level $\noiselev$ differs slightly between transformation models that include a scale change or not. The noise term $\| \hatnoiseopt \|$ increases at a rate of $O(\noiselev)$ for transformations without a scale change; however, transformations with a scale change introduce an offset to the initial noise level to yield a variation of $O(\noiselev+1)$. This is due to the following reason. The initial noise level before filtering is given by the norm of $\noiseopt = \targetp - \ptranopt $, where $\ptranopt \in \M(p)$. Meanwhile, when the transformation model $\tpar$ includes a scale change, the actions of filtering and transforming a pattern do not commute, and the filtered version $\filtptranopt$ of $\ptranopt$ does not lie on the transformation manifold $\M (\hatp)$ of the filtered reference pattern $\hatp$ (see Appendix \ref{app:tan_dist:sec:pf_lem_dep_noiselev} for more details). The ``lifting'' of the base point $\filtptranopt$ of $\hattargetp$ (with the decomposition $ \hattargetp =  \filtptranopt + \filtnoiseopt $) from the manifold $\M (\hatp)$ further increases the distance between $\hattargetp$ and $\M (\hatp)$, in addition to the deviation $\filtnoiseopt$. The overall noise level in case of filtering is therefore larger than the norm of the filtered version  $\filtnoiseopt $ of $\noiseopt$. Note that, for transformations involving a scale change, even if the initial noise level $\noiselev$ is zero, which means that $\targetp \in \M(p)$, we have $\hattargetp \notin \M(\hatp)$ after filtering. This creates a source of noise when the filtered versions of the image pair are used in the  alignment.

\subsection{Proof of Theorem \ref{thm:dep_alerrbnd}}

We are now ready to present a proof of the theorem.

\begin{proof} Remember from (\ref{eq:hatalerrbnd_app}) that the alignment error bound is given by
\begin{equation*}
\hatalerrbnd = \hatalerrbnd_1 +  \hatalerrbnd_2
\end{equation*}
where the error terms
\begin{equation}
\label{eq:hatE1_E2}
\begin{split}
 \hatalerrbnd_1  &= \half \,  \,  \hatMsecderUB \ \lambdamin^{-1} \ \big( [ \hatGij (\tparref) ] \big)  \sqrt{\tr( [ \hatGij (\tparref) ] )} \ \| \hattparopt - \tparref  \|_{1}^2 \\
 \hatalerrbnd_2 &= \sqrt{d} \ \hatMsecderUB \ \lambdamin^{-1} \ \big( [ \hatGij (\tparref) ] \big)  \ \| \hatnoiseopt \|  \   \|  \hattparopt - \tparref   \|_1
\end{split}
\end{equation}
are associated respectively with the nonzero manifold curvature (lifting of the manifold from the tangent space) and the noise on the target image. Also, remember that the variation of $\hatMsecderUB$ with $\rho$ is the same as that of $\| \hatderijptran \|$, and that $\lambdamin \ \big( [ \hatGij (\tparref) ] \big)$ and $\tr( [ \hatGij (\tparref) ] )$ have the same variation with $\rho$ as $\| \hatderiptranref \|^2$. Hence, using Corollary \ref{cor:dep_manderivs}, we obtain
\begin{eqnarray}
\hatMsecderUB \ \lambdamin^{-1} \ \big( [ \hatGij (\tparref) ] \big)  &=& O\left(1+ \,(1+\rho^2)^{-1/2}\right) O(1+ \rho^2)
\label{eq:dep_K_lmininv}\\
 \sqrt{\tr( [ \hatGij (\tparref) ] )}  &=& O\left((1+\rho^2)^{-1}\right)
\label{eq:dep_sqrt_tracemt}
\end{eqnarray}
which gives
\begin{equation*}
\hatalerrbnd_1  = O\left(1+ \,(1+\rho^2)^{-1/2}\right).
\end{equation*}
Then, from Lemma \ref{lem:dep_noiselev} and Equation (\ref{eq:dep_K_lmininv}), we determine the variation of $ \hatalerrbnd_2 $ as
\begin{equation*}
\hatalerrbnd_2  = O\left( (\noiselev+1) \, (1+\rho^2)^{1/2} \right) O\left( 1 +  (1+\rho^2)^{-1/2} \right) \approx O\left( (\noiselev+1) \, (1+\rho^2)^{1/2} \right) 
\end{equation*}
for transformations involving a scale change, and as
\begin{equation*}
\hatalerrbnd_2  = O\left( \noiselev \, (1+\rho^2)^{1/2} \right) O\left( 1 +  (1+\rho^2)^{-1/2} \right) \approx O\left( \noiselev \, (1+\rho^2)^{1/2} \right)
\end{equation*}
for transformations without a scale change, which finishes the proof of the theorem.
\end{proof}

%%% APPENDIX: DERIVATION OF MANIFOLD DERIVATIVES 
\section{Proof of the results used in Appendix \ref{app:pf_thm_dep_alerrbnd}}

\subsection{Derivations of $\| \hatderiptran \|$ and $\| \hatderijptran \|$ in terms of pattern spatial derivatives}
\label{app:tan_dist:sec:rel_manderiv_patderiv}

As the pattern $\hatp$ and its derivatives are square-integrable, there exists a bounded support $\Omega \in \Rsq$ such that the intensities of $\hatp$ and its derivatives are significantly reduced outside $\Omega$; i.e., \footnote{As filtering leads to a spatial diffusion in the intensity functions of the pattern and its derivatives, the size of the support $\Omega$ in fact depends on the filter size $\rho$. However, for the sake of simplicity of analysis, we ignore the dependence of $\Omega$ on $\rho$ and assume a single and sufficiently large support region $\Omega$, which can be selected with respect to the largest value of the filter size used in a hierarchical registration application.} 
\begin{equation*}
\hatp(X), \ \hatderxp(X), \ \hatderyp(X), \ \hatderxxp(X), \ \hatderxyp(X), \ \hatderyyp(X) \approx 0
\end{equation*}
for $X \notin \Omega$.
%that captures a substantial part of the energies of $\hatp$ and its derivatives, such that
%
%\begin{equation*}
%\| \hatp \|^2 = \int_{\Rsq} \hatp^2(X) dX \approx \int_{\Omega} \hatp^2(X) dX
%\end{equation*}
%%
Since the coordinate change function $\coordch$ is $C^2$-smooth, the derivatives of the transformed coordinates are bounded over $\Omega$. Hence, there exists a constant $\dercoordUB>0$ such that 
\begin{equation*}
| \derixp |,\ |\deriyp |, | \derijxp |,\ |\derijyp | \leq \dercoordUB 
\end{equation*}
for all $i, j = 1, \cdots, d$ and $X' \in \Omega$. 

Let us first clarify the notation used in the rest of our derivations. For a vector-valued function $g: \Rsq \rightarrow \mathbb{R}^n$, the notation $g$ denotes the function considered as an element of the function space it belongs to, while the notation $g(X)$ always stands for the value of $g$ evaluated at $X$; i.e., a vector in $\mathbb{R}^n$.

We begin with the term $\| \hatderiptran \|$. For all $X$, we have
\begin{equation*}
| \hatderiptran(X) | = |  \nabla \hatp(X')^T \deriXp  | \leq \| \nabla \hatp(X') \| \| \deriXp \| 
\end{equation*}
where $\deriXp = [\derixp \ \deriyp]^T$. Then, for $X \in \coordch_{\tpar}^{-1}(\Omega) $, $| \hatderiptran(X) | $ can be upper bounded as
\begin{equation*}
| \hatderiptran(X) | \leq \sqrt{2} \dercoordUB \  \| \nabla \hatp(X') \| .
\end{equation*}
We thus get
\begin{equation*}
\begin{split}
\| \hatderiptran \|^2 &= \int_{\Rsq} | \hatderiptran(X) |^2 dX = \int_{\Rsq} |  \nabla \hatp(X')^T \deriXp  |^2 dX \\
&\approx \int_{\coordch_{\tpar}^{-1}(\Omega) }  |  \nabla \hatp(X')^T \deriXp  |^2 dX
\leq  2 \dercoordUB^2   \int_{\coordch_{\tpar}^{-1}(\Omega) }     \| \nabla \hatp(X') \|^2  dX \\
&=  2 \dercoordUB^2  \int_{\Omega } \|  \nabla \hatp(X) \|^2 \ | \det (D\coordch_{\tpar}^{-1})(X) |  \ dX  
\end{split}
\end{equation*}
where $ \det (D\coordch_{\tpar}^{-1})(X)  $ is the Jacobian of the coordinate change function $\coordch_{\tpar}^{-1}$. In the above equations, when approximating the integration on $\Rsq$ with the integration on $\coordch_{\tpar}^{-1}(\Omega)$, we implicitly assume that $ \nabla \hatp(X')^T \deriXp   \approx 0$ outside the inverse image of the support region $\Omega$.   Such an assumption is reasonable as the transformed coordinates $X'$ are typically polynomial functions of the original coordinates $X$ and their rate of increase with $X$ is therefore dominated by the decay of the image intensity function with $X$ in a typical representation in $L^2(\Rsq)$ such as the Gaussian dictionary we use in this work, which is introduced in Section \ref{ch:tan_dist:ssec:dep_align_bnd}. Since the function $\coordch_{\tpar}$ is a smooth bijection on $\Rsq$, the Jacobian $ \det (D\coordch_{\tpar}^{-1})(X) $ is bounded on the bounded region $\Omega$. Therefore, there exists a constant $\jacobUB>0$ such that $| \det (D\coordch_{\tpar}^{-1})(X) | \leq \jacobUB $ for $X \in \Omega$. Hence, we obtain
\begin{equation*}
\| \hatderiptran \| \leq  \, \sqrt{ 2 \dercoordUB^2 \jacobUB}  \, \left( \int_{\Rsq} \|   \nabla \hatp(X) \|^2 dX \right)^{1/2}
  	=  \sqrt{ 2 \dercoordUB^2 \jacobUB}  \  \|  \normgrad \hatp \|
\end{equation*}
which shows that $\| \hatderiptran \|$ and $ \|  \normgrad \hatp \|$ have approximately the same rate of change with the filter size $\rho$; i.e.,
\[\| \hatderiptran \| = O(\|  \normgrad \hatp \|)  .\]

Next, we look at the term $\| \hatderijptran \|$. From triangle inequality we have
\begin{equation*}
\| \hatderijptran \| \leq \| v \| + \| w \|
\end{equation*}
where
\begin{equation*}
\begin{split}
v(X) &= \hatderxxp(X') \, \derixp \derjxp + \hatderxyp(X') \, (\derixp \derjyp + \derjxp \deriyp) + \hatderyyp(X') \, \deriyp \derjyp \\
w(X) &=  \hatderxp(X') \, \derijxp + \hatderyp(X') \, \derijyp.
\end{split}
\end{equation*}
Since $w$ is in the same form as $\hatderiptran$, one can upper bound it in the same way.
\begin{equation}
\label{eq:bndw}
\| w \| \leq  \,   \sqrt{ 2 \dercoordUB^2 \jacobUB}  \ \| \normgrad \hatp \|.
\end{equation}
We now examine the term $\| v \|$. Defining the derivative product vector
\begin{equation*}
B(X')=[ \derixp \derjxp \ \  \derixp \derjyp  \ \ \derjxp \deriyp \ \ \deriyp \derjyp ]^{T}, 
\end{equation*}
we have
\begin{equation*}
| v(X) | = |  (\hess \hatp) (X')^{T} \ B(X')   | \leq \| (\hess \hatp) (X') \|  \, \| B(X') \| .	 
\end{equation*}
At $X \in \coordch_{\tpar}^{-1}(\Omega) $, the upper bound $ \|   B(X')  \| \leq 2 \dercoordUB^2$ yields
\begin{equation*}
| v(X) |  \leq  2 \dercoordUB^2 \ \| (\hess \hatp) (X') \|.
\end{equation*}
Hence,
\begin{equation*}
\begin{split}
\| v \|^2 &= \int_{\Rsq} | v(X) |^2 dX = \int_{\Rsq} |   (\hess \hatp) (X')^{T} \ B(X')   |^2 dX \\
&\approx \int_{\coordch_{\tpar}^{-1}(\Omega) }  |   (\hess \hatp) (X')^{T} \ B(X')   |^2 dX
\leq  4 \dercoordUB^4  \int_{\coordch_{\tpar}^{-1}(\Omega) }     \|  (\hess \hatp) (X') \|^2  dX \\
&=  4 \dercoordUB^4  \int_{\Omega }    \|  (\hess \hatp) (X) \|^2   \ | \det (D\coordch_{\tpar}^{-1})(X) |  \ dX  
\ \ \leq \ \  4 \dercoordUB^4 \jacobUB   \int_{\Omega }   \|  (\hess \hatp) (X) \|^2   dX
\end{split}
\end{equation*}
and therefore
\begin{equation}
\label{eq:bndv}
\| v \| \leq 2 \dercoordUB^2 \sqrt{\jacobUB}  \left(  \int_{\Rsq}  \|  (\hess \hatp) (X) \|^2   dX \right)^{1/2} 
= 2 \dercoordUB^2 \sqrt{\jacobUB} \ \| \normhess \hatp   \|.
\end{equation}
Finally, putting together (\ref{eq:bndw}) and (\ref{eq:bndv}), we obtain the following upper bound on $\| \hatderijptran \|$
\begin{equation*}
\| \hatderijptran \|  \  \leq \ 2 \dercoordUB^2 \sqrt{\jacobUB}  \ \| \normhess \hatp  \|
			+ \sqrt{ 2 \dercoordUB^2 \jacobUB}  \, \|  \normgrad \hatp \|
\end{equation*}
which gives
\[ \| \hatderijptran \| = O\big( \| \normgrad \hatp  \| + \| \normhess \hatp  \| \big) .\]

%%% APPENDIX: DEPENDENCE OF HESSIAN AND GRADIENT NORMS

\subsection{Proof of Lemma \ref{lem:var_derivs}}
\label{app:tan_dist:sec:pf_lem_var_derivs}

Since the reference pattern consists of Gaussian atoms, the derivation of the norms of its gradient and Hessian involves the integration of products of Gaussian atom pairs. Therefore, in our analysis we make use of the following proposition, which gives the expression for the integration of the product of two Gaussian atoms \cite{WandJones1995}.

%%%% INTEGRATION OF GAUSSIAN PRODUCTS 
\begin{proposition}
Let $\phi_{\gamma_j}(X)=\phi(\sigma_j^{-1} \, \Psi_j^{-1} \, (X-\tau_j))$ and $\phi_{\gamma_k}(X)=\phi(\sigma_k^{-1} \, \Psi_k^{-1} \, (X-\tau_k))$. Then 
\begin{equation*}
\int_{\mathbb{R}^2} \phi_{\gamma_j}(X) \phi_{\gamma_k}(X) dX 
= \frac{ Q_{jk}}{2}
\end{equation*} 
where
\begin{equation}
\label{eq:Qjk_defn}
\begin{split}
Q_{jk}&:= \frac{ \pi \, | \sigma_j \sigma_k |}  {  \sqrt{ |\Sigma_{jk}| } }
\exp \left( - \frac{1}{2} (\tau_k - \tau_j)^{T} \, \Sigma_{jk}^{-1} \, (\tau_k - \tau_j) \right)\\
\Sigma_{jk}&:=\frac{1}{2} \left( \Psi_j \, \sigma_j^2 \, \Psi_j^{-1} + \Psi_k \, \sigma_k^2 \, \Psi_k^{-1}  \right).
\end{split}
\end{equation}
\label{prop:IntGaussProd}
\end{proposition}
%%%%%

We now prove Lemma \ref{lem:var_derivs}.

\begin{proof}

In order to determine the variations of $\| \normgrad \hatp  \|$ and $\| \normhess \hatp  \|$  with the filter size $\rho$, we first derive approximations for these terms in terms of the atom parameters of the reference pattern, which makes it easier to analyze them analytically. We then examine the dependence of these terms on $\rho$ with the help of their approximations.\\

\textit{Derivation of $\| \normgrad \hatp  \|$ }\\

We begin with the norm $\| \normgrad \hatp  \|$ of the gradient magnitude. In order to lighten the notation, we do the derivations for the unfiltered reference pattern $p$, which are directly generalizable for its filtered versions. We have
\begin{equation*}
\| \normgrad p  \|^2 = \int_{\Rsq} \|    \nabla p(X) \|^2 dX  
		= \int_{\Rsq} \left( \sumjinf \coefj (\nabla \atomj)^T \right)
		 \left( \sumkinf \coefk \nabla \atomk \right) dX.
\end{equation*}
It is easy to show that the gradient $\nabla \atomj$ of the atom $\atomj$ is given by 
\begin{equation*}
\nabla \atomj = -2 \, \atomj \, \Psi_j \, \sigma_j^{-2} \Psi_j^{-1} (X- \tau_j)
\end{equation*}
which yields 
\begin{equation*}
(\nabla \atomj)^T \nabla \atomk = 4 \, \atomj \atomk \, (X- \tau_j)^T \Theta_j^T \Theta_k (X- \tau_k)
\end{equation*}
where $\Theta_j := \Psi_j \, \sigma_j^{-2} \Psi_j^{-1}$. Putting this in the expression of $\| \normgrad p  \|^2$, we obtain
\begin{equation}
\label{eq:NGp}
\| \normgrad p  \|^2 = 4 \sumjinf \sumkinf \coefj \coefk \, \Ljk
\end{equation}
where
\begin{equation}
\label{eq:Ljk}
\Ljk  = \int_{\Rsq} \atomj \atomk \, (X- \tau_j)^T \Theta_j^T \Theta_k (X- \tau_k) \ dX.
\end{equation}

The evaluation of the above integral would give the exact expression of $\Ljk$ in terms of the atom parameters of $p$, which would however have a quite complicated form. On the other hand, we are interested in determining the variation of  $\Ljk$ with filtering rather than obtaining its exact expression. Hence, in order to make the derivation simpler, we approximate the above expression for $\Ljk$ with another term $\LjkUB$, which is easier to evaluate analytically and provides an upper bound for $\Ljk$ at the same time. Let us denote the smaller and greater eigenvalues of $\Theta_j$ as
\begin{equation*}
\iota_j = \lambdamin(\Theta_j),
\qquad \qquad 
\vartheta_j = \lambdamax(\Theta_j).
\end{equation*}
From Cauchy-Schwarz inequality,
\begin{equation*}
| (X- \tau_j)^T \Theta_j^T \Theta_k (X- \tau_k) | \leq \| \Theta_j (X- \tau_j)  \| \,  \| \Theta_k (X- \tau_k)  \|  
	\leq \, \vartheta_j \|  X- \tau_j  \|  \, \vartheta_k \|  X- \tau_k   \|.
\end{equation*}
Using this in the expression of $\Ljk$, we get
\begin{equation*}
\begin{split}
\Ljk &\leq | \Ljk | \leq  \int_{\Rsq} \atomj \atomk \, \vartheta_j \vartheta_k \,  \|  X- \tau_j \| \|  X- \tau_k   \| \, dX \\
&\leq \LjkUB :=  \vartheta_j \vartheta_k \sqrt{\LjUB} \sqrt{\LkUB} 
\end{split}
\end{equation*}
where
\begin{equation*}
\LjUB = \int_{\Rsq} \atomjsq \, \|  X- \tau_j  \|^2 dX.
\end{equation*}
Evaluating the above integral, we obtain
\begin{equation*}
\LjUB = \frac{\pi}{8} \, | \sigma_j | \, (\sigmaxj^2 + \sigmayj^2).
\end{equation*}
This gives the following upper bound for $\Ljk$
\begin{equation}
\label{eq:LjkUB}
\LjkUB =  \frac{\pi}{8} \vartheta_j \vartheta_k \left(   | \sigma_j \sigma_k | \, (\sigmaxj^2 + \sigmayj^2)(\sigmaxk^2 + \sigmayk^2) \right)^{1/2}.
\end{equation}
%
%Now, let $J^{-}$ and $J^{+}$ denote the set of $(j,k)$ indices with negative and positive coefficient products
%%
%\begin{equation}
%\begin{split}
%J^{-} &= \{ (j,k): \lambda_j  \lambda_k <0 \} \\
%J^{+}&= \{ (j,k): \lambda_j  \lambda_k >0 \}. 
%\end{split}
%\label{eq:JplusJminus}
%\end{equation}
%%
%Then, from (\ref{eq:NGp}), we can upper bound $\| \normgrad p  \|^2$ as follows.
%%
%\begin{equation}
%\label{eq:NGp}
%\| \normgrad p  \|^2 = 4 \sumj \sumk \coefj \coefk \, \Ljk
%\end{equation}

Now, generalizing (\ref{eq:NGp}) to filtered versions of the reference pattern, we  have
\begin{equation}
\label{eq:NGhatp}
\| \normgrad \hatp  \|^2 = 4 \sumjinf \sumkinf \hatcoefj \hatcoefk \, \hatLjk.
\end{equation}
We now determine the dependence of $\| \normgrad \hatp  \|$ on the filter size $\rho$. First,  from (\ref{eq:dep_coef}), the coefficient products have the variation 
\begin{equation}
\label{eq:coefjcoefk}
\hatcoefj  \hatcoefk = O((1+\rho^2)^{-2}) 
\end{equation}
with the filter size. Next, we look at the term $\hatLjk$. Note that the low-pass filter applied on the pattern $p$ increases the atom scale parameters $\sigmaxj$, $\sigmayj$ and therefore decreases the eigenvalues of the matrices $\Theta_j$, $\Theta_k$ in the exact expression for $\Ljk$ in (\ref{eq:Ljk}). Filtering also influences the terms $\atomj$ and $\atomk$ in (\ref{eq:Ljk}). The variations of these terms with $\rho$ are captured in the approximation $\LjkUB$ through the terms $\vartheta_j$, $\vartheta_k$, $\LjUB$, and $\LkUB$. Therefore, $\Ljk$ and $\LjkUB$ have the same rate of change with the filter size $\rho$. From (\ref{eq:LjkUB}), the approximation  $\hatLjkUB$ of $\hatLjk$ is given by
\begin{equation}
\label{eq:hatLjkUB}
\hatLjkUB= \frac{\pi}{8} \hat{\vartheta}_j \hat{\vartheta}_k \left(   | \hatsigma_j \hatsigma_k | \, (\hatsigmaxj^2 + \hatsigmayj^2)(\hatsigmaxk^2 + \hatsigmayk^2) \right)^{1/2}
\end{equation}
which is simply obtained by replacing the parameters $\sigma_j$ and $\vartheta_j$ with their filtered versions $\hatsigma_j$ and $\hat{\vartheta}_j$. From (\ref{eq:dep_atomscales}), we have
\begin{equation}
\label{eq:dep_sigmas}
\begin{split}
\hatsigmaxj, \, \hatsigmayj &= O( (1+\rho^2)^{1/2}) \\
| \hatsigma_j \hatsigma_k | &= O( (1+\rho^2)^{2}) \\
\hatvartheta_j &= \max( \hatsigmaxj^{-2} \, , \, \hatsigmayj^{-2} ) = O( (1+\rho^2)^{-1}).
\end{split}
\end{equation}
Putting these relations together in (\ref{eq:hatLjkUB}), we obtain
\begin{equation*}
\hatLjkUB=O(1)
\end{equation*}
with respect to $\rho$. Combining this with the rate of change of the coefficient product $\hatcoefj \hatcoefk$ in (\ref{eq:coefjcoefk}) yields $\hatcoefj \hatcoefk \hatLjk  = O((1+\rho^2)^{-2})$. Since each one the additive terms in the expression of $\| \normgrad \hatp  \|^2$ in (\ref{eq:NGhatp}) has the same rate of decrease with $\rho$, the infinite sum also decreases with $\rho$ at the same rate. Therefore, we get $\| \normgrad \hatp  \|^2 =  O((1+\rho^2)^{-2}) $, which gives
\begin{equation*}
\| \normgrad \hatp  \| =  O((1+\rho^2)^{-1}).
\end{equation*}

\textit{Derivation of $\| \normhess \hatp  \|$ }\\

We now continue with the norm $\| \normhess \hatp  \|$ of the Hessian magnitude. From (\ref{eq:Hess_defn}),
\begin{equation*}
\left( \normhess p \, (X) \right)^2= \| (\hess p) (X) \|^2 
	= ( \derxxp(X) )^2 +  2( \derxyp(X) )^2 +  ( \deryyp(X) )^2.
\end{equation*}
Hence,
\begin{equation*}
\| \normhess p  \|^2 = \intRsq \left( \normhess p \, (X) \right)^2 dX
 			= \intRsq ( \derxxp(X) )^2 +  2( \derxyp(X) )^2 +  ( \deryyp(X) )^2 dX.
\end{equation*}
The second derivatives of the pattern are of the form
\begin{equation*}
\derxxp(X) = \sumkinf \coefk \frac{\partial^2 \atomk} {\partial x^2}
\end{equation*}
and $\derxyp(X) $, $\deryyp(X) $ are obtained similarly. Then, $\| \normhess p  \|^2 $ is given by
\begin{equation*}
\begin{split}
\| \normhess p  \|^2  &= \sumjinf \sumkinf \coefj \coefk \intRsq \left(
			\frac{\partial^2 \atomj} {\partial x^2} \frac{\partial^2 \atomk} {\partial x^2}
		 +  2  \frac{\partial^2 \atomj} {\partial x \partial y} \frac{\partial^2 \atomk} {\partial x \partial y}	
		 + \frac{\partial^2 \atomj} {\partial y^2} \frac{\partial^2 \atomk} {\partial y^2}
\right) dX \\
			&= \sumjinf \sumkinf \coefj \coefk  \intRsq  \tr\left(\Hess(\atomj)  \Hess(\atomk) \right) dX
\end{split}
\end{equation*}
where 
\begin{equation*}
\Hess(\atomj) = \left[
\begin{array}{c c}
 \frac{\partial^2 \atomj} {\partial x^2} &  \frac{\partial^2 \atomj} {\partial x \partial y}  \\
  \frac{\partial^2 \atomj} {\partial x \partial y} & \frac{\partial^2 \atomj} {\partial y^2}
\end{array} \right]
\end{equation*}
denotes the Hessian matrix of $\atomj$. It is easy to show that 
\begin{equation*}
\begin{split}
\Hess(\atomj) &= -2 \Theta_j (X-\tau_j) \, \nabla^T \atomj - 2 \atomj \Theta_j \\
		      &= \atomj \big( 4 \Theta_j (X-\tau_j)  (X-\tau_j)^T \Theta_j - 2  \Theta_j \big) 
\end{split}
\end{equation*}
which yields
\begin{equation*}
\begin{split}
\tr\left(\Hess(\atomj)  \Hess(\atomk) \right) &= \atomj \atomk \bigg[ 16\, \tr\big(  \Theta_j (X-\tau_j)  (X-\tau_j)^T \Theta_j    \Theta_k (X-\tau_k)  (X-\tau_k)^T \Theta_k\big) \\
          & - 8 \, \tr \big( \Theta_j (X-\tau_j)  (X-\tau_j)^T \Theta_j  \Theta_k \big)
          - 8 \, \tr \big( \Theta_j \Theta_k (X-\tau_k)  (X-\tau_k)^T \Theta_k \big)\\
         & + 4 \, \tr \big( \Theta_j \Theta_k \big)
\bigg] .
\end{split}
\end{equation*}
The squared norm of the Hessian magnitude can then be written as
\begin{equation}
\label{eq:Nhp}
\| \normhess p  \|^2  =  \sumjinf \sumkinf \coefj \coefk \, (16 \Mjk - 8 \Njk - 8 \Nkj + 4 \Pjk)
\end{equation}
where
\begin{equation}
\label{eq:MNPjk}
\begin{split}
\Mjk &= \intRsq   \atomj \atomk \,  \tr \big(  \Theta_j (X-\tau_j)  (X-\tau_j)^T \Theta_j    \Theta_k (X-\tau_k)  (X-\tau_k)^T \Theta_k \big) \, dX \\
\Njk &= \intRsq   \atomj \atomk \, \tr \big( \Theta_j (X-\tau_j)  (X-\tau_j)^T \Theta_j  \Theta_k \big)  \, dX \\
\Pjk  &= \intRsq   \atomj \atomk \, \tr \big( \Theta_j \Theta_k \big)  \, dX.
\end{split}
\end{equation}

We now derive approximations $\MjkUB$, $\NjkUB$, $\PjkUB$ for the terms written above, which are easier to treat analytically and constitute upper bounds for these terms as well.

We begin with $\Mjk$. Denoting $A_j =  \Theta_j (X-\tau_j)  (X-\tau_j)^T \Theta_j  $,
\begin{equation}
\label{eq:Mjkbnd}
\Mjk \leq | \Mjk | \leq \intRsq  \atomj \atomk \, | \tr( A_j  A_k) | dX.
\end{equation}
Since $A_j$ is a rank-1 matrix,
\begin{equation*}
 | \tr( A_j  A_k) | = |  \lambdamax(A_j  A_k) | \leq \| A_j  A_k \| \leq  \| A_j \| \, \| A_k \|
\end{equation*}
where $\| \cdot \|$ denotes the operator norm for matrices. The first inequality above follows from the fact that the spectral radius of a matrix is smaller than its operator norm, and the second inequality comes from the submultiplicative property of the operator norm. From the inequality
\begin{equation*}
\| A_j \| = \|  \Theta_j (X-\tau_j)  (X-\tau_j)^T \Theta_j  \| \leq \vartheta_j^2 \, \| X-\tau_j \|^2
\end{equation*}
we get
\begin{equation*}
 | \tr( A_j  A_k) | \leq  \vartheta_j^2  \vartheta_k^2 \, \| X-\tau_j \|^2  \| X-\tau_k \|^2.
\end{equation*}
Using this bound in (\ref{eq:Mjkbnd}) yields 
\begin{equation*}
\Mjk \leq  \vartheta_j^2  \vartheta_k^2  \intRsq  \atomj \atomk \| X-\tau_j \|^2  \| X-\tau_k \|^2 dX
\end{equation*}
which gives the upper bound 
\begin{equation*}
\Mjk \leq \MjkUB :=  \vartheta_j^2  \vartheta_k^2  \, \sqrt{\MjUB}  \, \sqrt{\MkUB}
\end{equation*}
where
\begin{equation*}
\MjUB = \intRsq  \atomjsq \| X-\tau_j \|^4 dX.
\end{equation*}
Evaluating the above integral, we get
\begin{equation*}
\MjUB = \pi | \sigma_j | \left(  \frac{3}{32} \sigmaxj^4  +  \frac{1}{16} \sigmaxj^2 \sigmayj^2 + \frac{3}{32} \sigmayj^4  \right).
\end{equation*}
This finishes the derivation of $\MjkUB$.

Next, we look at the term $\Njk$. Performing similar steps as in $\Mjk$, we obtain
\begin{equation*}
| \tr(A_j \Theta_k) | \leq \vartheta_j^2 \vartheta_k  \| X-\tau_j \|^2.
\end{equation*}
This gives $\Njk \leq  \vartheta_j^2 \vartheta_k \sqrt{\MjUB} \, \| \phi_{\gamma_k} \|$. The norm $\| \phi_{\gamma_k} \|$ of the atom $\phi_{\gamma_k}$ is 
\begin{equation*}
\| \phi_{\gamma_k} \| = \sqrt{\frac{\pi | \sigma_k |}{2}}.
\end{equation*}
Hence, the term $\Njk$ is upper bounded as
\begin{equation*}
\Njk \leq \NjkUB :=  \sqrt{\frac{\pi | \sigma_k |}{2}} \, \vartheta_j^2 \vartheta_k \sqrt{\MjUB}. 
\end{equation*}

Lastly, we derive a bound for the term $\Pjk$. The magnitude of the trace of $\Theta_j \Theta_k$ can be bounded as
\begin{equation*}
| \tr(\Theta_j \Theta_k) |= |\lambdamin(\Theta_j \Theta_k) + \lambdamax(\Theta_j \Theta_k) | \leq 2 \, r(\Theta_j \Theta_k) \leq 2 \, \| \Theta_j \Theta_k \| \leq 2 \, \| \Theta_j \| \, \| \Theta_k \| = 2 \vartheta_j \vartheta_k 
\end{equation*}
where $r(\cdot)$ denotes the spectral radius of a matrix. The term $\Pjk$ can thus be bounded as
\begin{equation*}
\Pjk \leq 2  \vartheta_j \vartheta_k   \intRsq  \atomj \atomk dX.
\end{equation*}
From Proposition \ref{prop:IntGaussProd}, we get
\begin{equation*}
\Pjk \leq \PjkUB := \vartheta_j \vartheta_k  Q_{jk} 
\end{equation*}
where $Q_{jk} $ is as defined in (\ref{eq:Qjk_defn}).

Having thus derived approximations $\MjkUB$, $\NjkUB$, $\PjkUB$ for the terms $\Mjk$, $\Njk$, $\Pjk$ in (\ref{eq:Nhp}), we now have an analytical approximation of the norm $\| \normhess p  \|$ of the Hessian magnitude in terms of the atom parameters of the pattern. We now determine the order of variation of $\| \normhess p  \|$ with the filter size $\rho$ using this approximation. From (\ref{eq:Nhp}), we obtain the norm of the Hessian magnitude of the filtered pattern $\hatp$ as
\begin{equation}
\label{eq:normhesshatp_defn}
\| \normhess \hatp  \|^2  =  \sumjinf \sumkinf \hatcoefj \hatcoefk \, (16 \hatMjk - 8 \hatNjk - 8 \hatNkj + 4 \hatPjk).
\end{equation}
In the expressions of  $\Mjk$, $\Njk$, $\Pjk$ in (\ref{eq:MNPjk}), we see that filtering affects the terms $\Theta_j$ and the atoms $\atomj$. Comparing these terms with their approximations $\MjkUB$, $\NjkUB$, $\PjkUB$, we observe that the influence of smoothing on the matrices $\Theta_j$ is captured in the approximations via its influence on their eigenvalues $\vartheta_j$, while the influence of smoothing on the atoms is also preserved in the approximations as the atoms appear in the expressions of $\MjkUB$, $\NjkUB$, $\PjkUB$. Hence, the terms  $\hatMjk$, $\hatNjk$, $\hatPjk$  have the same rate of change with the filter size $\rho$ as their approximations $\hatMjkUB$, $\hatNjkUB$, $\hatPjkUB$. In the following, we determine the order of dependence of these terms on $\rho$.

We begin with  $\hatMjk$. The relations in (\ref{eq:dep_sigmas}) imply that
\begin{equation*}
\hatMjUB = \pi | \hatsigma_j | \left(  \frac{3}{32} \hatsigmaxj^4  +  \frac{1}{16} \hatsigmaxj^2 \hatsigmayj^2 + \frac{3}{32} \hatsigmayj^4  \right)
\end{equation*}
increases with $\rho$ at a rate of $O((1+\rho^2)^3)$ and the product $\hatvartheta_j^2 \hatvartheta_k^2$ decreases with $\rho$ at a rate of $O((1+\rho^2)^{-4})$. Therefore, the overall rate of variation of 
\begin{equation*}
\hatMjkUB =  \hatvartheta_j^2  \hatvartheta_k^2  \, \sqrt{\hatMjUB}  \, \sqrt{\hatMkUB}
\end{equation*}
with the filter size is given by 
\begin{equation}
\label{eq:dep_MjkUB}
\hatMjkUB = O((1+\rho^2)^{-1}).
\end{equation}
We similarly obtain the dependence of 
\begin{equation*}
\begin{split}
\hatNjkUB =  \sqrt{\frac{\pi | \hatsigma_k |}{2}} \, \hatvartheta_j^2 \hatvartheta_k \sqrt{\hatMjUB}
\end{split}
\end{equation*}
on the filter size as
\begin{equation}
\label{eq:dep_NjkUB}
\hatNjkUB =  O((1+\rho^2)^{-1}).
\end{equation}
Lastly, 
\begin{equation*}
\hatPjkUB = \hatvartheta_j \hatvartheta_k  \hat{Q}_{jk} 
\end{equation*}
where
\begin{equation*}
\begin{split}
\hat{Q}_{jk}&= \frac{ \pi \, | \hatsigma_j \hatsigma_k |}  {  \sqrt{ |\hatSigma_{jk}| } }
\exp \left( - \frac{1}{2} (\tau_k - \tau_j)^{T} \, \hatSigma_{jk}^{-1} \, (\tau_k - \tau_j) \right)\\
\hatSigma_{jk}&=\frac{1}{2} \left( \Psi_j \, \hatsigma_j^2 \, \Psi_j^{-1} + \Psi_k \, \hatsigma_k^2 \, \Psi_k^{-1}  \right).
\end{split}
\end{equation*}
One can determine the rate of change of $\hat{Q}_{jk}$ with $\rho$ as follows. First, since   the eigenvalues of the matrix $\hatSigma_{jk}$ increase with $\rho$, the term in the exponential approaches $0$ as $\rho$ increases. The variation of $\hat{Q}_{jk}$ is thus given by the variation of $ \pi \, | \hatsigma_j \hatsigma_k | /  \sqrt{ |\hatSigma_{jk}| }$. The term $\sqrt{ |\hatSigma_{jk}| }$ has the same rate of change with $\rho$ as $| \hatsigma_j |$; therefore, $\sqrt{ |\hatSigma_{jk}| } = O(1+\rho^2)$. This gives
\begin{equation}
\label{eq:dep_hatQ}
\hat{Q}_{jk}  = O(1+\rho^2)
\end{equation}
and 
\begin{equation}
\label{eq:dep_PjkUB}
\hatPjkUB  = O((1+\rho^2)^{-1}).
\end{equation}
Finally, combining the results (\ref{eq:dep_MjkUB}), (\ref{eq:dep_NjkUB}) and (\ref{eq:dep_PjkUB}) in (\ref{eq:normhesshatp_defn}), and remembering that the coefficient products vary with $\rho$ as $ \hatcoefj \hatcoefk = O((1+\rho^2)^{-2})$, we conclude that the norm $\| \normhess \hatp  \|$ of the Hessian magnitude decreases with the filter size $\rho$ at a rate of
\begin{equation*}
\| \normhess \hatp  \| = O((1+\rho^2)^{-3/2})
\end{equation*}
which finishes the proof of the lemma.

\end{proof}

%%%  APPENDIX: DEPENDENCE OF NOISE NORM

\subsection{Proof of Lemma \ref{lem:dep_noiselev}}

\label{app:tan_dist:sec:pf_lem_dep_noiselev}

\begin{proof} Remember from (\ref{eq:lambdaopt_defn}) and (\ref{eq:hatlambdaopt_defn}) that the projection of the unfiltered target pattern $\targetp$ onto $\M(p)$ is $\ptranopt$, and the projection of the filtered target pattern  $\hattargetp$ onto $\M(\hatp)$ is $\hatptranopt$. Since $\hatptranopt$ is the point on $\M(\hatp)$ that has the smallest distance to $\hattargetp$, we have the following for the distance $\| \hatnoiseopt \| $ between $\hattargetp$ and $\M(\hatp)$
\begin{equation*}
\| \hatnoiseopt \| = \|  \hattargetp -  \hatptranopt   \| \leq \| \hattargetp - \hatptranoptinit \|
\end{equation*}
where $\hatptranoptinit$ is the filtered pattern $\hatp$ transformed by the transformation vector $\tparopt$ that is optimal in the alignment of the unfiltered patterns. 

As discussed in Section \ref{ch:tan_dist:ssec:dep_align_bnd}, the deviation between the  transformations $\tparopt $ and $\hattparopt$ depends on the transformation model. Here we do not go into the investigation of the difference between $\hatptranopt$ and $\hatptranoptinit$, and content ourselves with the upper bound $\| \hattargetp - \hatptranoptinit \|$ for $\| \hatnoiseopt \| $ in order to keep our analysis generic and valid for arbitrary transformation models. Our purpose is then to determine how the distance $ \| \hattargetp - \hatptranoptinit \|$ depends on the initial noise level
\begin{equation*}
\noiselev = \| \noiseopt \| = \| \targetp  - \ptranopt \|
\end{equation*}
and the filter size $\rho$.  The noise pattern $\noiseopt$ becomes
\begin{equation*}
\filtnoiseopt = \hattargetp - \filtptranopt
\end{equation*}
when filtered by the filter kernel in (\ref{eq:Gausskerdefn}), where $\filtptranopt$ is the filtered version of $\ptranopt$ with the same kernel. Now, an important observation is that $\filtnoiseopt \neq  \hattargetp - \hatptranoptinit$ for geometric transformations that change the scale of the pattern, because
\begin{equation}
\label{eq:filt_trans_notcommute}
\filtptranopt \neq \hatptranoptinit
\end{equation}
i.e., the operations of filtering a pattern and applying it a geometric transformation  do not commute for such transformation models. The reason is that filtering modifies the scale matrices $\sigma_k$ of atoms, and when the geometric transformation involves a scale change, the commutativity of these two operations fails. For geometric transformations that do not involve a scale change, the equality $\filtptranopt = \hatptranoptinit$ holds. This is explained in more detail in the rest of this section. For the sake of generality, we base our derivation on the hypothesis (\ref{eq:filt_trans_notcommute}) and proceed by bounding the deviation of $\hattargetp - \hatptranoptinit$ from $\filtnoiseopt$. We thus use the following inequality for bounding $\| \hatnoiseopt \| $ 
\begin{equation}
\label{eq:bnd_hatnoiseopt}
\begin{split}
\| \hatnoiseopt \|  &\leq \| \hattargetp - \hatptranoptinit \| \leq \|  \hattargetp - \filtptranopt   \| + \| \filtptranopt - \hatptranoptinit \|  \\
	&= \| \filtnoiseopt \| + \| \filtptranopt - \hatptranoptinit \|.
\end{split}
\end{equation}
Hence, we achieve the examination of $\| \hatnoiseopt \| $ in two steps. We first determine the variation of $\| \filtnoiseopt \|$ with the initial noise level $\noiselev$ and the filter size $\rho$. Then, we study the second term $ \| \filtptranopt - \hatptranoptinit \|$ as a function of the filter size. We finally put together these two results in order to obtain the variation of the term $\| \hatnoiseopt \| $. \\

\textit{Derivation of $\| \filtnoiseopt \|$}\\

We begin with deriving an analytical expression for the norm $\noiselev$ of the noise pattern $\noiseopt$, whose variation with filtering is then easy to determine. Since the noise pattern $\noiseopt$ is in $L^2(\Rsq)$, and the linear span of the Gaussian dictionary $\mathcal{D}$ is dense in $L^2(\Rsq)$,  $\noiseopt$ can be represented as the linear combination of a sequence of atoms in $\mathcal{D}$
\begin{equation*}
n(X) = \sumkinf \varsigma_k \, \phi_{\chi_k}(X)
\end{equation*}
where $\varsigma_k$ are the atom coefficients and $\chi_k$ are the atom parameters. Then, 
\begin{equation*}
\noiselev^2 = \| \noiseopt \|^2 = \sumjinf \sumkinf  \varsigma_j \varsigma_k  
		\intRsq \phi_{\chi_j}(X) \phi_{\chi_k}(X) dX
		= \sumjinf \sumkinf  \varsigma_j \varsigma_k R_{jk}
\end{equation*}
where the term $R_{jk}$ is in the same form as the term $Q_{jk}$ given in (\ref{eq:Qjk_defn}) and obtained with the atom parameters of $\noiseopt$. Then, the squared norm of the filtered version of $\noiseopt$ is
\begin{equation*}
\| \filtnoiseopt \|^2 = \sumjinf \sumkinf  \hatvarsigma_j \hatvarsigma_k \hat{R}_{jk}.
\end{equation*}
Now, the coefficients $\hatvarsigma_j $ have the same variation with $\rho$ as $\hatcoefj$; therefore, from (\ref{eq:dep_coef}), we obtain
\begin{equation*}
\hatvarsigma_j \hatvarsigma_k = O((1+\rho^2)^{-2}).
\end{equation*}
Next, $\hat{R}_{jk}$ and $\hat{Q}_{jk}$ have the same variation with $\rho$ since they are of the same form. Thus, the relation in (\ref{eq:dep_hatQ}) implies that 
\begin{equation}
\hat{R}_{jk}  = O(1+\rho^2).
\end{equation}
Putting these results in the expression of $\| \filtnoiseopt \|^2 $, we see that the norm $\| \filtnoiseopt \|$ of the filtered noise pattern decreases with $\rho$ at a rate
\begin{equation*}
\| \filtnoiseopt \| = O((1+\rho^2)^{-1/2}).
\end{equation*}
Lastly, we look at the dependence of $\| \filtnoiseopt \| $ on the initial noise level $\noiselev=\| \noiseopt \|$. Since convolution with a filter kernel is a linear operator, the norm of the filtered noise pattern is linearly proportional to the norm of the initial noise pattern. Therefore,  $\| \filtnoiseopt \| $ varies linearly with $\noiselev$. Combining this with the above result, we  obtain the joint variation of  $\| \filtnoiseopt \| $ with $\noiselev$ and $\rho$ as
\begin{equation}
\label{eq:dep_filtnoiseopt}
\| \filtnoiseopt \| = O( \noiselev \, (1+\rho^2)^{-1/2}).
\end{equation}
\\

\textit{Derivation of  $ \| \filtptranopt - \hatptranoptinit \|$ }\\

In order to study the variation of the term $ \| \filtptranopt - \hatptranoptinit \|$ with the filter size in a convenient way, we assume that the composition of the geometric transformation $\tpar \in \tpardom$ generating the manifold $\M(p)$ and the geometric transformation $\gamma \in \Gamma$ generating the dictionary $\mathcal{D}$ can be represented as a transformation vector in $\Gamma$; i.e., for all $\tpar \in \tpardom$ and  $\gamma \in \Gamma$, there exists $\gamma \circ \tpar \in \Gamma$ such that 	
\begin{equation*}
A_{\tpar}(\phi_{\gamma})(X) = \phi_{\gamma  \circ \tpar }(X).
\end{equation*}
Note that this assumption holds for common geometric transformation models $\tpar$ such as translations, rotations, scale changes and their combinations.

In order to ease the notation, we derive the variation of $ \| \filtptran - \hatptraninit \|$ for an arbitrary transformation vector $\tpar$, which is also valid for the optimal transformation vector $\tparopt$. The transformed version $\ptran$ of $p$ can be represented as
\begin{equation*}
\ptran(X) = \sumkinf \coefk \, \atomktrans.
\end{equation*}
Let us denote the scale, rotation and translation matrices corresponding to the composite transformation vector $\gamma_k  \circ \tpar $ respectively as $\sigma_k \op \tpar$, $\Psi_k \op \tpar$, and $\tau_k \op \tpar$. Then the filtered version of the transformed pattern $\ptran$ is given by
\begin{equation*}
\filtptran (X)= \sumkinf \coefk \, \frac{| \sigma_k \op \tpar |}{| \widehat{\sigma_k \op \tpar}  |} \
		\phi_{\widehat{\gamma_k  \circ \tpar}}(X)
\end{equation*}
where $ \widehat{\sigma_k \op \tpar} = \sqrt{(\sigma_k \op \tpar)^2 + \filtsc^2 }$ is the scale matrix of the filtered atom parameters $\widehat{\gamma_k  \circ \tpar}$. The rotation and translation matrices $\Psi_k \op \tpar$ and $\tau_k \op \tpar$ do not change as filtering affects only the scale matrix.

Now we derive the expression of $\hatptraninit $, which is obtained by filtering $p$ first, and then applying it a geometric transformation. Remember from Section \ref{ch:tan_dist:ssec:dep_align_bnd} that the filtered pattern $\hatp$ is 
\begin{equation*}
\hatp (X) = \sumkinf \coefk \frac{| \sigma_k |}{| \hatsigma_k |} \, \phi_{\hatgamma_k}(X)
\end{equation*}
and the transformed version of $\hatp$ by $\tpar$ is
\begin{equation*}
\hatptraninit (X) = \sumkinf \coefk \frac{| \sigma_k |}{| \hatsigma_k |} \, \phi_{\hatgamma_k \circ \tpar}(X)
\end{equation*}
where the atom parameter vector $\hatgamma_k \circ \tpar$ has the scale matrix $\hatsigma_k \op \tpar = \sqrt{\sigma_k^2 + \filtsc^2} \op \tpar$, rotation matrix $\Psi_k \op \tpar$ and translation vector $\tau_k \op \tpar$. Comparing the expressions of $\filtptran $ and $\hatptraninit$, we see that these patterns have different atom scale matrices and atom coefficients if the transformation $\tpar$ involves a scale change. The atoms of $\filtptran$ and $\hatptraninit$ have the same rotation and translation matrices. Hence, if $\tpar$ does not modify the scale matrices of atoms, we have $\sigma_k \op \tpar = \sigma_k$; therefore, $\filtptran = \hatptraninit$.

The modification that the transformation $\tpar$ makes in the atom scale parameters can be represented with a scale change matrix
\begin{equation*}
S = \left[
\begin{array}{c c}
s_x & 0   \\
0  & s_y
\end{array} \right]
\end{equation*}
such that 
\begin{equation*}
\sigma_k \op \tpar = S \, \sigma_k.
\end{equation*}
Here we avoid writing the dependence of $S$ on $\tpar$ for notational convenience. We also represent the scale change of all atoms with the same matrix $S$ to ease the notation. However, this is not a strict hypothesis; i.e., since we treat the scale change parameters $s_x$ and $s_y$ as constants when examining the variation of $\| \filtptran -\hatptraninit \|$ with the filter size $\rho$, our result is generalizable to the case when different atoms have different scale change matrices $S_k$.

With this representation, the atom scale matrices of $\filtptran$ and $\hatptraninit$ are respectively obtained as
\begin{equation*}
\widehat{\sigma_k \op \tpar} = \sqrt{S^2 \sigma_k^2 + \filtsc^2}, 
\qquad \qquad
\hatsigma_k \op \tpar = S \sqrt{\sigma_k^2 + \filtsc^2} 
\end{equation*}
and the atom coefficients in these two patterns are respectively given by
\begin{equation*}
\coefk \, \frac{| \sigma_k \op \tpar |}{| \widehat{\sigma_k \op \tpar}  |}
= \coefk \, \frac{| S \sigma_k |}{| \sqrt{ S^2 \sigma_k^2 + \filtsc^2}  |} \ , 
\qquad \qquad
\coefk \frac{| \sigma_k |}{| \hatsigma_k |} = \coefk \frac{| \sigma_k |}{|  \sqrt{\sigma_k^2 + \filtsc^2}  |}.
\end{equation*}

The difference between the two patterns can then be upper bounded as
\begin{equation}
\label{eq:bnd_filttran_tranfilt}
\begin{split}
\| \filtptran -\hatptraninit \| &= \left\| \sumkinf 
	\coefk \, \frac{| S \sigma_k |}{| \sqrt{ S^2 \sigma_k^2 + \filtsc^2}  |} 
	\phi_{\widehat{\gamma_k  \circ \tpar}}
   - 	 \sumkinf 
   	\coefk  \,  \frac{| \sigma_k |}{|  \sqrt{\sigma_k^2 + \filtsc^2}  |}
	\phi_{\hatgamma_k \circ \tpar}  \right\| \\
 	& \leq \| e_1 \| + \| e_2 \|
\end{split}
\end{equation}
where
\begin{equation*}
\begin{split}
 e_1 & =  \sumkinf \coefk \, \frac{| S \sigma_k |}{| \sqrt{ S^2 \sigma_k^2 + \filtsc^2}  |} 
	(  \phi_{\widehat{\gamma_k  \circ \tpar}} - \phi_{\hatgamma_k \circ \tpar} ) \\
 e_2 &=	\sumkinf \coefk \, \left( \frac{| S \sigma_k |}{| \sqrt{ S^2 \sigma_k^2 + \filtsc^2}  |}  
 	- \frac{| \sigma_k |}{|  \sqrt{\sigma_k^2 + \filtsc^2}  |} \right)
	\phi_{\hatgamma_k \circ \tpar} .
\end{split}
\end{equation*}
In the following, we determine the rate of change of the terms $\| e_1 \|$ and $\| e_2 \|$ with the filter size $\rho$, which will then be used to estimate the dependence of $\| \filtptran -\hatptraninit \|$ using (\ref{eq:bnd_filttran_tranfilt}). We momentarily omit the atom index $k$ for lightening the notation. We begin with $\| e_1 \|$. Since $e_1$ is a linear combination of atom differences, its variation with $\rho$ is given by the product of the variations of the coefficients and the atom difference norms with $\rho$.
\begin{equation}
\label{eq:norm_e1}
\| e_1 \| = O\left(\coef \, \frac{| S \sigma |}{| \sqrt{ S^2 \sigma^2 + \filtsc^2}  |} \right)
		O\left(\| \phi_{\widehat{\gamma  \circ \tpar}} - \phi_{\hatgamma \circ \tpar}  \| \right).
\end{equation}
The coefficients decrease with $\rho$ at a rate
\begin{equation}
\label{eq:e1_dep_coef}
\coef \, \frac{| S \sigma |}{| \sqrt{ S^2 \sigma^2 + \filtsc^2}  |} = O((1+\rho^2)^{-1}).
\end{equation}
Next, we look at the dependence of the term $\| \phi_{\widehat{\gamma  \circ \tpar}} - \phi_{\hatgamma \circ \tpar}  \|$ on $\rho$.
\begin{equation*}
\begin{split}
\| \phi_{\widehat{\gamma  \circ \tpar}} - \phi_{\hatgamma \circ \tpar}  \|^2 
 = \intRsq & \bigg[   \phi\left((S^2 \sigma^2 + \filtsc^2)^{-1/2} (\Psi \op \tpar)^{-1} (X- \tau \op \tpar) \right) \\
&- \phi\left( (S^2 \sigma^2 + S^2\filtsc^2)^{-1/2} (\Psi \op \tpar)^{-1} (X- \tau \op \tpar) \right) 
\bigg]^2 dX
\end{split}
\end{equation*}
Defining $a_x := s_x^2  \sigma_x^2 + \rho^2$, $b_x := s_x^2 ( \sigma_x^2 + \rho^2)$, and defining $a_y$ and $b_y$ similarly, the evaluation of the above integral yields
\begin{equation*}
\| \phi_{\widehat{\gamma  \circ \tpar}} - \phi_{\hatgamma \circ \tpar}  \|^2 
	=\frac{\pi}{2}(\sqrt{a_x a_y} + \sqrt{b_x b_y} ) - 2\pi \sqrt{ \frac{a_x a_y b_x b_y}{(a_x + b_x)(a_y + b_y)} }.
\end{equation*}
As the parameters $a_x$, $b_x$, $a_y$, $b_y$ increase with $\rho$ at a rate of $O(1+\rho^2)$, the rate of increase of the squared norm of the atom difference $ \phi_{\widehat{\gamma  \circ \tpar}} - \phi_{\hatgamma \circ \tpar}$ with $\rho$ is given by
\begin{equation*}
\| \phi_{\widehat{\gamma  \circ \tpar}} - \phi_{\hatgamma \circ \tpar} \|^2= O(1+\rho^2).
\end{equation*}
Putting this result in (\ref{eq:norm_e1}) together with the decay rate of coefficients given in (\ref{eq:e1_dep_coef}) yields
\begin{equation}
\label{eq:dep_e1}
\| e_1 \| = O( (1+ \rho^2)^{-1/2}).
\end{equation}
Let us now examine the term $\| e_2 \|$. The rate of change of $\| e_2 \|$ can be estimated from the variation of the coefficients and the atom norms as follows 
\begin{equation*}
\| e_2 \| = O\left(\coef \, \left[ \frac{| S \sigma |}{| \sqrt{ S^2 \sigma^2 + \filtsc^2}  |}  
 	- \frac{| \sigma |}{|  \sqrt{\sigma^2 + \filtsc^2}  |} \right] \right)
	O(\| \phi_{\hatgamma \circ \tpar} \|).
\end{equation*}
The coefficients decay with $\rho$ at a rate
\begin{equation*}
\coef \, \left( \frac{| S \sigma |}{| \sqrt{ S^2 \sigma^2 + \filtsc^2}  |}  
 	- \frac{| \sigma |}{|  \sqrt{\sigma^2 + \filtsc^2}  |} \right)
	= O((1+\rho^2)^{-1}).
\end{equation*}
Next, the squared norm of the atom is calculated as
\begin{equation*}
\begin{split}
\| \phi_{\hatgamma \circ \tpar} \|^2 &= \intRsq  \phi^2 \left( (S^2 \sigma^2 + S^2\filtsc^2)^{-1/2} (\Psi \op \tpar)^{-1} (X- \tau \op \tpar) \right) dX \\
	&= \frac{\pi}{2} s_x s_y \sqrt{(\sigma_x^2 + \rho^2)(\sigma_y^2 + \rho^2)}
\end{split}
\end{equation*}
which shows that the atom norm increases with $\rho$ at a rate
\begin{equation*}
\| \phi_{\hatgamma \circ \tpar} \| = O( (1+\rho^2)^{1/2}).
\end{equation*}
Hence, we obtain the order of dependence of $\| e_2\|$ on $\rho$ as
\begin{equation}
\label{eq:dep_e2}
\| e_2 \|  = O( (1+\rho^2)^{-1/2}).
\end{equation}
Finally, from (\ref{eq:dep_e1}), (\ref{eq:dep_e2}), and the inequality in (\ref{eq:bnd_filttran_tranfilt}), we obtain the variation of the error term $\| \filtptran -\hatptraninit \| $ with $\rho$ as
\begin{equation}
\label{eq:dep_filttransdiff}
\| \filtptran -\hatptraninit \| = O( (1+\rho^2)^{-1/2}).
\end{equation}
\\

\textit{Variation of $\| \hatnoiseopt \|$ with noise level and filter size}\\

We can now put together the results obtained so far to determine the variation of the noise term $\| \hatnoiseopt \|$. Using the upper bound on  $\| \hatnoiseopt \|$ given in (\ref{eq:bnd_hatnoiseopt}) and the variations of $\| \filtnoiseopt \|$ and $\| \filtptranopt -  \hatptranoptinit \| $ given in (\ref{eq:dep_filtnoiseopt}) and (\ref{eq:dep_filttransdiff}), the joint variation of the noise term $\| \hatnoiseopt \|$ with the initial noise level $\noiselev$ and the filter size $\rho$ is obtained as
\begin{equation*}
\| \hatnoiseopt \| = O\left( (\noiselev + 1) (1+\rho^2)^{-1/2} \right)
\end{equation*}
for geometric transformations that change the scale of the pattern. We see that the initial noise level $\noiselev$ is augmented by an offset term, which results from the fact that the operations of filtering and applying a geometric transformation do not commute when the transformation involves a scale change. %In fact, this can be physically explained as follows. Remembering that the unfiltered target pattern has a decomposition $\targetp = \ptranopt + \noiseopt$ in terms of a manifold point $\ptranopt$ on the transformation manifold $\M(p)$ of the unfiltered reference pattern $p$ and a noise term $\noiseopt$, we see that the actual noise level inherent in the target pattern is $\| \noiseopt \|$. However, when the target pattern $\targetp$ is filtered, it becomes $\hattargetp = \filtptranopt + \filtnoiseopt$. However, when the filtered target pattern $\targetp$ is registered with respect to the filtered reference pattern $\hatp$, the point $\filtptranopt$ no longer lies on the transformation manifold $\M (\hatp)$ of $\hatp$; therefore, the actual noise level (the distance between $\targetp$ and   $\M (\hatp)$ is  larger than $\| \filtnoiseopt \|$. The offset added to the initial noise level $\noiselev$ is due to the deviation of $\filtptranopt$ from $\M (\hatp)$.
Since filtering and transforming commute for transformation models that do not modify the scales of atoms, the second error term $\|  \filtptranopt - \hatptranoptinit \|$ in (\ref{eq:bnd_hatnoiseopt}) vanishes for such geometric transformations. Thus, if the transformation model $\tpar$ does not involve a scale change, the variation of $\| \hatnoiseopt \| $ is given by
\begin{equation*}
\| \hatnoiseopt \| = O\left( \noiselev  (1+\rho^2)^{-1/2} \right).
\end{equation*}
This finishes the proof of the lemma.

\end{proof}

\section{Proof of the results on algorithm convergence}

\subsection{Proof of Theorem \ref{thm:conv_sin_scale}}
\label{app:tan_dist:sec:pf_conv_sin_scale}

\begin{proof}

From Theorem \ref{thm:bnd_alignerrTD}, we can define an upper bound $\alerrbnd_k $ for the alignment error $\| \tparest^k - \tparopt \|$ of iteration $k$ as follows.
\begin{equation}
\| \tparest^k - \tparopt   \| \leq \alerrbnd_k :=  \MsecderUB \ \lambdamin^{-1} \ \big( [ \Gij (\tparest^{k-1}) ] \big) 
\left( \half \,  \sqrt{\tr( [ \Gij (\tparest^{k-1}) ] )} \ \| \tparopt - \tparest^{k-1}  \|_{1}^2
+  \sqrt{d} \ \noiselev  \   \|  \tparopt - \tparest^{k-1}   \|_1
\right)
\label{eq:defn_alerbndk}
\end{equation}
%

%Let us denote
%%
%\begin{equation*}
%\alpha :=  d \geoconone \geoconthree.
%\end{equation*}
%
In order to show that the estimates $\{ \tparest^k \}_{k=0}^{\infty}$ converge to the optimal solution $\tparopt$, it suffices to show that 
\begin{equation}
\alerrbnd_k \leq \alpha \, \alerrbnd_{k-1}
\label{eq:conv_cond_Ek}
\end{equation}
for all $k$ for some $0< \alpha < 1$. This ensures that $\lim_{k\rightarrow \infty} \alerrbnd_k = 0$; therefore, the alignment errors $\| \tparest^k - \tparopt   \|$ converge to $0$. 

%Note that due to the condition (\ref{eq:cond_geoconthree}), we have $\alpha < 1$. 

By replacing the terms in (\ref{eq:defn_alerbndk}) with their supremums on the manifold defined in (\ref{eq:defn_geo_constants}), we obtain the following inequality:
\begin{equation}
\begin{split}
 \alerrbnd_k &\leq   
 \half \,  \geocontwo \geoconone  \ \| \tparopt - \tparest^{k-1}  \|_{1}^2
+  \sqrt{d} \ \noiselev  \ \geoconone \  \|  \tparopt - \tparest^{k-1}   \|_1
 \\
	& \leq  \half \, d \geocontwo  \geoconone  \ \| \tparopt - \tparest^{k-1}  \|^2
+  d \ \noiselev  \ \geoconone \  \|  \tparopt - \tparest^{k-1}   \| \\
	& \leq  \half \, d \geocontwo \geoconone  \ \alerrbnd_{k-1}^2
+  d \ \noiselev  \ \geoconone \  \alerrbnd_{k-1}.
\end{split}
\label{eq:Ek_recurs}
\end{equation}
In particular, for $k=1$, 
\begin{equation}
\alerrbnd_1  \leq  \half \, d \geocontwo \geoconone  \alerrbnd_0^2 +  d \ \noiselev  \ \geoconone \  \alerrbnd_0
\label{eq:shrk_cond_iter1}
\end{equation}
where $\alerrbnd_0:= \|  \tparopt  - \tparref \|$ is the error in the initial solution $\tparref$. Now let us define
\begin{equation*}
\alpha := \half d \geocontwo \geoconone \alerrbnd_0 + d \ \noiselev \geoconone.
\end{equation*}
From the hypotheses (\ref{eq:cond_noise_conv}) and (\ref{eq:cond_inisol_conv}), we have 
\begin{equation*}
\alpha  < 1.
\end{equation*}
This together with  (\ref{eq:shrk_cond_iter1}) implies that 
\begin{equation*}
\alerrbnd_1 \leq \alpha \alerrbnd_0.
\end{equation*}

Now it remains to show that $\alerrbnd_k \leq \alpha \alerrbnd_{k-1}$ for all $k$, which can be done by strong induction. Assume that $\alerrbnd_n \leq \alpha \alerrbnd_{n-1}$ for all $n = 1, \dots, k-1$. Then, we have
\begin{equation*}
\alerrbnd_{k-1} \leq \alpha \alerrbnd_{k-2} \leq \alpha^2 \alerrbnd_{k-3} \leq \dots \leq \alpha^{k-1} \alerrbnd_{0}.
\end{equation*}
Since $\alpha <1$, this gives $\alerrbnd_{k-1}  \leq \alerrbnd_0 $. From (\ref{eq:Ek_recurs}), we obtain
\begin{equation*}
\begin{split}
 \alerrbnd_k \leq  \alerrbnd_{k-1}\left(  \half \, d \geocontwo  \geoconone  \ \alerrbnd_{k-1}
+  d \ \noiselev  \ \geoconone \right)
	\leq \alerrbnd_{k-1}\left(  \half \, d \geocontwo \geoconone  \ \alerrbnd_{0}
+  d \ \noiselev  \ \geoconone \right)
	= \alpha \alerrbnd_{k-1}.
\end{split}
\end{equation*}
We thus get $\alerrbnd_k \leq \alpha \alerrbnd_{k-1}$ for all $k$, which concludes the proof.

\end{proof}

\subsection{Proof of Corollary \ref{thm:conv_td_hier}}
\label{sec:pf_thm_conv_td_hier}

\begin{proof}

We begin with deriving the optimal filter size that minimizes the alignment error in iteration $k$ of the algorithm. First, we observe from (\ref{eq:hatE1_E2}) that the alignment error in iteration $k$ can be upper bounded as follows: 
\begin{equation*}
\begin{split}
\| \tparest^k - \hattparopt  \|  \leq & \frac{1}{2}  \,   
  \hatMsecderUB \  \lambdamin^{-1} \ \big( [ \hatGij (\tparest^{k-1} ) ] \big) 
 \sqrt{\tr( [ \hatGij (\tparest^{k-1} ) ] )}  
  \ \| \hattparopt - \tparest^{k-1}  \|_1^2 \\
	&+ \sqrt{d} \,  \hatMsecderUB \  \lambdamin^{-1} \ \big( [ \hatGij (\tparest^{k-1} ) ] \big)  
	 \ \| \hatnoiseopt \| 
	  \|  \hattparopt - \tparest^{k-1}   \|_1 .
 \end{split}
\end{equation*}
Ignoring the small perturbation $\| \hattparopt - \tparopt \|$ due to filtering in the projection of the target pattern onto the manifold, we can approximate $\hattparopt \approx \tparopt$. Also, bounding the $\ell^1$-norms in the above expression in terms of $\ell^2$-norms, we obtain
\begin{equation}
\begin{split}
\| \tparest^k - \tparopt  \| 	  \leq &  \frac{1}{2}  d \,   
  \hatMsecderUB \  \lambdamin^{-1} \ \big( [ \hatGij (\tparest^{k-1} ) ] \big) 
 \sqrt{\tr( [ \hatGij (\tparest^{k-1} ) ] )}  
  \ \| \tparopt - \tparest^{k-1}  \|^2 \\
	&+ d \,  \hatMsecderUB \  \lambdamin^{-1} \ \big( [ \hatGij (\tparest^{k-1} ) ] \big)  
	 \ \| \hatnoiseopt \| 
	  \|  \tparopt - \tparest^{k-1}   \|.
\end{split}
\label{eq:alerbnd_inter1}
\end{equation}

Remember that, for any fixed $\tpar \in \tpardom$, the  terms $ \hatMsecderUB \  \lambdamin^{-1} \ \big( [ \hatGij (\tpar) ] \big) $ and $ \sqrt{\tr( [ \hatGij (\tpar) ] )}$ have a variation with the filter size $\rho$ as given in (\ref{eq:dep_K_lmininv}) and (\ref{eq:dep_sqrt_tracemt}). Moreover, at $\rho=0$, the definitions of $\geocontwo$ and $\geoconone$ in (\ref{eq:defn_geo_constants}) give the suprema of these terms attained over $\tpardom$. From these two relations, we deduce that the following inequalities
\begin{equation}
\begin{split}
 \sqrt{\tr( [ \hatGij (\tpar) ] )}  & \leq \beta_1 \geocontwo \, (1+\rho^2)^{-1}\\
  \hatMsecderUB \  \lambdamin^{-1} \ \big( [ \hatGij (\tpar) ] \big)  &\leq \beta_2  \geoconone  \left(1+ \,(1+\rho^2)^{-1/2}\right) (1+ \rho^2) 
\end{split}
\label{eq:form_geoconsts_filt}
\end{equation}
hold for some constants $\beta_1$ and $\beta_2$. The above expressions capture the dependence of these two terms on the filter size $\rho$ as well as on the tangent magnitude and curvature constants $\geocontwo$ and $\geoconone$. In the above inequalities, we omit the constants appearing in the exact variations of these terms with the filter size for the sake of simplicity. From the definitions of $\geocontwo$ and $\geoconone$ in (\ref{eq:defn_geo_constants}), we observe that taking $\beta_1=1$ and $\beta_2=1/2$ results in equalities in (\ref{eq:form_geoconsts_filt}) for the case $\rho=0$. In the following, we adopt these values for the constants $\beta_1$ and $\beta_2$. Although this choice does not guarantee the inequalities in (\ref{eq:form_geoconsts_filt}) for all values of $\rho$, this approximation simplifies our analysis and allows us to obtain an approximate expression for the variation of the alignment error with the filter size $\rho$ that holds up to a multiplication by a constant. Evaluating the expressions in (\ref{eq:form_geoconsts_filt}) at  $\rho_k$ and using them in (\ref{eq:alerbnd_inter1}), we obtain 
\begin{equation}
\begin{split}
\| \tparest^k - \tparopt  \| 
	 \leq &    \frac{1}{4} d \, \geocontwo \geoconone   \left(1+ \,(1+\rho_k^2)^{-1/2}\right)  \ \| \tparopt - \tparest^{k-1}  \|^2 \\
	&+ \half d \, \geoconone  \left(1+ \,(1+\rho_k^2)^{-1/2}\right) (1+ \rho_k^2)  \,  \| \hatnoiseopt \| \,  \|  \tparopt - \tparest^{k-1}   \|.
\end{split}
\label{eq:alerbnd_inter_filt}
\end{equation}

Now, from Lemma \ref{lem:dep_noiselev}, we can  approximate the noise term $\| \hatnoiseopt \|$ in iteration $k$ in terms of the filter size and the effective noise level parameter $\noiseeff$ as
\begin{equation*}
\| \hatnoiseopt \| \approx \noiseeff \, (1+ \rho_k^2)^{-1/2}.
\end{equation*}
Using this in (\ref{eq:alerbnd_inter_filt}) gives the following upper bound $\alerrbnd_k$ for the alignment error in iteration $k$
\[
\| \tparest^k - \tparopt  \| \leq \alerrbnd_k
\]
where
\begin{equation}
\begin{split}
\alerrbnd_k &:= \frac{1}{4} d \,  \geocontwo \geoconone  \left(1+ \,(1+\rho_k^2)^{-1/2}\right)  \ \| \tparopt - \tparest^{k-1}  \|^2 \\
	&+ \half \, d \, \geoconone \noiseeff \left(1+ \,(1+\rho_k^2)^{-1/2}\right) (1+ \rho_k^2)^{1/2}    \|  \tparopt - \tparest^{k-1}   \|.
\end{split}
\label{eq:alerbnd_iter_filt}
\end{equation}

Finally, from (\ref{eq:alerbnd_iter_filt}), we determine the optimal value of the filter size $\rho_k$ in iteration $k$ by evaluating the value of $\rho$  that minimizes $\alerrbnd_k$.
\begin{eqnarray}
\rho_k &=& 
\sqrt{ \frac{ \geocontwo \| \tparopt - \tparest^{k-1} \| }{ 2 \, \noiseeff  } - 1} \, \, \text{    if    } \| \tparopt - \tparest^{k-1} \| \geq \frac{2 \, \noiseeff }{ \geocontwo}
\label{eq:opt_choice_rhok_app}
\\ 
\rho_k &=& \ 0 \qquad \qquad \qquad \, \, \text{    if   }   \| \tparopt - \tparest^{k-1} \| < \frac{2 \, \noiseeff  }{ \geocontwo}
\label{eq:opt_choice_rhok0_app}
\end{eqnarray}

Now, the alignment error bound $\alerrbnd_k$ in (\ref{eq:alerbnd_iter_filt}) as a function of $\rho_k$ is either increasing or it has one global minimum at the value of $\rho_k$ specified in (\ref{eq:opt_choice_rhok_app}). Therefore, any choice of the filter size $\rho_k$ that is between $0$ and the optimal value in (\ref{eq:opt_choice_rhok_app})-(\ref{eq:opt_choice_rhok0_app}) yields an alignment error that is smaller than or equal to the error obtained by applying no filtering ($\rho_k = 0$). Hence, evaluating the right-hand side of the expression in (\ref{eq:alerbnd_iter_filt}) at $\rho_k = 0$, we get
\begin{equation}
\begin{split}
\alerrbnd_k \leq  \half \, d \, \geocontwo \geoconone   \ \| \tparopt - \tparest^{k-1} \|^2 
	+  d \, \geoconone \noiseeff    \, \| \tparopt - \tparest^{k-1} \| .
\end{split}
\label{eq:Ek_pf_hier}
\end{equation}
We then proceed as in the proof of Theorem \ref{thm:conv_sin_scale}. Defining 
\begin{equation*}
\alpha := \half d \geocontwo \geoconone \alerrbnd_0 + d \ \noiseeff \geoconone
\end{equation*}
where $\alerrbnd_0= \|  \tparopt  - \tparref \|$, the condition in (\ref{eq:cond_init_err_filt}) ensures that $\alpha < 1$. From (\ref{eq:Ek_pf_hier}), we have 
\begin{equation*}
\alerrbnd_1 \leq \alpha \alerrbnd_0
\end{equation*}
in iteration $k=1$. Applying the same steps as those in the proof of Theorem \ref{thm:conv_sin_scale}, one can then easily show that $E_k \leq \alpha E_{k-1}$ for all $k$, which implies that the alignment error upper bounds converge to $0$. 

\end{proof}

\section{Proof of the results on classification performance}

\subsection{Proof of Lemma  \ref{lem:dist_est_error}}
\label{sec:pf_lem_dist_est_error}

\begin{proof}
We first bound the distance estimation error using the reverse triangle inequality as follows
\begin{equation}
\big| \|  \targetp - p_{\tparopt} \| -    \| \targetp -  p_{\tparest} \|  \big| \leq \| p_{\tparopt} -  p_{\tparest}  \|.
\label{eq:rev_tri_disterr}
\end{equation}
Next, in order to derive an upper bound on  $\| p_{\tparopt} -  p_{\tparest}  \|$, we define a curve
\begin{equation*}
p_{\tpar(t)}: [0, 1] \rightarrow \M(p)
\end{equation*}
such that
\begin{equation*}
\tpar(t)=\tparest + t (\tparopt - \tparest).
\end{equation*}
We have
\begin{equation*}
p_{\tparopt} = p_{\tparest} + \int_{0}^1 \frac{d p_{\tpar(t)}}{dt} dt.
\end{equation*}
Hence,
\begin{equation*}
\begin{split}
\| p_{\tparopt} -  p_{\tparest}  \|  &= \left\|  \int_0^1 \frac{d p_{\tpar(t)}}{dt} dt  \right\| 
   =  \left\|  \int_0^1  \partial_i p_{\tpar(t)}   \frac{d \tpar^i(t)}{dt} dt  \right\|  \\
   &= \left\|  \int_0^1   \partial_i p_{\tpar(t)}   (\tparopt^i - \tparest^i) dt  \right\|  
   \leq  \int_0^1 \sum_{i=1}^d  \|  \partial_i p_{\tpar(t)} \| \,  | \tparopt^i - \tparest^i | \, dt  \\
   & \leq  \MderUB  \int_0^1  \sum_{i=1}^d     | \tparopt^i - \tparest^i | \, dt = \MderUB \| \tparopt - \tparest \|_1. 
\end{split}
\end{equation*}
Combining this with (\ref{eq:rev_tri_disterr}), we get the stated upper bound on the distance estimation error
\begin{equation*}
\big| \|  \targetp - p_{\tparopt} \| -    \| \targetp -  p_{\tparest} \|  \big|  \leq \MderUB \| \tparopt - \tparest \|_1. 
\end{equation*}

\end{proof}

\subsection{Proof of Theorem \ref{thm:misclass_prob}}
\label{pf:thm_misclass_prob}

\begin{proof}

Let
\begin{equation*}
\tilde \noiselev_j :=  \| \targetp -  p_{\tparest^j}^j \|  
\end{equation*}
denote the estimate given by the tangent distance method of the distance $\noiselev_j$ between the query pattern $\targetp$ and the manifold $\M(p^j)$, for $j=1, \dots, M$. Since $p^j_{\tparopt^j}$ is the projection of $\targetp$ onto $\M(p^j)$, we have $\tilde \noiselev_j \geq \noiselev_j$ for all $j$. As the query pattern $\targetp$ belongs to class $m$, it is correctly classified with the tangent distance method if 
$
\tilde \noiselev_m <  \tilde \noiselev_j 
$
for all $j\neq m$.

Let us denote the distance estimation error for class $m$ as follows
\begin{equation*}
\disterr_{\noiselev_m} := | \tilde \noiselev_m  -  \noiselev_m  | = \tilde \noiselev_m  -  \noiselev_m =\| \targetp - p^m_{\tparest^m} \|  -  \| \targetp - p^m_{\tparopt^m} \|.
\end{equation*}
Now let $j$ be any fixed class label other than $m$. Since we have
$\tilde \noiselev_m  =  \noiselev_m + \disterr_{\noiselev_m} $
and
$\noiselev_j < \tilde \noiselev_j$, the condition
\begin{equation*}
\disterr_{\noiselev_m} < \sepmarg
\end{equation*}
implies
\begin{equation*}
\tilde \noiselev_m  =  \noiselev_m + \disterr_{\noiselev_m} <   \noiselev_m + \sepmarg 
\leq   \noiselev_j \leq  \tilde \noiselev_j.
\end{equation*}
Therefore, if the condition $\disterr_{\noiselev_m} < \sepmarg$ is satisfied, we have $\tilde \noiselev_m  < \tilde \noiselev_j$. From Lemma \ref{lem:dist_est_error}, we have
\begin{equation*}
\disterr_{\noiselev_m} \leq \MderUB_m \| \tparopt^m - \tparest^m \|_1.
\end{equation*}
Furthermore, applying Theorem \ref{thm:bnd_alignerrTD}, we can upper bound the distance estimation error as 
\begin{equation*}
\begin{split}
\disterr_{\noiselev_m} &\leq \MderUB_m \, \sqrt{d} \,  \| \tparopt^m - \tparest^m \| \\
&\leq
 \MderUB_m \, \sqrt{d} \, \MsecderUB_m \ \lambdamin^{-1} \ \big( [ \Gij^m (\tparref^m) ] \big) 
\left( \half \,  \sqrt{\tr( [ \Gij^m (\tparref^m) ] )} \ \| \tparopt^m - \tparref^m  \|_{1}^2
+  \sqrt{d} \ \noiselev_m  \   \|  \tparopt^m - \tparref^m   \|_1
\right)\\
&\leq
\overline{\disterr}_{\noiselev_m}
\end{split}
\end{equation*}
where
\begin{equation*}
 \overline{\disterr}_{\noiselev_m}:= \MderUB_m \, \sqrt{d} \, \MsecderUB_m \ \lambdamin^{-1} \ \big( [ \Gij^m (\tparref^m) ] \big) 
\left( \half \,  \sqrt{\tr( [ \Gij^m (\tparref^m) ] )} \ \Delta^2
+  \sqrt{d} \ \maxdist_m  \   \Delta
\right).
\end{equation*}
In the following $P(\cdot)$ denotes probability and $\mathbb{E}[\cdot]$ denotes expectation. We have
\begin{equation*}
P\left(  \tilde \noiselev_m  < \tilde \noiselev_j \right) \geq P(\disterr_{\noiselev_m} < \sepmarg).
\end{equation*}
Applying Markov's inequality, we get
\begin{equation*}
P(\disterr_{\noiselev_m} \geq \sepmarg) \leq \frac{ \mathbb{E} [{ \disterr}_{\noiselev_m} ] }{\sepmarg} \leq  \frac{ \overline{ \disterr}_{\noiselev_m}  }{\sepmarg}.
\end{equation*}
Therefore,
\begin{equation*}
P\left(  \tilde \noiselev_m  < \tilde \noiselev_j \right) \geq P(\disterr_{\noiselev_m} < \sepmarg) \geq 1 - \frac{ \overline{ \disterr}_{\noiselev_m}  }{\sepmarg}.
\end{equation*}

Using the union bound on all class labels $j \in \{1, \dots, M \} \setminus \{ m\}$, we lower bound the probability of correctly classifying $\targetp$ as 
\begin{equation*}
P\left(\tilde l(\targetp) = l (\targetp)\right)  = P\left( \tilde \noiselev_m  < \tilde \noiselev_j , \, \forall j\neq m \right)  \geq 1 - \frac{(M-1)}{\sepmarg}  \overline{ \disterr}_{\noiselev_m}
\end{equation*}
which gives the upper bound on the misclassification probability stated in the theorem.
\end{proof}

\end{document}